\documentclass{article}

    \usepackage[final,nonatbib]{neurips_2022}

\usepackage[utf8]{inputenc} % allow utf-8 input
\usepackage[T1]{fontenc}    % use 8-bit T1 fonts
\usepackage{hyperref}       % hyperlinks
\usepackage{url}            % simple URL typesetting
\usepackage{booktabs}       % professional-quality tables
\usepackage{amsfonts}       % blackboard math symbols
\usepackage{nicefrac}       % compact symbols for 1/2, etc.
\usepackage{microtype}      % microtypography

\usepackage{cite}
\usepackage{color}
\usepackage{latexsym}
\usepackage{amsmath}
\usepackage{amsthm}
\usepackage{amssymb}
\usepackage{algorithm}
\usepackage{algorithmic}
\usepackage{graphics} % for pdf, bitmapped graphics files
\usepackage{epsfig} % for postscript graphics files
\graphicspath{{Fig/}}

\usepackage{enumitem}
\usepackage{multirow}
\usepackage{array}
\usepackage{bbding}
\usepackage{threeparttable}
\usepackage{thmtools, thm-restate}

\usepackage{subcaption}
\usepackage{xcolor}
\usepackage{colortbl}
\usepackage{soul}

\usepackage{placeins}

\usepackage{booktabs}
\usepackage{makecell}
\usepackage{bigstrut,rotating}

\newcommand{\alg}{$\mathsf{FAT}$-$\mathsf{Clipping}~$}
\newcommand{\algns}{$\mathsf{FAT}$-$\mathsf{Clipping}$}
\newcommand{\algpr}{$\mathsf{FAT}$-$\mathsf{Clipping}$-$\mathsf{PR}~$}
\newcommand{\algprns}{$\mathsf{FAT}$-$\mathsf{Clipping}$-$\mathsf{PR}$}
\newcommand{\algpi}{$\mathsf{FAT}$-$\mathsf{Clipping}$-$\mathsf{PI}~$}
\newcommand{\algpins}{$\mathsf{FAT}$-$\mathsf{Clipping}$-$\mathsf{PI}$}

\newcommand{\x}{\mathbf{x}}

\newcommand{\y}{\mathbf{y}}

\newtheorem{thm}{Theorem}
\newtheorem{cor}[thm]{Corollary}
\newtheorem{lem}{Lemma}

\newtheorem{rem}{Remark}

\newtheorem{assum}{Assumption}
\title{Taming Fat-Tailed (``Heavier-Tailed'' with Potentially Infinite Variance) Noise in Federated Learning}

\author{
Haibo Yang \\
Dept. of ECE\\
The Ohio State University\\
Columbus, OH 43210 \\
\texttt{yang.5952@osu.edu} \\
\And
Peiwen Qiu \\
Dept. of ECE\\
The Ohio State University\\
Columbus, OH 43210 \\
\texttt{qiu.617@osu.edu} \\ 
\And
Jia Liu \\
Dept. of ECE\\
The Ohio State University\\
Columbus, OH 43210 \\
\texttt{liu@ece.osu.edu} \\ 
}

\begin{document}

\maketitle
 
% !TEX root = main.tex

\begin{abstract}
In recent years, federated learning (FL) has emerged as an important distributed machine learning paradigm to collaboratively learn a global model with multiple clients, while keeping data local and private.
However, a key assumption in most existing works on FL algorithms' convergence analysis is that the noise in stochastic first-order information has a finite variance.
Although this assumption covers all light-tailed (i.e., sub-exponential) and some heavy-tailed noise distributions (e.g., log-normal, Weibull, and some Pareto distributions), it fails for many fat-tailed noise distributions (i.e., ``heavier-tailed'' with potentially infinite variance) that have been empirically observed in the FL literature.
To date, it remains unclear whether one can design convergent algorithms for FL systems that experience fat-tailed noise.
This motivates us to fill this gap in this paper by proposing an algorithmic framework called \alg (\ul{f}ederated \ul{a}veraging with \ul{t}wo-sided learning rates and \ul{clipping}), which contains two variants: \alg per-round (\algprns) and \alg per-iteration (\algpins).
Specifically, for the largest tail-index $\alpha \in (1,2]$ such that the fat-tailed noise in FL still has a bounded $\alpha$-moment, we show that both variants achieve $\mathcal{O}((mT)^{\frac{2-\alpha}{\alpha}})$ and $\mathcal{O}((mT)^{\frac{1-\alpha}{3\alpha-2}})$ convergence rates in the strongly-convex and general non-convex settings, respectively, where $m$ and $T$ are the numbers of clients and communication rounds.
Moreover, with more clipping operations compared to \algprns, \algpi further enjoys a linear speedup effect with respect to the number of local updates at each client and being lower-bound-matching (i.e., order-optimal).
Collectively, our results advance the understanding of designing efficient algorithms for FL systems that exhibit fat-tailed first-order oracle information.
%
%In this paper, we provide preliminary evidence of heavy-tailed property empirically observed in federated learning, which leads to theoretically algorithmic divergence and numerically catastrophic failure of model performance.
%So far, such emerging phenomena lacks correspondingly theoretical analysis since conventional convergence analysis based on bounded second-order moment is inapplicable under such heavy-tailed noise.
%Accordingly, it motivates us to develope efficient algorithms and the corresponding theoretical guarantees to tackle this challenge.
%Toward this end, we first propose two clipping methods for federated learning by taking clipping operations in each round or each iteration, dubbed as CPR-FedAvg (Clipping-Per-Round) and CPI-FedAvg (Clipping-Per-Iteration).
%Then we show the convergence bounds for CPR-FedAvg and CPI-FedAvg in strongly-convex and non-convex loss functions under heavy-tailed noise.
%Compared with lower bounds, our results indicate the convergence rate of CPI-FedAvg is optimal. 
%Numerically, the effectiveness of our algorithms is validated in various settings.
%In addition, we explore the impact of data heterogeneity and local steps in federated learning on the heavy-tailed noise.
\end{abstract}
% !TEX root = main.tex

\section{Introduction} \label{sec:intro}
In recent years, federated learning (FL) has emerged as an important distributed machine learning paradigm, where, coordinated by a server, a set of clients collaboratively learn a global model, while keeping their training data local and private.
%Federated Learning (FL) is a popular distributed machine learning paradigm that leverages clients to collaboratively learn a model with decentralized data under the coordination of a central server.
With intensive research in recent years, researchers have developed many FL algorithms (e.g., FedAvg~\cite{mcmahan2017communication} and many follow-ups~\cite{Li2020fedprox,Karimireddy2020SCAFFOLD,wang2020fednova,zhang2020fedpd,acar2021feddyn,yang2021linearspeedup,reddi2021adaptive,khanduri2021stem,gu2021fast,luo2021nofear,karimireddy2021mime}) that  have been theoretically shown to achieve fast convergence rates in the presence of various types of randomness and heterogeneity resulted from training data, network environments, computing resources at clients, etc.
Moreover, many of these algorithms enjoy the so-called ``linear speedup'' effect, i.e., the convergence time to a first-order stationary point is inversely proportional to the number of workers and local update steps.

%Akin to stochastic gradient descent (SGD) for centralized learning, federated averaging (FedAvg)~\cite{mcmahan2017communication} is the backbone algorithm for FL in deep learning model.
%FedAvg utilizes local steps and model parameter exchange to achieve better communication efficiency and data privacy.
%Since then, it has sparked many follow-ups~\cite{Li2020fedprox,Karimireddy2020SCAFFOLD,wang2020fednova,zhang2020fedpd,acar2021feddyn,yang2021linearspeedup,reddi2021adaptive,khanduri2021stem,gu2021fast,luo2021nofear,karimireddy2021mime}, aiming to better tackle the challenges of system or data heterogeneity in FL.
%These algorithms have been able to achieve (nearly) {\em optimal} convergence rate as first-order stochastic methods and substantially great performances in practice.

However, despite the recent advances in FL algorithm design and theoretical understanding,
a ``cloud that remains obscures the sky of FL'' is a common assumption that can be found in almost all works on performance analysis of FL algorithms, which states that the random noise in stochastic first-order oracles (e.g., stochastic gradients or associated estimators) has a {\em finite variance.}
Although this assumption is not too restrictive and can cover all light-tailed (i.e., sub-exponential) and some heavy-tailed noise distributions (e.g., log-normal, Weibull, and some Pareto distributions), it fails for many ``fat-tailed'' distributions (i.e., ``heavier-tailed'' with potentially infinite variance\footnote{In the literature, the terminologies ``heavy-tailed'' and ``fat-tailed''  are not universally defined and could be interchangeable sometimes.
In this paper, we follow the convention of those authors who reserve the term ``fat-tailed'' to mean the subclass of heavy-tailed distributions that exhibit power law decay behavior as well as infinite variance (see, e.g., \cite{nair2022fundamentals,taylor2016}).
Thus, every fat-tailed distribution is heavy-tailed, but the reverse is not true. 
}).
% (i.e., ``heavier-tailed'' with potentially infinite variance).
In fact, fat-tailed distributions have already been empirically observed under centralized learning settings~\cite{simsekli2019tail,gurbuzbalaban2021heavy,Zhang2020Why,gorbunov2020stochastic,panigrahi2019non}, let alone in the more heterogeneous FL environments. 
%However, despite the successes of FedAvg and its variants, one pillar in the convergence analysis is the bounded second-order moment as a widely-used assumption.
%In other words, the distribution of stochastic gradient noise is usually modelled or approximated as a (sub-)Gaussian.
%However, recent studies have provided many evidences of {\em heavy-tailed} noise for stochastic gradient even in centralized learning~\cite{simsekli2019tail,gurbuzbalaban2021heavy,Zhang2020Why,gorbunov2020stochastic,panigrahi2019non}.
%In FL, this heavy-tailed property becomes more prominent.
Later in Section~\ref{sec:mot}, we will also provide empirical evidence that shows that fat-tailed noise distributions can be easily induced by FL systems with non-i.i.d. datasets and heterogeneous local updates across clients.

The presence of fat-tailed noise poses two major challenges in FL algorithm design and analysis:
i) Experimentally, it has been shown in~\cite{charles2021large} that many existing FL algorithms suffer severely from fat-tailed noise and frequently exhibit the so-called ``catastrophic failure of model performance'' (i.e., sudden and dramatic drops of learning accuracy during the training phase);
ii) Theoretically, the infinite variance of the random noise in the stochastic first-order oracles renders most of the proof techniques in existing FL algorithmic convergence analysis inapplicable, which necessitates new algorithmic ideas and proof strategies.
In light of these empirical and theoretical challenges, two foundational questions naturally emerge in FL algorithm design and analysis: {\em 1) Can we develop FL algorithms with convergence guarantee under fat-tailed noise? 2) If the answer to 1) is ``yes,'' could we characterize their finite-time convergence rates?}
%{\em Can we develop efficient algorithm and achieve optimal convergence rate for FL under heavy-tailed noise?}
%Specifically, the local steps coupled with the data heterogeneity in FL lead to more heavy-tailed property (see Sec~\ref{sec:mot} for more details) and thus resulting in frequently catastrophic failure of model performance~\cite{charles2021large}. 
%On the theoretical side, the infinitely second-order moment induced by the heavy-tails makes the conventionally theoretical analysis inapplicable, and thereby impeding algorithmic convergence of FedAvg-based methods.
%In light of the above observations, a highly non-trivial question naturally emerges from both the theoretical and practical perspectives: {\em can we design efficient algorithm and achieve optimal convergence rate for FL under heavy-tailed noise?}
In this paper, we provide affirmative answer to the above questions.
Our major contributions in this paper are highlighted as follows:

\begin{list}{\labelitemi}{\leftmargin=1em \itemindent=0em \itemsep=.2em}
%\item We empirically show the heavy-tailed property in FL. Furthermore, its dependence on data heterogeneity and local steps is investigated. We show that, for more heterogeneous datasets and/or more local steps, the heavy-tailed property is more prominent (smaller heavy-tailed index $\alpha$ in Sec~\ref{sec:mot}). The direct consequence is frequently catastrophic model failure in the training process due to such heavy-tails.

\item To address the challenges of the fat-tailed noise in FL algorithm design, we propose an algorithmic framework called \alg (\ul{f}ederated \ul{a}veraging with \ul{t}wo-sided learning rates and \ul{clipping}), which leverages a clipping technique to mitigate the impact of fat-tailed noise and uses a two-sided learning rate mechanism to lower communication complexity.
Our \alg framework contains two variants: \alg per-round (\algprns) and \alg per-iteration (\algpins).
We show that, for the largest tail-index $\alpha \in (1,2]$ such that the fat-tailed noise in FL still has a bounded $\alpha$-moment, both \alg variants achieve $\mathcal{O}((mT)^{\frac{2-\alpha}{\alpha}})$ and $\mathcal{O}((mT)^{\frac{1-\alpha}{3\alpha-2}})$ convergence rates in the strongly-convex and general non-convex settings, respectively, where $m$ and $T$ are the numbers of clients and communication rounds. 

\item Between the proposed \alg variants, \algpr only performs one clipping operation in each communication round before client communicates to the server, while \algpi performs clipping in each iteration of local model update.
We show that, at the expense of more clipping operations compared to \algprns, \algpi further achieves a linear speedup effect with respect to the number local model updates at each client and is lower-bound matching in terms of convergence rate.

%two clipping methods for federated learning, i.e., CPR-FedAvg (Clipping-Per-Round) and CPI-FedAvg (Clipping-Per-Iteration).
%    

%\item We theoretically analyze the convergence rate of CPR-FedAvg and CPI-FedAvg for strongly-convex and non-convex functions under heavy-tailed noise in FL.
%The convergence rates are detailed in Table~\ref{tab:bound}.
%To our best knowledge, this is the first analysis under heavy-tailed noise in FL.
%Compared with lower bounds, we highlight the optimal rate of CPI-FedAvg. 
    
\item In addition to theoretical analysis, we also conduct extensive numerical experiments to study the fat-tailed phenomenon in FL systems and verify the efficacy of our proposed \alg algorithms for FL systems with fat-tailed noise.
We first provide concrete empirical evidence that fail-tailed noise distributions are not uncommon in FL systems with non-i.i.d. datasets and heterogeneous local updates.
We show that our \alg algorithms render a much smoother FL training process, which effectively prevents the ``catastrophic failure'' in various FL settings.
\end{list}

\begin{table}[t!]
	\centering
	\caption{Convergence rate comparisons under fat-tailed noise distributions (shaded parts are our results; metrics: $f(\x) - f(\x^*) \leq \epsilon$ and $\| \nabla f(\x) \| \leq \epsilon$ for strongly-convex and non-convex functions, respectively):
	$\alpha = 2$ and $\alpha \in (1, 2)$ correspond to non-fat-tailed and fat-tailed noises, respectively. 
	%We omit the higher order in the convergence bound. 
	Here, $R$ is the total number of iterations for centralized algorithms (SGD and GClip); $K$ and $T$ are local update steps and communication rounds in the FL setting, respectively; $m$ is the number of clients. 
	N/A means no theoretical guarantee for convergence. 
	Note that the total number of iterations $R$ in FL can be computed as $R = KT$, which relates to that in the centralized setting. %showing the results in FL generalize that for centralized algorithms. 
	}
     \renewcommand{\arraystretch}{1.2}
     {\scriptsize
	\begin{tabular}{p{2cm}<{\centering} | p{2.6cm}<{\centering} | p{2.15cm}<{\centering} ||p{2.86cm}<{\centering} | p{2.15cm}<{\centering}} 
		\hline
          \multirow{2}{*}{Methods} & \multicolumn{2}{c ||}{Strongly Convex Objective Functions} & \multicolumn{2}{c}{Nonconvex Objective Functions} \\
          \cline{2-5}
            & Fat-Tailed & Non-Fat-Tailed & Fat-Tailed & Non-Fat-Tailed \\ \hline 
		SGD\cite{ghadimi2013stochastic} & N/A & $\mathcal{O}(R^{-1})$ & N/A & $\mathcal{O}(R^{-\frac{1}{4}})$ \\
		GClip\cite{zhang2020adaptive} & $\mathcal{O}(R^{\frac{ 2 - 2\alpha}{\alpha}})$ & $\mathcal{O}(R^{-1})$ & $\mathcal{O}(R^{\frac{1 - \alpha}{3 \alpha - 2}})$ & $\mathcal{O}(R^{-\frac{1}{4}})$ \\
          \hline
          FedAvg\cite{Karimireddy2020SCAFFOLD,yang2021linearspeedup} & N/A & $\tilde{\mathcal{O}}((mKT)^{-1})$ & N/A & $\mathcal{O}((mKT)^{-\frac{1}{4}})$ \\
          \rowcolor{lightgray!50}
		$\boldsymbol{\mathsf{FAT}}${\bf-}$\boldsymbol{\mathsf{Clipping}}${\bf-}$\boldsymbol{\mathsf{PR}}$ & $\mathcal{O}((mT)^{\frac{2 - 2\alpha}{\alpha}} K^{\frac{2}{\alpha}})$ & $\tilde{\mathcal{O}}((mKT)^{-1})$ & $\mathcal{O}((mT)^{\frac{ 1 - \alpha}{3 \alpha - 2}} K^{\frac{2 - \alpha}{3 \alpha - 2}})$ & $\mathcal{O}((mKT)^{-\frac{1}{4}})$  \\
          \rowcolor{lightgray!50}
		$\boldsymbol{\mathsf{FAT}}${\bf-}$\boldsymbol{\mathsf{Clipping}}${\bf-}$\boldsymbol{\mathsf{PI}}$ & $\tilde{\mathcal{O}}((mKT)^{\frac{ 2 - 2 \alpha }{\alpha}})$ & $\tilde{\mathcal{O}}((mKT)^{-1})$ & $\mathcal{O}((mKT)^{\frac{ 1 - \alpha}{3 \alpha - 2}})$ & $\mathcal{O}((mKT)^{-\frac{1}{4}})$  \\
          \hline
          \rowcolor{lightgray!50}
		{\bf Lower Bound} & $\Omega((mKT)^{\frac{ 2 - 2 \alpha }{\alpha}})$ & $\Omega((mKT)^{-1})$ & $\Omega((mKT)^{\frac{ 1 - \alpha}{3 \alpha - 2}})$ & $\Omega((mKT)^{-\frac{1}{4}})$  \\
          \hline
	\end{tabular}}
    \label{tab:bound}
\end{table}

For quick reference and easy comparisons, we summarize all convergence rate results in Table~\ref{tab:bound}.
The rest of the paper is organized as follows.
In Section~\ref{sec:related_work}, we review the literature to put our work in comparative perspectives.
In Section~\ref{sec:mot}, we provide empirical fat-tailed evidence for FL to further motivate this work.
Section~\ref{sec:alg} presents our \alg algorithms and their convergence analyses.
% Section~\ref{sec:discussion} discusses the implication of the convergence rate analysis.
Section~\ref{sec:numerical} presents numerical results and Section~\ref{sec:conclusion} concludes this paper.
Due to space limitation, all proof details and some experiments are provided in the supplementary material.
% !TEX TS-program = latex
% !TEX root = main.tex

\section{Related work} \label{sec:related_work}

In this section, we will provide a quick overview on three related topics in the literature: i) federated learning, ii) heavy-tailed noise in learning, and iii) the clipping techniques, thus putting our work into comparative perspective to highlight our novelty and differences.

\smallskip
\textbf{1) Federated Learning:}
As mentioned earlier, FL has recently emerged as an important distributed learning paradigm. 
The first and perhaps the most popular FL method, the federated averaging (FedAvg) algorithm~\cite{mcmahan2017communication}, was initially proposed as a heuristic to improve communication efficiency and data privacy.
Since then, FedAvg has sparked many follow-ups to further address the challenges of data/system heterogeneity and further reduce iteration and communication complexities.
Notable approaches include adding regularization for the local loss function~\cite{Li2020fedprox,acar2021feddyn,zhang2020fedpd}, using variance reduction techniques~\cite{Karimireddy2020SCAFFOLD}, taking adaptive learning rate strategy~\cite{reddi2021adaptive} or adaptive communication strategy~\cite{Wang2019adapCom,yang2022anarchic}, and many momentum variants~\cite{wang2020fednova,khanduri2021stem,gu2021fast}.
Empirically, these algorithms are shown to be communication-efficient~\cite{mcmahan2017communication} and enjoy better generalization performance ~\cite{lin2018don}.
Moreover, many state-of-the-art algorithms enjoy the ``linear speedup'' effect in terms of the numbers of clients and local update steps in different FL settings~\cite{Karimireddy2020SCAFFOLD,yang2021linearspeedup,yang2022anarchic,Zhang2022NETFLEETAL}.
%In theory, the state-of-the-art analyses provide optimal convergence rates for (generalized) FedAvg (and its variants) as first-order optimization methods~\cite{yang2021linearspeedup,Karimireddy2020SCAFFOLD,gu2021fast}, showing the speedup in terms of the local steps and clients' number.
We note, however, that all these theoretical results are built upon the finite variance assumption of stochastic gradient noise.
%, which is a widely-used assumption for the stochastic gradient noise.
Unfortunately, when the stochastic gradient noise is fat-tailed, the finite variance assumption no longer holds, and hence the associated theoretical analysis is also invalid.
%As a result, it leaves a theoretical gap in such sense.
This motivates us to fill this gap in this paper and conduct the first theoretical analysis for FL systems that experience fat-tailed noise.
% As recent studies have provided preliminary evidence of heavy-tailed noise empirically in federated optimization~\cite{charles2021large}, the conventional analysis is infeasible under heavy-tailed noise.
% Hence, it leaves a theoretical gap in such sense.
% In this paper, we provide the first theoretical analysis in FL under heavy-tailed noise to fill this gap.

\smallskip
\textbf{2) Heavy-Tailed Noise in Learning:}
Recently, heavy-tailed noise has been empirically observed in modern machine learning systems and theoretically analyzed~\cite{nguyen2019first,simsekli2019tail,simsekli2020fractional,zhang2020adaptive,gorbunov2020stochastic,hodgkinson2021multiplicative,gurbuzbalaban2021heavy,wang2021convergence}. 
Heavy-tailed noise significantly affects the learning dynamics and computational complexity, such as the first exit time escaping from saddle point~\cite{nguyen2019first} and iteration complexity~\cite{zhang2020adaptive}.
This is dramatically different from classic dynamic analysis often based on sub-Gaussian noise assumption~\cite{yaida2018fluctuation,hu2019diffusion} and algorithmic convergence analysis with bounded variance assumption~\cite{bottou2018optimization,ghadimi2013stochastic}.
However, for FL, there exist few investigations about heavy-tailed behaviors.
In this paper, we first demonstrate through extensive experiments that fat-tailed (i.e., heavier-tailed) noise in FL can be easily induced by data heterogeneity and local update steps.
We then propose efficient algorithms to mitigate the impacts of fat-tails.
% According to our results in Sec~\ref{sec:mot}, the data heterogeneity and multiple local steps in FL may further exaggerate the heavy-tailed property and in turn cause unexpected consequence, such as frequently catastrophic model failure.
% \cite{charles2021large} observed potentially catastrophic model failure, which is a preliminary evidence of heavy-tailed noise but lacked further analysis.
% In this paper, we systematically investigate the heavy-tailed property in FL through extensive experiments, focusing on its dependence of data heterogeneity and local steps.
% Furthermore, we propose two clipping methods to tackle the heavy-tailed noise with optimal convergence rate guarantee.

\smallskip
\textbf{3) The Clipping Technique:}
Since our \alg algorithms are based on the idea of clipping, here we provide an overview on this technique.
As far as we know, dating back to at least 1985~\cite{Shor1985MinimizationMF}, gradient clipping has been an effective technique to ensure convergence for optimization problems with fast-growing objective functions.
In deep learning, clipping is a widely adopted technique to address the exploding gradient problem.
Recently, gradient clipping was theoretically shown to be able to accelerate the training of centralized learning~\cite{zhang2020improved,Zhang2020Why,chen2020understanding,qian2021understanding}.
Also, clipping is an effective approach to mitigate heavy-tailed noise~\cite{Zhang2020Why,gorbunov2020stochastic} in centralized learning.
In FL, clipping has been used as the preconditioning step for preserving differential privacy (DP)~\cite{Zhang2021UnderstandingCF,das2021privacy,andrew2021differentially}.
Unlike these works, in this paper, we utilize clipping to address algorithmic divergence caused by fat-tailed noise in FL.

% !TEX TS-program = latex
% !TEX root = main.tex

\section{Fat-tailed noise phenomenon in federated learning} \label{sec:mot}
In this section, we first introduce the basic FL problem statement and the standard FedAvg algorithm for FL.  
Then, we provide some necessary background of fat-tailed distributions and provide empirical evidence to show that fat-tailed noise can be easily induced by heterogeneity of data and local updates in FL, which further motivates this work.
Lastly, we demonstrate the algorithmic divergence and frequently catastrophic model failure under fat-tailed noise.

%\textbf{Notation.}
%In this paper,  we let $m$ be the total number of clients.
%We use $K$ to denote the number of local steps per communication round at each client.
%We let $T$ be the number of total communication rounds.
%In addition, we use boldface to denote matrices/vectors.
%We let $[\cdot]_{t, i}^k$ represent the parameter of $k$-th local step in the $i$-th client after the $t$-th communication round.
%We use $\norm{\cdot}$ to denote the $\ell^{2}$-norm.
%For a natural number $m$, we use $[m]$ to represent the set $\{1, \cdots, m \}$. 

\begin{algorithm}[t!]
    \caption{Generalized FedAvg Algorithm (GFedAvg).} \label{alg:fedavg} 
    \begin{algorithmic}[1]
    \STATE $\text{Initialize } \x_1$. 
    \FOR{$t = 1, \cdots, T$ (communication round)} 
        \FOR{each client $i \in [m]$ in parallel} 
            \STATE {Update local model: $\x_{t, i}^1 = \x_{t}$.}
            \FOR{$k = 1, \cdots, K$ (local update step)}
                \STATE {Compute an unbiased estimate $ \nabla f_i(\x_{t, i}^k, \xi_{t, i}^k)$ of $\nabla f_i(\x_{t, i}^k)$}. 
                \STATE {Local update: $\x_{t,i}^{k+1} = \x_{t, i}^k - \eta_L \nabla f_i(\x_{t, i}^k, \xi_{t, i}^k)$.}
            \ENDFOR \\
            \STATE {Send $\Delta_t^i = \sum_{k \in [K]} \nabla f_i(\x_{t, i}^k, \xi_{t, i}^k)$ to the server.}
        \ENDFOR
        \STATE {Global Aggregation At Server:}
            \STATE { \hspace{20pt} Receive $\Delta_t^i, i \in [m]$.}
            % Let $\tilde{\Delta}_t = \frac{1}{m} \sum_{i \in [m]} \tilde{\Delta}_t^i$. \\% \centerline{$ (\star) \Delta_t = \frac{1}{\sum_{i \in S} p_i} \sum_{i \in S} p_i \Delta_t^i$} \\
            \STATE {\hspace{20pt} Server Update: $\x_{t+1} = \x_t - \frac{\eta \eta_L}{m} \sum_{i \in [m]} \Delta_t^i$.}
            \STATE {\hspace{20pt} Broadcasting $\x_{t+1}$ to clients.}   
    \ENDFOR
    \end{algorithmic}
\end{algorithm}

{\bf1) Problem Statement of Federated Learning and the FedAvg Algorithm:} 
The goal of FL is to solve the following optimization problem:
\begin{equation}\label{objective}
    \min_{\x \in \mathbb{R}^d} f(\x) := \frac{1}{m} \sum_{i=1}^{m} f_i(\x),
\end{equation}
where $m$ is the number of clients and $f_i(\x) \triangleq \mathbb{E}_{\xi_i \sim D_i}[f(\x, \xi_i)]$ is the local loss function associated with a local data distribution $D_i$.
A key challenge in FL stems from data heterogeneity, i.e., $D_i \neq D_j, \forall i \neq j$.
In FL, the standard and perhaps the most popular algorithm is the federated averaging (FedAvg) method.
Here in Algorithm~\ref{alg:fedavg}, we illustrate a more generalized version of the original FedAvg (GFedAvg) with separate learning rates on the client and server sides~\cite{Karimireddy2020SCAFFOLD,yang2021linearspeedup,reddi2021adaptive}.
Note that when $\eta = 1$, GFedAvg reduces to the original FedAvg~\cite{mcmahan2017communication}.
In each communication round of GFedAvg, each client performs local update steps and returns the update difference $\Delta_t^i$.
The server then aggregates these results and update the global model~\footnote{We assume all clients participate in the training at each communication round, but the results can be extended to that with (uniformly random sampled) subset of clients in each communication round\cite{yang2021linearspeedup,Karimireddy2020SCAFFOLD}.}
and the updated model parameters will then be retrieved by the clients to start the next round of local updates.

{\bf 2) Empirical Evidence of Fat-Tailed Noise Phenomenon in Federated Learning:} 
With the basics of FL and the FedAvg algorithm, we are now in a position to demonstrate the empirical evidence of the existence of fat-tailed noise in FL systems.
As mentioned earlier, in most performance analyses of FL algorithms, a common assumption is the bounded variance assumption of the local stochastic gradients: $\mathbb{E}[\| \nabla f_i(\x, \xi) - \nabla f_i(\x) \|^2] \leq \sigma^2$.
This assumption holds for all light-tailed noise distributions (i.e., the sub-exponential family) and some heavy-tailed distributions (e.g., log-normal, Weibull, and some Pareto distributions).

However, the finite-variance assumption fails to hold for many fat-tailed noise distributions.
For instance, for a random variable $X$, if its density $p(x)$ has a power-law tail decreasing as $1/|x|^{\alpha+1}$ with $\alpha \in (0,2)$, then only the $\alpha$-moment of this noise exists with $\alpha <2$.
To more precisely characterize fat-tailed distributions, in this paper, we adopt the notion of tail-index $\alpha$~\cite{simsekli2019tail} to parameterize fat-tailed and heavy-tailed distributions.
%, which $\alpha\in \left( 0,2 \right]$ and $\alpha=2$ corresponds the case with finite variance. 
More specifically, if the density of a random variable $X$'s distribution decays with a power law tail as $1/|x|^{\alpha+1}$ where $\alpha \in (0,2]$, then $\alpha$ is called the {\em tail-infex}.
This $\alpha$-parameter determines the behavior of the distribution:
the smaller the $\alpha$-value, the heavier the tail of the distribution.
Also, the $\alpha$-parameter also determines the moments:
$\mathbb{E}[|X|^r] < \infty$ if and only if $r<\alpha$, which implies that $X$ has infinite variance when $\alpha <2$, i.e., being {\em fat-tailed}.

%Most of the analysis for SGD-type methods utilize the ``light-tailed'' assumption (or its variants) defined as follows:
%\begin{defn} [Light-Tailed Random Vector.]
%    If a random vector $\y$ satisfies $\mathbb{P}\{\| \y - \mathbb{E}[\y] \| > c \} \leq 2 \exp(-\frac{c^2}{2 \sigma^2}), \forall c > 0$ where $\mathbb{E}[\y]$ exists and $\sigma > 0$, then we say $\y$ is light-tailed random vector.
%\end{defn}
%We note that, by assigning $\y = \nabla f(\x, \xi)$ (the stochastic gradient at $\x$), the light-tailed assumption directly leads to $\mathbb{E}[\exp(\| \nabla f(\x, \xi) - \nabla f(\x) \|^2 / \sigma^2)] \leq \exp(1)$ and subsequently bounded variance $\mathbb{E}[\| \nabla f(\x, \xi) - \nabla f(\x) \|^2] \leq \sigma^2$ by Jensen's inequality\footnote{In general, light-tailed assumption implies bounded variance, but the reverse does not hold. In this paper, we omit these subtle distinctions among light-tailed assumption and bounded variance assumption, and rather attentively focus on light-tails vs. heavy-tails.}.
%However, recent works suggest that this light-tailed assumption is too optimistic in practical models and heavy-tailed noise is widely observed in different deep neural networks~\cite{gurbuzbalaban2021heavy,hodgkinson2021multiplicative,simsekli2020fractional}.

Next, we investigate the tail property of model updates returned by clients in the GFedAvg algorithm.
Due to multiple local steps in the GFedAvg algorithm, we view the whole update vector $\Delta_t^i$ returned by each client, which we called ``pseudo-gradient,'' as a random vector and then analyze its statistical properties.
Note that in the special case with the number of local update $K=1$, $\Delta_t^i$ coincides with a single stochastic gradient of a random sample, (i.e., $\Delta_t^i = \nabla f_i(\x_{t}, \xi_{t})$).
%By doing so, we focus on the heavy-tailed property in the client's scale (i.e., $\Delta_t^i$) while typical analysis for SGD is performed in sample's level (i.e., $\nabla f(\x, \xi)$).
%If the local step is set as $K=1$, our setting coincides with that in typical SGD, (i.e., $\Delta_t^i = \nabla f_i(\x_{t}, \xi_{t})$).

% \textbf{100 clients(cifar10) + 143 clients(nlp) noise norm vs. density}

% \begin{figure}
%      \centering
%      \begin{subfigure}[b]{0.48\textwidth}
%          \centering
%          \includegraphics[width=\textwidth]{cifar10_noise_noniid2_localepoch2}
%          \caption{Non-IID case} %($p=2$)
%          \label{fig:noniid_2}
%      \end{subfigure}
%      \hfill
%      \begin{subfigure}[b]{0.48\textwidth}
%          \centering
%          \includegraphics[width=\textwidth]{cifar10_noise_noniid10_localepoch2}
%          \caption{IID case} %($p=10$)
%          \label{fig:noniid_10}
%      \end{subfigure}
%         \caption{Distributions of the norms of the pseudo-gradient noises computed with CNN on CIFAR-10 dataset. $m=100$ clients participate in the training.}
%         \label{fig:noise norm density}
% \end{figure}

%\begin{figure}[ht]
%\centering
%\includegraphics[width=0.65\linewidth]{cifar10_noise_density}
%\caption{\label{fig:noise norm density}Distributions of the norms of the pseudo-gradient noises computed with CNN on CIFAR-10 dataset in non-i.i.d. case (top) and i.i.d. case (bottom). $m=100$ clients participate in the training.}
%\end{figure}

\begin{figure*}[t!]
%\centering
\begin{minipage}{0.32\textwidth}
	\includegraphics[width=1\linewidth]{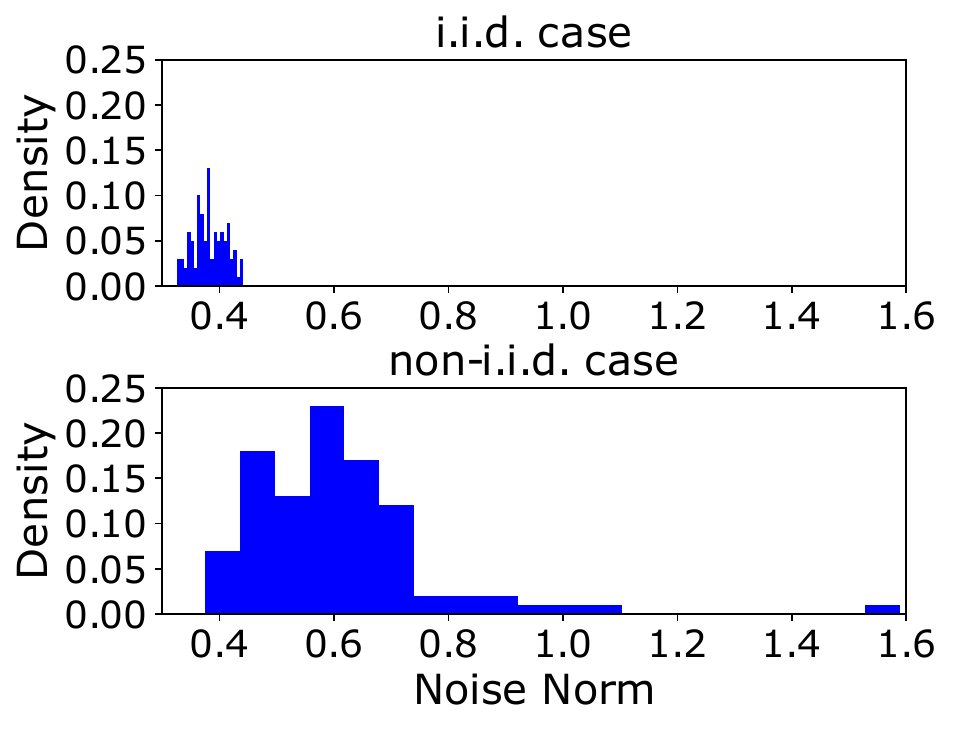}
	\caption{Distributions of the norms of the pseudo-gradient noises computed with CNN on CIFAR-10 dataset in i.i.d. case (top) and non-i.i.d. case (bottom). $m=100$ clients participate in the training.}
	\label{fig:noise norm density}
  \end{minipage}
  \hspace{0.01\textwidth}
\begin{minipage}{0.32\textwidth}
	\includegraphics[width=1\linewidth]{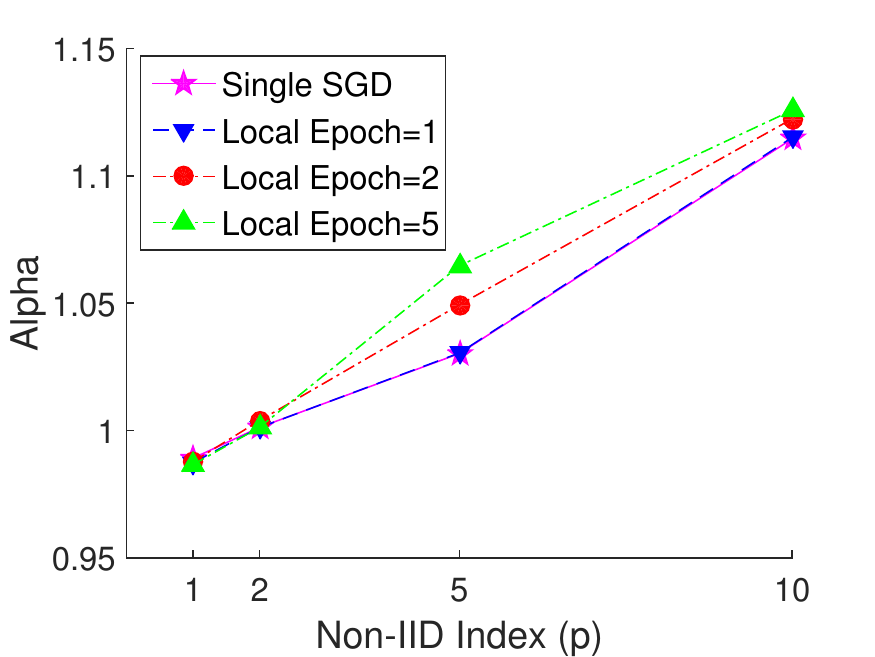}
	\caption{Estimation of $\alpha$ for CIFAR-10 dataset. The non-IID index $p$ represents the data heterogeneity level, and $p=10$ is the IID case. The smaller the $p$, the more heterogeneous the data across clients.}
	\label{alpha}
\end{minipage}
  \hspace{0.01\textwidth}
\begin{minipage}{0.32\textwidth}
	%\vspace{-.63in}
	\includegraphics[width=1\linewidth]{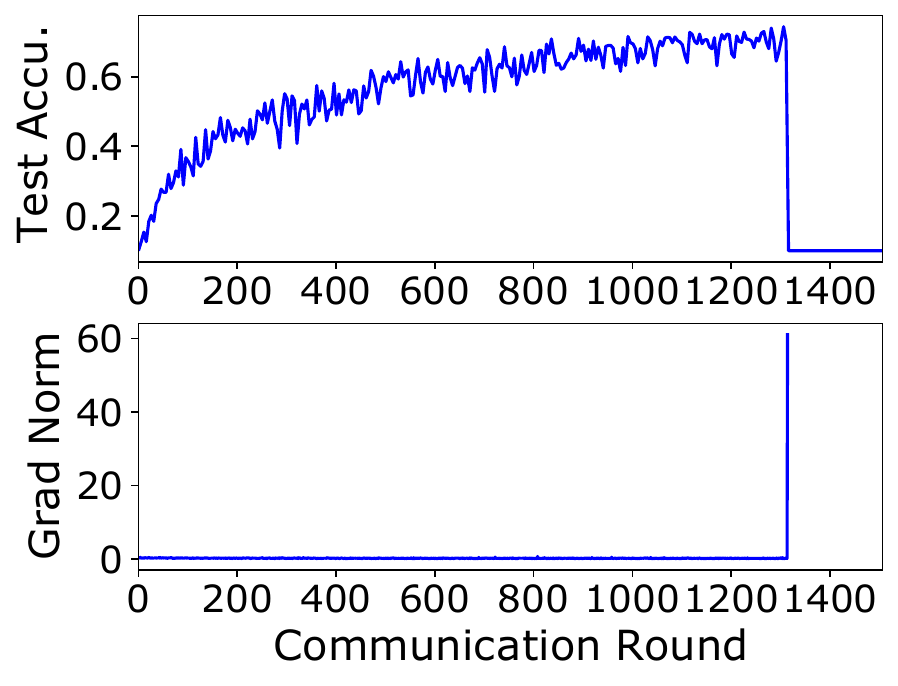}
	\caption{Catastrophic training failures happen when applying GFedAvg on CIFAR-10 dataset, where the test accuracy experiences a sudden and dramatic drop and the pseudo-gradient norm increases substantially.} 
	\label{fig:training failure}
\end{minipage}
% \vspace{-.25in}
\end{figure*}

We study the mismatch between the ``non-fat-tailed'' condition ($\alpha=2$) and the empirical behavior of the stochastic psudo-gradient noise. 
In Fig.~\ref{fig:noise norm density}, we illustrate the distributions of the norms of the stochastic pseudo-gradient noises computed with convolutional neural network (CNN) on the CIFAR-10 dataset in both i.i.d. and non-i.i.d. client dataset settings. 
We can clearly observe that the non-i.i.d. case exhibits a rather fat-tailed behavior, where the pseudo-gradient norm could be as large as $1.6$.
Although the i.i.d. case appears to have a much lighter tail, our detailed analysis shows that it still exhibits a fat-tailed behavior.
% \textbf{$\alpha$-index on data heterogeneity and local step}
%\begin{figure}[ht]
%\centering
%\includegraphics[width=0.5\linewidth]{Alpha}
%\caption{\label{alpha}Estimation of $\alpha$ for CIFAR-10 dataset. The non-IID index $p$ represents the data heterogeneity level, and $p=10$ is the IID case. The smaller the $p$, the more heterogeneous the data across clients.}
%\end{figure}
To see this, in Fig.~\ref{alpha}, we estimate $\alpha$-value for the CIFAR-10 dataset in different scenarios: 1) different local update steps, and 2) different data heterogeneity. 
We use a parameter $p$ to characterize the data heterogeneity level, with $p=10$ corresponding to the i.i.d. case. 
The smaller the $p$, the more heterogeneous the data among clients. 
Fig.~\ref{alpha} shows that the $\alpha$-value is smaller than $1.15$ in all scenarios, and $\alpha$ increases as the non-i.i.d. index $p$ increases (i.e., closer to the i.i.d. case). 
This implies that the stochastic pseudo-gradient noise is fat-tailed and the ``fatness'' increases as the clients' data become more heterogeneous.

{\bf 3) The Impacts of Fat-Tailed Noise on Federated Learning:} 
Next, we show that the fat-tailed noise could lead to a ``catastrophic model failure'' (i.e., a sudden and dramatic drop of learning accuracy), consistent with previous observations in the FL literature~\cite{charles2021large}. 
To demonstrate this, we apply GFedAvg on the CIFAR-10 dataset and randomly sample five clients among $m=10$ clients in each communication round. 
In Fig.~\ref{fig:training failure}, we illustrate a trial where a catastrophic training failure occurred. 
Correspondingly, we can observe in Fig.~\ref{fig:training failure} a spike in the norm of the pseudo-gradient. 
This exceedingly large pseudo-gradient norm motivates us to apply the clipping technique to curtail the gradient updates.
% \textbf{model failure + norm}
% \begin{figure}
%      \centering
%      \begin{subfigure}[b]{0.48\textwidth}
%          \centering
%          \includegraphics[width=\textwidth]{cifar10_textaccuracy_noniid2}
%          \caption{Test accuracy}
%          \label{fig:accuracy}
%      \end{subfigure}
%      \hfill
%      \begin{subfigure}[b]{0.48\textwidth}
%          \centering
%          \includegraphics[width=\textwidth]{pseudo-gradient_norm_noniid2_localepoch1}
%          \caption{Norm of the pseudo-gradient}
%          \label{fig:pseudo-gradient norm}
%      \end{subfigure}
%         \caption{Catastrophic training failure when applying FedAvg to CIFAR-10 dataset.}
%         \label{fig:training failure}
% \end{figure}
%
%\begin{figure}[ht]
%\centering
%\includegraphics[width=0.65\linewidth]{model_failure_noniid2_localepoch1}
%\caption{\label{fig:training failure}Catastrophic training failure when applying FedAvg to CIFAR-10 dataset.}
%\end{figure}
It is also worth noting that even if the squared norm of stochastic gradient may not be infinitely large in practice (i.e., having a bounded support empirically), it could still be too large and cause catastrophic model failures.
%For the non-fat-tailed case with $\alpha = 2$, it could be regarded as Gaussian (or light-tailed) noise, where the second order moment is bounded.
%For $\alpha < 2$, we say that the stochastic gradient noise is {\em heavy-tailed}.
In fact, under fat-tailed noise, the FedAvg algorithm could {\em diverge}, which follows from the fact that there exists one function that SGD diverges under heavy-tailed noise (see Remark 1 in ~\cite{zhang2020adaptive}).
As a result, the returned value by one client might be exceedingly large, 
%(i.e., $\mathbb{E} [\| \Delta_t^i \|^2] = \infty$ ), 
leading to divergence of the FedAvg-type algorithms.

%In light of above results, in what follows, we discuss several insights and open problems:
%\begin{list}{\labelitemi}{\leftmargin=1em \itemindent=0em \itemsep=.2em}
%\item Fat-tailed (or heavy-tailed) noise and tail index $\alpha$. 
It is worth pointing out that, although we have empirically shown heavy/fat-tailed noise in FL for the first time in this paper, we are by no means the only one to have observed heavy-tailed or fat-tailed noise phenomenon property in learning. 
Previous works have also found heavy/fat-tailed noise phenomenon in centralized training with SGD-type algorithms. 
For example, the work in~\cite{simsekli2019tail} showed the heavy-tailed noise phenomenon while (centralized) training the AlexNet on CIFAR-10. 
Here, we adopt a procedure similar to that in~\cite{simsekli2019tail} to evaluate the tail index $\alpha$ of the noise norm distribution in FL. 
As indicated above, we also observe that the (pseudo-)stochastic gradient noise is heavy/fat-tailed rather than Gaussian. 

It is also worth noting that it remains controversial whether the heavy/fat-tailed noise phenomenon exists in all models and datasets.
For example, the work in~\cite{panigrahi2019non} showed that the stochastic gradient noise is Gaussian at least in the early phases of training, while \cite{meng2020dynamic} showed that the stationary distribution of stochastic gradient noise is heavy-tailed and state-dependent. 
Also, the evaluation methodologies of $\alpha$ could be different in different works with different statistical errors, thus leading to different observations~\cite{panigrahi2019non,zhang2020adaptive}. 
%So we believe the analysis of $\alpha$ in this paper is merely qualitative rather than quantitative, as the absolute value of $\alpha$ seems less meaningful.
We believe that the phenomenon of heavy/fat-tailed noise in training with SGD-type methods is an under-explored area that deserves more efforts from the community.

To conclude this section, we would also like to leave a caveat regarding catastrophic training failures. 
In this section, we have shown that, under heavy/fat-tailed noises, catastrophic training failures happen in FL training, which is consistent with the observations in large-cohort FL training~\cite{charles2021large}. 
However, this does not necessarily mean that all FL trainings will suffer from catastrophic failures. 
Sometimes, such catastrophic failures may not happen at all (see the appendix for such empirical evidence). 
Here, we hypothesize that the heavy/fat-tailed noise phenomenon in FL is highly correlated with catastrophic failures in FL.
This is based on our subsequent observations that such catastrophic failures in FL can be effectively mitigated by employing clipping methods.
However, whether or not the heavy/fat-tailed noise phenomenon is truly the culprit for catastrophic failures still needs further investigations.
Nonetheless, the mere existence of such a correlation between heavy/fat-tailed noise and catastrophic failures in FL warrants our study on mitigating heavy/fat-tailed noise in this paper.

% !TEX TS-program = latex
% !TEX root = main.tex

\section{The \alg algorithmic framework for fat-tailed federated learning} \label{sec:alg}
Given the evidence of fat-tailed noise in FL and its potential catastrophic training failure as shown in Section~\ref{sec:mot}, there is a compelling need to design an efficient FL algorithm with provable convergence guarantee under fat-tailed noise in FL.
Interestingly, the observation of an exceedingly large pseudo-gradient norm in Fig.~\ref{fig:training failure} suggests a natural idea to mitigate fat-tailed noise: {\em clipping}.
Toward this end, in Section~\ref{subsec:alg} we first propose a clipping-based algorithmic framework called \algns, which contains two variants: \alg per-round (\algprns) and \alg per-iteration (\algpins). 
Then in Section~\ref{subsec:analysis}, we analyze their convergence rate performances.

\subsection{The \algpr and \algpi algorithms} \label{subsec:alg}

\begin{algorithm}[t!]
    \caption{The \algpr Algorithm.} \label{alg:cpr} 
    \begin{algorithmic}[1]
    \STATE $\text{Initialize } \x_1$. 
    \FOR{$t = 1, \cdots, T$ (communication round)} 
        \FOR{each client $i \in [m]$ in parallel} 
            \STATE {Update local model: $\x_{t, i}^1 = \x_{t}$.}
            \FOR{$k = 1, \cdots, K$ (local update step)}
                \STATE {Compute an unbiased estimate $ \nabla f_i(\x_{t, i}^k, \xi_{t, i}^k)$ of $\nabla f_i(\x_{t, i}^k)$}.
                % \STATE \colorbox{lightgray}{Clipping: $\tilde{\nabla} f_i(\x_{t, i}^k, \xi_{t, i}^k) = \min{ \{1, \frac{\lambda}{\| \nabla f_i(\x_{t, i}^k, \xi_{t, i}^k) \|}} \} \nabla f_i(\x_{t, i}^k, \xi_{t, i}^k)$.} 
                \STATE {Local update: $\x_{t, i}^{k+1} = \x_{t, i}^k - \eta_L \nabla f_i(\x_{t, i}^k, \xi_{t, i}^k)$.}
            \ENDFOR \\
            \STATE {Let $\Delta_t^i = \sum_{k \in [K]} \nabla f_i(\x_{t, i}^k, \xi_{t, i}^k)$ .}
            \STATE \colorbox{lightgray}{ {\bf Clipping:} $\tilde{\Delta}_t^i = \min{ \{1, \frac{\lambda}{\| \Delta_t^i \|} \}} \Delta_t^i$, where $\Delta_t^i = \sum_{k \in [K]} \nabla f_i(\x_{t, i}^k, \xi_{t, i}^k)$.}
            \STATE { Send $\tilde{\Delta}_t^i$ to the server.}
        \ENDFOR
        \STATE {Global Aggregation At Server:}
            \STATE { \hspace{20pt} Receive $\tilde{\Delta}_t^i, i \in [m]$.}
            % Let $\tilde{\Delta}_t = \frac{1}{m} \sum_{i \in [m]} \tilde{\Delta}_t^i$. \\% \centerline{$ (\star) \Delta_t = \frac{1}{\sum_{i \in S} p_i} \sum_{i \in S} p_i \Delta_t^i$} \\
            \STATE {\hspace{20pt} Server Update: $\x_{t+1} = \x_t - \frac{\eta \eta_L}{m} \sum_{i \in [m]} \tilde{\Delta}_t^i$.}
            \STATE {\hspace{20pt} Broadcasting $\x_{t+1}$ to clients.}   
    \ENDFOR
    \end{algorithmic}
\end{algorithm}

\begin{algorithm}[t!]
    \caption{The \algpi Algorithm.} \label{alg:cpi} 
    \begin{algorithmic}[1]
    \STATE $\text{Initialize } \x_1$. 
    \FOR{$t = 1, \cdots, T$ (communication round)} 
        \FOR{each client $i \in [m]$ in parallel} 
            \STATE {Update local model: $\x_{t, i}^1 = \x_{t}$.}
            \FOR{$k = 1, \cdots, K$ (local update step)}
                \STATE {Compute an unbiased estimate $ \nabla f_i(\x_{t, i}^k, \xi_{t, i}^k)$ of $\nabla f_i(\x_{t, i}^k)$}.
                \STATE \colorbox{lightgray}{{\bf Clipping:} $\tilde{\nabla} f_i(\x_{t, i}^k, \xi_{t, i}^k) = \min{ \{1, \frac{\lambda}{\| \nabla f_i(\x_{t, i}^k, \xi_{t, i}^k) \|}} \} \nabla f_i(\x_{t, i}^k, \xi_{t, i}^k)$.} 
                \STATE {Local update: $\x_{t, i}^{k+1} = \x_{t, i}^k - \eta_L \tilde{\nabla} f_i(\x_{t, i}^k, \xi_{t, i}^k)$.}
            \ENDFOR \\
            \STATE {Send $\tilde{\Delta}_t^i = \sum_{k \in [K]} \tilde{\nabla} f_i(\x_{t, i}^k, \xi_{t, i}^k)$ to the server.}
        \ENDFOR
        \STATE {Global Aggregation At Server:}
            \STATE { \hspace{20pt} Receive $\tilde{\Delta}_t^i, i \in [m]$.}
            % Let $\tilde{\Delta}_t = \frac{1}{m} \sum_{i \in [m]} \tilde{\Delta}_t^i$. \\% \centerline{$ (\star) \Delta_t = \frac{1}{\sum_{i \in S} p_i} \sum_{i \in S} p_i \Delta_t^i$} \\
            \STATE {\hspace{20pt} Server Update: $\x_{t+1} = \x_t - \frac{\eta \eta_L}{m} \sum_{i \in [m]} \tilde{\Delta}_t^i$.}
            \STATE {\hspace{20pt} Broadcasting $\x_{t+1}$ to clients.}   
    \ENDFOR
    \end{algorithmic}
\end{algorithm}

We illustrate the \algpr and \algpi algorithms in Algorithms~\ref{alg:cpr} and \ref{alg:cpi}, respectively.
It can be seen that both \algpr and \algpi share a similar algorithmic structure with GFedAvg, with the key differences lying in the additional clipping operations.
In \algprns, each client performs a clipping in each communication round on the returned $\Delta_t^i$:
\begin{align}
    \tilde{\Delta}_t^i = \min\bigg\{ 1, \frac{\lambda}{\| \Delta_t^i \|}  \bigg\} \Delta_t^i, \label{eq:cpr_clipping}
\end{align}
and then sends $\tilde{\Delta}_t^i$ instead of $\Delta_t^i$ to the server (Line~10 in Algorithm~\ref{alg:cpr}).
By contrast, in \algpins, each client clips the stochastic gradient before each local update step (Line~7 in Algorithm~\ref{alg:cpi}):
\begin{align}
    &\tilde{\nabla} f_i(\x_{t, i}^k, \xi_{t, i}^k) = \min \bigg\{1, \frac{\lambda}{\| \nabla f_i(\x_{t, i}^k, \xi_{t, i}^k) \|} \bigg\} \nabla f_i(\x_{t, i}^k, \xi_{t, i}^k), \\
    &\x_{t, i}^{k+1} = \x_{t, i}^k - \eta_L \tilde{\nabla} f_i(\x_{t, i}^k, \xi_{t, i}^k). \label{eq:cpi_clipping}
\end{align}
Then, $\tilde{\Delta}_t^i = \sum_{k \in [K]} \tilde{\nabla} f_i(\x_{t, i}^k, \xi_{t, i}^k)$ is sent to the server for aggregation (Line~10 in Algorithm~\ref{alg:cpi}).

\subsection{Convergence analysis of the \alg algorithms} \label{subsec:analysis}
Before conducting the convergence analysis for the \alg algorithms, we first state two standard assumptions that are commonly used in the literature of first-order stochastic methods.

\begin{assum}[$L$-Lipschitz Continuous Gradient] \label{assum_smooth}
	There exists a constant $L > 0$, such that $ \| \nabla f_i(\x) - \nabla f_i(\y) \| \leq L \| \x - \y \|,
	\forall \x, \y \in \mathbb{R}^d, and \ i \in [m]$.
\end{assum}
\begin{assum}[Unbiased Local Gradient Estimator] \label{assum_unbias}
	The local gradient estimator is unbiased, i.e.,
	$\mathbb{E} [\nabla f_i(\x, \xi)] = \nabla f_i(\x)$, $\forall i \in [m]$, where $\xi$ is a random local data sample at the $i$-th worker.
\end{assum}

Next, we state the key bounded $\alpha$-moment assumption for {\em fat-tailed} the stochastic first-order oracle, which leverages the notion of tail-index introduced in Section~\ref{sec:mot}:
%\footnote{We give a detailed comparison among different commonly-used assumptions for FL in Appendix.}.
\begin{assum}[Bounded $\alpha$-Moment] \label{assum_moment}
	There exists a real number $\alpha \in (1, 2]$ and a constant $G \geq 0$, such that $\mathbb{E} [\| \nabla f_i(\x, \xi)||^\alpha] \leq G^\alpha$, $\forall i \in [m], \x \in \mathbb{R}^d$.
\end{assum}

% -------------------------------
%           CPR-FedAvg
% -------------------------------
\textbf{1) Convergence Rates of the \algpr Algorithm:}
We first state the convergence rates of \algpr for $\mu$-strongly convex and non-convex objective functions.
\begin{restatable}{theorem}{ClippingPerRoundBoundedConvex}{\em (Convergence Rate of \algpr in the Strongly Convex Case)}
	\label{thm:ClippingPerRoundBoundedConvex}
	Suppose that $f(\cdot)$ is a $\mu$-strongly convex function. Under Assumptions~\ref{assum_smooth}--\ref{assum_moment}, if $\eta \eta_L K \geq \frac{2}{\mu T}$,
	then the output $\bar{\x}_T $ of \algpr being chosen in such a way that $\bar{\x}_T = \x_t$ with probability $\frac{w_t}{\sum_{j \in [T]} w_j}$, where $w_t = (1 - \frac{1}{2} \mu \eta \eta_L K)^{1 - t} $, satisfies:
		\begin{align*}
			f(\bar{\x}_T) - f(\x^*)
			&\leq \frac{\mu}{2} \exp{\left( - \frac{1}{2} \mu \eta \eta_L K T \right)} \!+\! \frac{\eta \eta_L K}{2} G^\alpha \lambda^{2-\alpha} \!+\! \frac{4}{\mu} \left[2 G^{2 \alpha} \lambda^{2 - 2 \alpha} \!+\! 2 L^2 \eta_L^2 K^2 G^{\alpha} \lambda^{2 - \alpha} \right],
		\end{align*}
	where $\x^*$ denotes the global optimal solution.
	Further, let $\eta \eta_L K = \frac{2c}{\mu} \frac{\ln(T)}{mKT}$, where $c \geq 1$ is a constant satisfying $m^{\frac{2-2\alpha}{\alpha}} K^{\frac{2}{\alpha}} T^{c + \frac{2-2\alpha}{\alpha}}  \geq 1 $, and let $\eta_L \leq (mKT)^{\frac{1-\alpha}{\alpha}}$. It then follows that $$f(\bar{\x}_T) - f(\x^*)
		= \mathcal{O}((mT)^{\frac{2-2\alpha}{\alpha}} K^{\frac{2}{\alpha}}).$$
\end{restatable}

\begin{restatable}{theorem}{ClippingPerRoundBoundedNonconvex}{\em (Convergence Rate of \algpr in the Nonconvex Case)}
	\label{thm:ClippingPerRoundBoundedNonconvex}
	Suppose that $f(\cdot)$ is a nonconvex function. Under Assumptions~\ref{assum_smooth}--\ref{assum_moment}, if $\eta \eta_L K L \leq 1$,
	then the sequence of outputs $\{ \x_k \}$ generated by \algpr satisfies:
	\begin{align*}
		\min_{t \in [T]} \mathbb{E} \| \nabla f(\x_t) \|^2 &\leq \frac{2 \left( f(\x_1) - f(x_T) \right)}{\eta \eta_L K T} \!+\! \left( L^2 \eta_L^2 K^2 G^2 \!+\! K^2 G^{2\alpha} \lambda^{-2(\alpha -1)} \!+\! L \eta_L K^2 G^{1+\alpha} \lambda^{1-\alpha} \right) \\
		&\quad + \frac{L \eta \eta_L}{m} \left(K G^\alpha \lambda^{2-\alpha} \right).
	\end{align*}
	Further, choosing learning rates and clipping parameter in such a way that $\eta \eta_L = m^{\frac{2 \alpha - 2}{3 \alpha - 2}} K^{\frac{- \alpha - 2}{3 \alpha - 2}} T^{\frac{-\alpha}{3 \alpha - 2}}, \eta_L \leq (mT)^{\frac{1 - \alpha}{3\alpha - 2}} K^{\frac{4 - 4 \alpha}{3 \alpha - 2}}$, and $\lambda = (mK^4T)^{\frac{1}{3 \alpha - 2}}$, we have 
	\begin{align*}
		\min_{t \in [T]} \mathbb{E} \| \nabla f(\x_t) \|^2 = \mathcal{O}((mT)^{\frac{2 - 2 \alpha}{3 \alpha - 2}} K^{\frac{4 - 2 \alpha}{3 \alpha - 2}}).
	\end{align*}
\end{restatable}

\begin{rem}{\em 
    We note that the above convergence rates for \algpr does not generalize the results of FedAvg when $\alpha = 2$ (non-fat-tailed noise).
    Specifically, FedAvg is able to achieve $\tilde{\mathcal{O}}((mKT)^{-1})$ and $\mathcal{O}((mKT)^{-\frac{1}{4}})$ convergence rates for strongly convex ($f(\x) - f(\x^*) \leq \epsilon$) and non-convex function ($\| \nabla f(\x) \| \leq \epsilon$), respectively~\cite{Karimireddy2020SCAFFOLD,arjevani2019lower}.
%    , which are the optimal rates by ignoring logarithmic terms
    In contrast, \algpr achieves $\mathcal{O}((mT)^{-1}K)$ and $\mathcal{O}((mT)^{-\frac{1}{4}})$ for strongly-convex and non-convex functions, respectively.
    These two rates are consistent with those of FedAvg in terms of $m$ and $T$, but not in terms of $K$.
   }
\end{rem}

Interestingly, with a separate proof for non-fat-tailed noise ($\alpha = 2$), we can show that clipping does not affect the dependence on $K$ in the convergence rates.
Thus, \algpr has the {\em same} convergence rates as those of FedAvg.
Due to space limitation, we state an informal version of these theorems here.
The full versions of Theorem
~\ref{thm:ClippingPerRoundBoundedGaussian} ~\ref{thm:ClippingPerRoundBoundedConvexGaussian}
and their proofs are formally stated in Appendix.

\textbf{Theorem 6 $\&$ 7 (informal)} (Convergence Rates of \algpr for Non-Fat-Tailed Noise):
{\em For $\alpha = 2$, CPR-FedAvg achieves convergence rate $\tilde{\mathcal{O}}((mKT)^{-1})$ for strongly-convex and $\mathcal{O}((mKT)^{-\frac{1}{4}})$ for non-convex functions, respectively.}

% Due to space limitation, we formally state Theorem~\ref{thm:ClippingPerRoundBoundedGaussian} ~\ref{thm:ClippingPerRoundBoundedConvexGaussian} for $\alpha = 2$ and their proofs in Appendix.

% -------------------------------
%           CPI-FedAvg
% -------------------------------
\textbf{2) Convergence Rate of the \algpi Algorithm:}
Next, we provide the convergence rates of \algpi for $\mu$-strongly convex and non-convex objective functions.
\begin{restatable} {theorem}{ClippingPerIterationBoundedConvex}{\em (Convergence Rate of \algpi in the Strongly Convex Case)}
	\label{thm:ClippingPerIterationBoundedConvex}
	Suppose that $f(\cdot)$ is a $\mu$-strongly convex function.
	Under Assumptions~\ref{assum_smooth}--\ref{assum_moment}, if $\eta \eta_L K \geq \frac{2}{\mu T}$,
	then the output $\bar{\x}_T $ of \algpi being chosen in such a way that $\bar{\x}_T = \x_t$ with probability $\frac{w_t}{\sum_{j \in [T]} w_j}$, where $w_t = (1 - \frac{1}{2} \mu \eta \eta_L K)^{1 - t}$, satisfies:
		\begin{align*}
			f(\bar{\x}_T) - f(\x^*)
			&\leq \frac{\mu}{2} \exp{\left( - \frac{1}{2} \mu \eta \eta_L K T \right)} + \frac{\eta \eta_L K}{2} G^\alpha \lambda^{2-\alpha} \\
			& \quad + \frac{4}{\mu} [2 G^{2 \alpha} \lambda^{-2(\alpha - 1)} + 2 L^2 \eta_L^2 K^2 G^{\alpha} \lambda^{2 - \alpha}],
		\end{align*}
	where $\x^*$ denotes the global optimal solution.
	Further, let $\eta \eta_L K = \frac{2c}{\mu} \frac{\ln(T)}{mKT}$, where $c \geq 1$ is a constant satisfying $(mK)^{\frac{2-2\alpha}{\alpha}} T^{c + \frac{2-2\alpha}{\alpha}} \geq 1$, and let $\lambda = (mKT)^{\frac{1}{\alpha}}$, and $\eta_L \leq (mT)^{-\frac{1}{2}} K^{- \frac{3}{2}}$). It then follows that
	\begin{align*}
		f(\bar{\x}_T) - f(\x^*)
		&= \tilde{\mathcal{O}}((mKT)^{\frac{2 - 2 \alpha}{\alpha}}).
	\end{align*}
\end{restatable}

\begin{restatable}{theorem}{ClippingPerIterationBoundedNonconvex}{\em (Convergence Rate of \algpi in the Nonconvex Case)}
	\label{thm:ClippingPerIterationBoundedNonconvex}
	Suppose that $f(\cdot)$ is a non-convex function.
	Under Assumptions~\ref{assum_smooth}--\ref{assum_moment}, if $\eta \eta_L K L \leq 1$,
	then the sequence of outputs $\{ \x_k \}$ generated by \algpi satisfies:
		\begin{align*}
			\min_{t \in [T]} \mathbb{E} \| \nabla f(\x_t) \|^2 &\leq \frac{2 \left( f(\x_1) - f(x_T) \right)}{\eta \eta_L K T} + \left( 2 G^{2 \alpha} \lambda^{-2(\alpha - 1)} + 2 L^2 \eta_L^2 K^2 G^\alpha \lambda^{2 - \alpha} \right) \nonumber \\
			&\quad + \frac{L \eta \eta_L}{m} \left(G^\alpha \lambda^{2-\alpha} \right).
		\end{align*}
		Further, choosing learning rates and clipping parameter in such a way that $\eta \eta_L = m^{\frac{2 \alpha - 2}{3 \alpha - 2}} (KT)^{\frac{-\alpha}{3 \alpha - 2}}, \eta_L \leq (mKT)^{\frac{- \alpha}{6 \alpha - 4}}$, and $\lambda = (mKT)^{\frac{1}{3 \alpha - 2}}$, we have
		\begin{align*}
			\min_{t \in [T]} \mathbb{E} \| \nabla f(\x_t) \|^2 \leq \mathcal{O}((mKT)^{\frac{2 - 2 \alpha}{3 \alpha - 2}}).
		\end{align*}
\end{restatable}

\begin{rem}{\em
	In comparison to \algprns, convergence rates of \algpi generalize the results of FedAvg for the non-fat-tailed noise case (i.e., $\alpha = 2$).
	Specifically, when $\alpha=2$, \algpi achieves $\tilde{\mathcal{O}}((mKT)^{-1})$ and $\mathcal{O}((mKT)^{-\frac{1}{4}})$ convergence rates for strongly convex and nonconvex objective functions, respectively.
	These two convergence rates are consistent with those of FedAvg in terms of $m$, $K$ and $T$ (ignoring logarithmic factors in the strongly-convex case).
}
\end{rem}

Next, we show that the convergence rates for \algpi is {\em order-optimal} for $\alpha \in (1, 2]$ by proving the following lower bounds.
\begin{cor}[Convergence Rate Lower Bound]
	Given any $\alpha \in (1, 2]$, for any potentially randomized algorithm, there exists a stochastic strongly-convex function satisfying Assumption~\ref{assum_moment} with $G \leq 1$, such that the output of $\x_T$ after $T$ communication rounds has an expected error lower bounded by 
	\begin{align*}
		\mathbb{E}[f(\x_t)] - f(\x_*) = \Omega((mKT)^{\frac{2 - 2 \alpha}{\alpha}}).
	\end{align*}
Also, there exists a non-convex function satisfying Assumption~\ref{assum_moment}, such that the output of $\x_T$ after $T$ communication rounds has an expected error lower bounded by 
	\begin{align*}
		%\mathbb{E} \| \nabla f(\x_t) \|^2 = 
		\mathbb{E} [\| \nabla f(\x_t) \|]^2 = \Omega((mKT)^{\frac{2 - 2 \alpha}{3 \alpha - 2}}).
	\end{align*}
\end{cor}

With $T$ communication rounds, the total number of stochastic gradients is $mKT$. 
Thus, the lower bounds above can be obtained from the centralized SGD with fat-tailed noise~\cite[Theorems~5 and 6]{zhang2020adaptive}).
Clearly, the above lower bounds imply the optimality of the convergence rates of \algpi.

\section{Numerical results} \label{sec:numerical}

In this section, we conduct numerical experiments to verify the theoretical findings in Section~\ref{sec:alg} using 1) a synthetic function, 2) a convolutional neural network (CNN) with two convolutional layers on CIFAR-10 dataset~\cite{krizhevsky2009learning}, and 3) RNN on Shakespeare dataset. 
Due to space limitation, we relegate experiment details and extra experimental results to the supplementary material.

% \begin{minipage}{0.32\textwidth}
% 	\includegraphics[width=1\linewidth]{Gaussian}
% 	\caption{Convergence comparisons of FedAvg, \algpi, and \algpr for solving strongly convex models: synthetic data with $\xi$ having Gaussian tails (non-fat). }
% 	\label{fig:Gaussian}
%   \end{minipage}

\begin{figure*}[t!]
\begin{minipage}{0.32\textwidth}
	\includegraphics[width=1\linewidth]{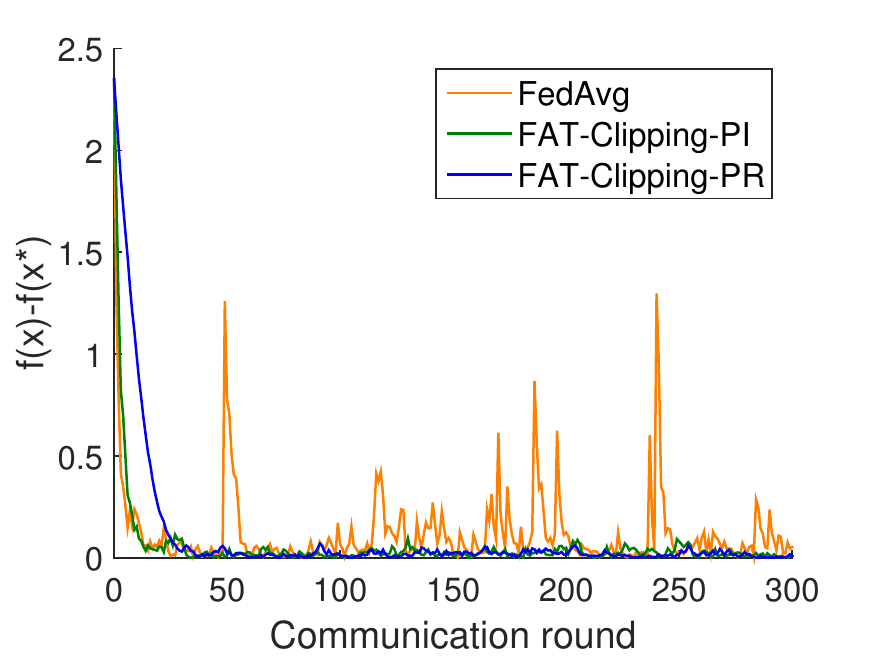}
	\caption{Convergence comparisons of FedAvg, \algpins, and \algpr for solving strongly convex models: synthetic data with $\xi$ having Cauchy tails (fat). }
	\label{fig:Cauchy}
  \end{minipage}
  \hspace{0.01\textwidth}
\begin{minipage}{0.32\textwidth}
	\includegraphics[width=1\linewidth]{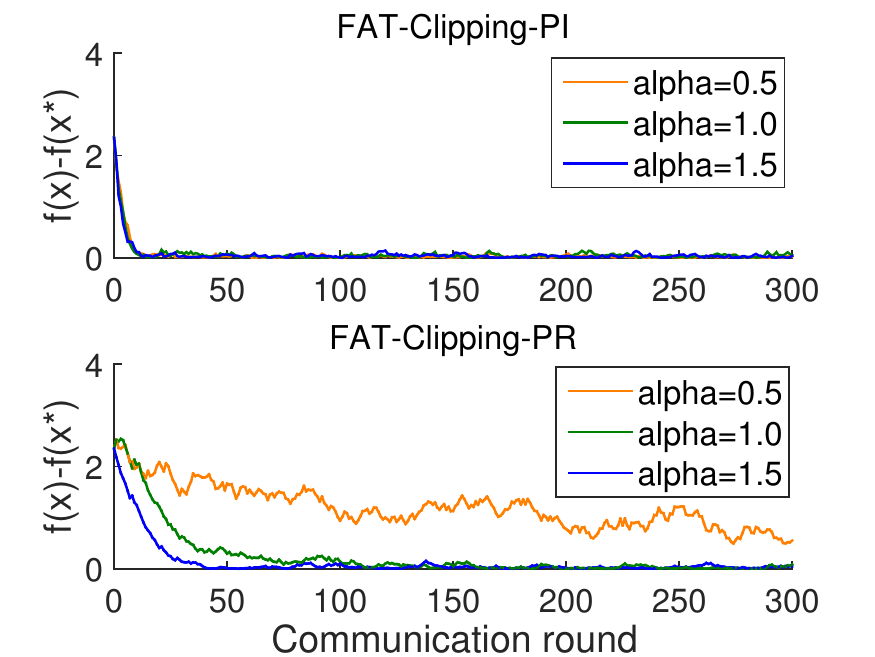}
	\caption{Convergence comparisons of \algpins, and \algpr for solving strongly convex models: synthetic data with $\xi$ having different fat tails represented by $\alpha$.}
	\label{fig:FAT_Clipping_alpha}
\end{minipage}
  \hspace{0.01\textwidth}
\begin{minipage}{0.32\textwidth}
	\includegraphics[width=1\linewidth]{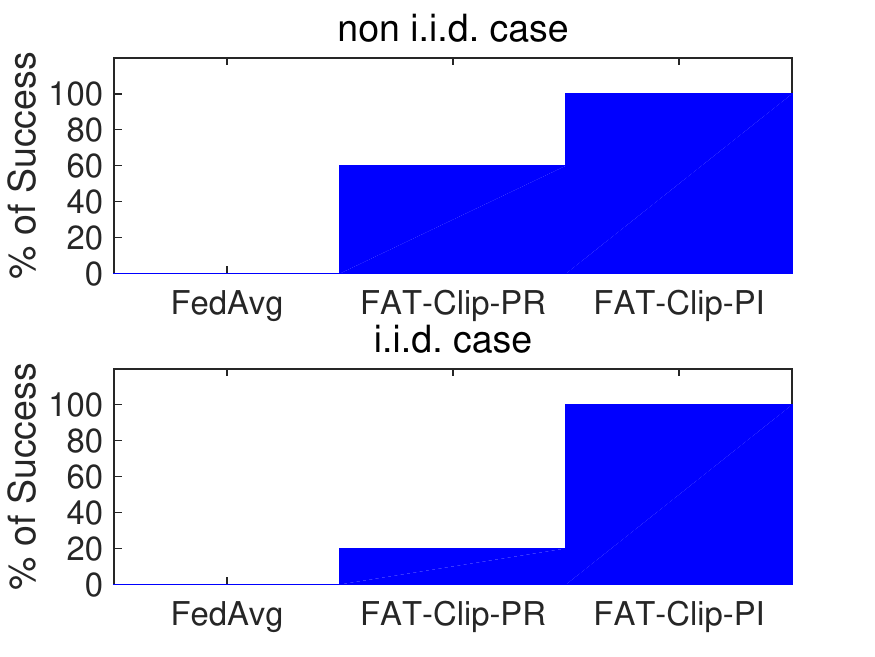}
	\caption{Percentage of successful training over 5 trials when applying FedAvg, \algpr and \algpi to CIFAR-10 dataset in non-i.i.d. case and i.i.d. case.} 
	\label{fig:cifar-10}
\end{minipage}
% \vspace{-.25in}
\end{figure*}

\textbf{1) Strongly Convex Model with Synthetic Data:} 
We consider a strongly convex model for Problem~(\ref{objective}) as follows:
%\begin{equation}\label{syn_fun}
$f_{i}\left( x \right)=\mathbb{E}_{\xi}\left[ f\left( x,\xi \right) \right]$ and $f\left( x,\xi \right)=\frac{1}{2}\left\| x \right\|^{2}+\left\langle \xi,x \right\rangle$,
%\end{equation}
where $\xi$ is a random vector. 
We compare FedAvg, \algpins, and \algprns, where the noise $\xi$ is Cauchy distributed (fat-tailed). 
Also, we compare \algpi and \algpr with $\xi$ having different tail-indexes ($\alpha=0.5, 1.0, \text{ and }1.5$).
%We test two different distributions of $\xi$ with the same mean and variance, which are Gaussian ($\alpha=2$, non-fat-tailed) and Cauchy ($\alpha<2$, fat-tailed) distributions. 
For each distribution, we use the same experimental setup, and $m=5$ clients participate in the training.
We show the trajectories of FedAvg, \algpins, and \algpr for solving Problem~(\ref{objective}) with $\xi$ having Cauchy tails in Fig.~\ref{fig:Cauchy} and with $\xi$ having different $\alpha$-values in Fig.~\ref{fig:FAT_Clipping_alpha}.
We can clearly observe from Fig.~\ref{fig:Cauchy} that \algpi and \algpr converge rapidly in the Cauchy case, and \algpi converges faster than \algpr as our theoretical results predict. 
In contrast, FedAvg is not convergent in the Cauchy case.
%, where the noise is fat-tailed.
%In contrast, while FedAvg converges in the light-tailed Gaussian case as expected, it is not convergent in the Cauchy case, where the noise is fat-tailed.
%In contrast, with clipping being applied during the training, both \algpi and \algpr can converge regardless of the noise distribution. 
In Fig.~\ref{fig:FAT_Clipping_alpha}, we can see that the convergence processes of \algpr and \algpi become slower as the $\alpha$-value increases as our theoretical results predict, but the differences in \algpi are much less obvious compared to those of \algprns.
%we show that \algpi outperforms \algpr, and it is more stable in fat-tailed cases. }

\textbf{2) CNN (Non-convex Model) on the CIFAR-10:}
This setting has $m=10$ clients in total, and five clients are randomly selected to participate in each round of the training. 
We compare \alg algorithms with FedAvg under different data heterogeneity.
To simulate data heterogeneity across clients, we distribute the data to each client in a label-based partition following the same procedure as in existing works (e.g.,~\cite{mcmahan2017communication, li2020convergence, yang2021linearspeedup}): we use a parameter $p$ to represent the number of labeled classes in each client, with $p=10$ corresponding to the i.i.d. case and the rest corresponding to non-i.i.d. cases. 
The smaller the $p$-value, the more heterogeneous the data across clients. 
In Fig.~\ref{fig:cifar-10}, we present the percentage of successful training over 5 trials when applying FedAvg, \algpr and \algpi on CIFAR-10 in non-i.i.d. case ($p=2$) and i.i.d. case ($p=10$). 
\algpi has 100\% successful rates (i.e., no catastrophic model failures) in both non-i.i.d. and i.i.d cases, and \algpr has 60\% and 20\% successful rates in non-i.i.d. and i.i.d. cases, respectively. 
However, FedAvg fails in all 5 trials. 
Thus, compared to FedAvg, \alg methods (\algpi in particular) significantly reduce catastrophic training failures. 

% !TEX root = main.tex
\section{Conclusions and future work} \label{sec:conclusion}
In this paper, we investigated the problem of designing efficient federated learning algorithms with convergence performance guarantee in the presence of fat-tailed noise in the stochastic first-order oracles.
We first showed empirical evidence that fat-tailed noise in federated learning can be induced by data heterogeneity and local update steps.
%For more heterogeneous data and local steps, the heavy-tailed property becomes more prominent.
To address the fat-tailed noise challenge in FL algorithm design, we proposed a clipping-based algorithmic framework called \alg.
The \alg framework contains two variants \algpr and \algpi, which perform clipping operations in each communication round and in each local update step, respectively.
Then, we derived the convergence rate bounds of \algpr and \algpi for strongly convex and non-convex loss functions under fat-tailed noise.
Not only does our work shed light on theoretical understanding of FL under fat-tailed noise, it also opens the doors to many new interesting questions in FL systems that experience fat-tailed noise.

\section*{Acknowledgements}
This work has been supported in part by NSF grants CAREER CNS-2110259, CNS-2112471, ECCS-2140277, and CCF-2110252.

%
%
%
% \newpage
\bibliographystyle{IEEEtran}{}
\bibliography{BIB/FederatedLearning,BIB/CommunicationEfficiency,BIB/HeavytailNoise,BIB/Teams}

\begin{thebibliography}{10}
\providecommand{\url}[1]{#1}
\csname url@samestyle\endcsname
\providecommand{\newblock}{\relax}
\providecommand{\bibinfo}[2]{#2}
\providecommand{\BIBentrySTDinterwordspacing}{\spaceskip=0pt\relax}
\providecommand{\BIBentryALTinterwordstretchfactor}{4}
\providecommand{\BIBentryALTinterwordspacing}{\spaceskip=\fontdimen2\font plus
\BIBentryALTinterwordstretchfactor\fontdimen3\font minus
  \fontdimen4\font\relax}
\providecommand{\BIBforeignlanguage}[2]{{%
\expandafter\ifx\csname l@#1\endcsname\relax
\typeout{** WARNING: IEEEtran.bst: No hyphenation pattern has been}%
\typeout{** loaded for the language `#1'. Using the pattern for}%
\typeout{** the default language instead.}%
\else
\language=\csname l@#1\endcsname
\fi
#2}}
\providecommand{\BIBdecl}{\relax}
\BIBdecl

\bibitem{mcmahan2017communication}
B.~McMahan, E.~Moore, D.~Ramage, S.~Hampson, and B.~A. y~Arcas,
  ``Communication-efficient learning of deep networks from decentralized
  data,'' in \emph{Artificial intelligence and statistics}.\hskip 1em plus
  0.5em minus 0.4em\relax PMLR, 2017, pp. 1273--1282.

\bibitem{Li2020fedprox}
T.~Li, A.~K. Sahu, M.~Zaheer, M.~Sanjabi, A.~Talwalkar, and V.~Smith,
  ``Federated optimization in heterogeneous networks,'' in \emph{Proceedings of
  Machine Learning and Systems}, I.~Dhillon, D.~Papailiopoulos, and V.~Sze,
  Eds., vol.~2, 2020, pp. 429--450.

\bibitem{Karimireddy2020SCAFFOLD}
S.~P. Karimireddy, S.~Kale, M.~Mohri, S.~Reddi, S.~Stich, and A.~T. Suresh,
  ``{SCAFFOLD}: Stochastic controlled averaging for federated learning,'' in
  \emph{Proceedings of the 37th International Conference on Machine Learning},
  ser. Proceedings of Machine Learning Research, H.~D. III and A.~Singh, Eds.,
  vol. 119.\hskip 1em plus 0.5em minus 0.4em\relax PMLR, 13--18 Jul 2020, pp.
  5132--5143.

\bibitem{wang2020fednova}
J.~Wang, Q.~Liu, H.~Liang, G.~Joshi, and H.~V. Poor, ``Tackling the objective
  inconsistency problem in heterogeneous federated optimization,''
  \emph{Advances in Neural Information Processing Systems}, vol.~33, 2020.

\bibitem{zhang2020fedpd}
X.~Zhang, M.~Hong, S.~Dhople, W.~Yin, and Y.~Liu, ``Fedpd: A federated learning
  framework with optimal rates and adaptivity to non-iid data,'' \emph{arXiv
  preprint arXiv:2005.11418}, 2020.

\bibitem{acar2021feddyn}
D.~A.~E. Acar, Y.~Zhao, R.~M. Navarro, M.~Mattina, P.~N. Whatmough, and
  V.~Saligrama, ``Federated learning based on dynamic regularization,'' in
  \emph{International Conference on Learning Representations}, 2021.

\bibitem{yang2021linearspeedup}
H.~Yang, M.~Fang, and J.~Liu, ``Achieving linear speedup with partial worker
  participation in non-{IID} federated learning,'' in \emph{International
  Conference on Learning Representations}, 2021.

\bibitem{reddi2021adaptive}
S.~J. Reddi, Z.~Charles, M.~Zaheer, Z.~Garrett, K.~Rush, J.~Konecn{\'y},
  S.~Kumar, and H.~B. McMahan, ``Adaptive federated optimization,'' in
  \emph{International Conference on Learning Representations}, 2021.

\bibitem{khanduri2021stem}
P.~Khanduri, P.~SHARMA, H.~Yang, M.~Hong, J.~Liu, K.~Rajawat, and P.~Varshney,
  ``{STEM}: A stochastic two-sided momentum algorithm achieving near-optimal
  sample and communication complexities for federated learning,'' in
  \emph{Thirty-Fifth Conference on Neural Information Processing Systems},
  2021.

\bibitem{gu2021fast}
X.~Gu, K.~Huang, J.~Zhang, and L.~Huang, ``Fast federated learning in the
  presence of arbitrary device unavailability,'' in \emph{Thirty-Fifth
  Conference on Neural Information Processing Systems}, 2021.

\bibitem{luo2021nofear}
M.~Luo, F.~Chen, D.~Hu, Y.~Zhang, J.~Liang, and J.~Feng, ``No fear of
  heterogeneity: Classifier calibration for federated learning with non-{IID}
  data,'' in \emph{Thirty-Fifth Conference on Neural Information Processing
  Systems}, 2021.

\bibitem{karimireddy2021mime}
S.~P. Karimireddy, M.~Jaggi, S.~Kale, M.~Mohri, S.~J. Reddi, S.~U. Stich, and
  A.~T. Suresh, ``Breaking the centralized barrier for cross-device federated
  learning,'' in \emph{Thirty-Fifth Conference on Neural Information Processing
  Systems}, 2021.

\bibitem{nair2022fundamentals}
J.~Nair, A.~Wierman, and B.~Zwart, \emph{The Fundamentals of Heavy
  Tails}.\hskip 1em plus 0.5em minus 0.4em\relax Cambridge University Press,
  2022, vol.~53.

\bibitem{taylor2016}
\BIBentryALTinterwordspacing
J.~E. Taylor. Heavy-tailed distributions. [Online]. Available:
  \url{https://math.la.asu.edu/~jtaylor/teaching/Spring2016/STP421/lectures/stable.pdf}
\BIBentrySTDinterwordspacing

\bibitem{simsekli2019tail}
U.~Simsekli, L.~Sagun, and M.~Gurbuzbalaban, ``A tail-index analysis of
  stochastic gradient noise in deep neural networks,'' in \emph{International
  Conference on Machine Learning}.\hskip 1em plus 0.5em minus 0.4em\relax PMLR,
  2019, pp. 5827--5837.

\bibitem{gurbuzbalaban2021heavy}
M.~Gurbuzbalaban, U.~Simsekli, and L.~Zhu, ``The heavy-tail phenomenon in
  sgd,'' in \emph{International Conference on Machine Learning}.\hskip 1em plus
  0.5em minus 0.4em\relax PMLR, 2021, pp. 3964--3975.

\bibitem{Zhang2020Why}
\BIBentryALTinterwordspacing
J.~Zhang, T.~He, S.~Sra, and A.~Jadbabaie, ``Why gradient clipping accelerates
  training: A theoretical justification for adaptivity,'' in
  \emph{International Conference on Learning Representations}, 2020. [Online].
  Available: \url{https://openreview.net/forum?id=BJgnXpVYwS}
\BIBentrySTDinterwordspacing

\bibitem{gorbunov2020stochastic}
E.~Gorbunov, M.~Danilova, and A.~Gasnikov, ``Stochastic optimization with
  heavy-tailed noise via accelerated gradient clipping,'' \emph{Advances in
  Neural Information Processing Systems}, vol.~33, pp. 15\,042--15\,053, 2020.

\bibitem{panigrahi2019non}
A.~Panigrahi, R.~Somani, N.~Goyal, and P.~Netrapalli, ``Non-gaussianity of
  stochastic gradient noise,'' \emph{arXiv preprint arXiv:1910.09626}, 2019.

\bibitem{charles2021large}
Z.~Charles, Z.~Garrett, Z.~Huo, S.~Shmulyian, and V.~Smith, ``On large-cohort
  training for federated learning,'' \emph{Advances in Neural Information
  Processing Systems}, vol.~34, 2021.

\bibitem{ghadimi2013stochastic}
S.~Ghadimi and G.~Lan, ``Stochastic first-and zeroth-order methods for
  nonconvex stochastic programming,'' \emph{SIAM Journal on Optimization},
  vol.~23, no.~4, pp. 2341--2368, 2013.

\bibitem{zhang2020adaptive}
J.~Zhang, S.~P. Karimireddy, A.~Veit, S.~Kim, S.~Reddi, S.~Kumar, and S.~Sra,
  ``Why are adaptive methods good for attention models?'' \emph{Advances in
  Neural Information Processing Systems}, vol.~33, pp. 15\,383--15\,393, 2020.

\bibitem{Wang2019adapCom}
\BIBentryALTinterwordspacing
J.~Wang and G.~Joshi, ``Adaptive communication strategies to achieve the best
  error-runtime trade-off in local-update sgd,'' in \emph{Proceedings of
  Machine Learning and Systems}, A.~Talwalkar, V.~Smith, and M.~Zaharia, Eds.,
  vol.~1, 2019, pp. 212--229. [Online]. Available:
  \url{https://proceedings.mlsys.org/paper/2019/file/c8ffe9a587b126f152ed3d89a146b445-Paper.pdf}
\BIBentrySTDinterwordspacing

\bibitem{yang2022anarchic}
H.~Yang, X.~Zhang, P.~Khanduri, and J.~Liu, ``Anarchic federated learning,'' in
  \emph{International Conference on Machine Learning}.\hskip 1em plus 0.5em
  minus 0.4em\relax PMLR, 2022, pp. 25\,331--25\,363.

\bibitem{lin2018don}
T.~Lin, S.~U. Stich, K.~K. Patel, and M.~Jaggi, ``Don't use large mini-batches,
  use local sgd,'' \emph{arXiv preprint arXiv:1808.07217}, 2018.

\bibitem{Zhang2022NETFLEETAL}
X.~Zhang, M.~Fang, Z.~Liu, H.~Yang, J.~Liu, and Z.~Zhu, ``Net-fleet: achieving
  linear convergence speedup for fully decentralized federated learning with
  heterogeneous data,'' \emph{Proceedings of the Twenty-Third International
  Symposium on Theory, Algorithmic Foundations, and Protocol Design for Mobile
  Networks and Mobile Computing}, 2022.

\bibitem{nguyen2019first}
T.~H. Nguyen, U.~Simsekli, M.~Gurbuzbalaban, and G.~Richard, ``First exit time
  analysis of stochastic gradient descent under heavy-tailed gradient noise,''
  \emph{Advances in neural information processing systems}, vol.~32, 2019.

\bibitem{simsekli2020fractional}
U.~Simsekli, L.~Zhu, Y.~W. Teh, and M.~Gurbuzbalaban, ``Fractional underdamped
  langevin dynamics: Retargeting sgd with momentum under heavy-tailed gradient
  noise,'' in \emph{International Conference on Machine Learning}.\hskip 1em
  plus 0.5em minus 0.4em\relax PMLR, 2020, pp. 8970--8980.

\bibitem{hodgkinson2021multiplicative}
L.~Hodgkinson and M.~Mahoney, ``Multiplicative noise and heavy tails in
  stochastic optimization,'' in \emph{International Conference on Machine
  Learning}.\hskip 1em plus 0.5em minus 0.4em\relax PMLR, 2021, pp. 4262--4274.

\bibitem{wang2021convergence}
H.~Wang, M.~Gurbuzbalaban, L.~Zhu, U.~Simsekli, and M.~A. Erdogdu,
  ``Convergence rates of stochastic gradient descent under infinite noise
  variance,'' \emph{Advances in Neural Information Processing Systems},
  vol.~34, 2021.

\bibitem{yaida2018fluctuation}
S.~Yaida, ``Fluctuation-dissipation relations for stochastic gradient
  descent,'' \emph{arXiv preprint arXiv:1810.00004}, 2018.

\bibitem{hu2019diffusion}
W.~Hu, C.~J. Li, L.~Li, and J.-G. Liu, ``On the diffusion approximation of
  nonconvex stochastic gradient descent,'' \emph{Annals of Mathematical
  Sciences and Applications}, vol.~4, no.~1, 2019.

\bibitem{bottou2018optimization}
L.~Bottou, F.~E. Curtis, and J.~Nocedal, ``Optimization methods for large-scale
  machine learning,'' \emph{Siam Review}, vol.~60, no.~2, pp. 223--311, 2018.

\bibitem{Shor1985MinimizationMF}
N.~Z. Shor, ``Minimization methods for non-differentiable functions,'' in
  \emph{Springer Series in Computational Mathematics}, 1985.

\bibitem{zhang2020improved}
B.~Zhang, J.~Jin, C.~Fang, and L.~Wang, ``Improved analysis of clipping
  algorithms for non-convex optimization,'' \emph{Advances in Neural
  Information Processing Systems}, vol.~33, pp. 15\,511--15\,521, 2020.

\bibitem{chen2020understanding}
X.~Chen, S.~Z. Wu, and M.~Hong, ``Understanding gradient clipping in private
  sgd: A geometric perspective,'' \emph{Advances in Neural Information
  Processing Systems}, vol.~33, pp. 13\,773--13\,782, 2020.

\bibitem{qian2021understanding}
J.~Qian, Y.~Wu, B.~Zhuang, S.~Wang, and J.~Xiao, ``Understanding gradient
  clipping in incremental gradient methods,'' in \emph{International Conference
  on Artificial Intelligence and Statistics}.\hskip 1em plus 0.5em minus
  0.4em\relax PMLR, 2021, pp. 1504--1512.

\bibitem{Zhang2021UnderstandingCF}
X.~Zhang, X.~Chen, M.-F. Hong, Z.~S. Wu, and J.~Yi, ``Understanding clipping
  for federated learning: Convergence and client-level differential privacy,''
  \emph{ArXiv}, vol. abs/2106.13673, 2021.

\bibitem{das2021privacy}
R.~Das, A.~Hashemi, S.~Sanghavi, and I.~S. Dhillon, ``Privacy-preserving
  federated learning via normalized (instead of clipped) updates,'' \emph{arXiv
  preprint arXiv:2106.07094}, 2021.

\bibitem{andrew2021differentially}
G.~Andrew, O.~Thakkar, B.~McMahan, and S.~Ramaswamy, ``Differentially private
  learning with adaptive clipping,'' \emph{Advances in Neural Information
  Processing Systems}, vol.~34, 2021.

\bibitem{meng2020dynamic}
Q.~Meng, S.~Gong, W.~Chen, Z.-M. Ma, and T.-Y. Liu, ``Dynamic of stochastic
  gradient descent with state-dependent noise,'' \emph{arXiv preprint
  arXiv:2006.13719}, 2020.

\bibitem{arjevani2019lower}
Y.~Arjevani, Y.~Carmon, J.~C. Duchi, D.~J. Foster, N.~Srebro, and B.~Woodworth,
  ``Lower bounds for non-convex stochastic optimization,'' \emph{arXiv preprint
  arXiv:1912.02365}, 2019.

\bibitem{krizhevsky2009learning}
A.~Krizhevsky, G.~Hinton \emph{et~al.}, ``Learning multiple layers of features
  from tiny images,'' 2009.

\bibitem{li2020convergence}
X.~Li, K.~Huang, W.~Yang, S.~Wang, and Z.~Zhang, ``On the convergence of fedavg
  on non-iid data,'' in \emph{International Conference on Learning
  Representations}, 2020.

\bibitem{li2020federated}
T.~Li, A.~K. Sahu, M.~Zaheer, M.~Sanjabi, A.~Talwalkar, and V.~Smith,
  ``Federated optimization in heterogeneous networks,'' \emph{Proceedings of
  Machine Learning and Systems}, vol.~2, pp. 429--450, 2020.

\bibitem{lecun1998gradient}
Y.~LeCun, L.~Bottou, Y.~Bengio, and P.~Haffner, ``Gradient-based learning
  applied to document recognition,'' \emph{Proceedings of the IEEE}, vol.~86,
  no.~11, pp. 2278--2324, 1998.

\end{thebibliography}

% !TEX root = main.tex

\section*{Checklist}

\begin{enumerate}

\item For all authors...
\begin{enumerate}
  \item Do the main claims made in the abstract and introduction accurately reflect the paper's contributions and scope?
    \answerYes{}
  \item Did you describe the limitations of your work?
    \answerYes{}
  \item Did you discuss any potential negative societal impacts of your work?
    \answerNo{As a theoretical paper towards further understanding of federated optimization, we do not see a direct path to any negative applications.}
  \item Have you read the ethics review guidelines and ensured that your paper conforms to them?
    \answerYes{}
\end{enumerate}

\item If you are including theoretical results...
\begin{enumerate}
  \item Did you state the full set of assumptions of all theoretical results?
    \answerYes{See Section~\ref{subsec:analysis}.}
        \item Did you include complete proofs of all theoretical results?
    \answerYes{See Appendix.}
\end{enumerate}

\item If you ran experiments...
\begin{enumerate}
  \item Did you include the code, data, and instructions needed to reproduce the main experimental results (either in the supplemental material or as a URL)?
    \answerNo{The dataset used in this paper is widely-used public datasets and we have provided detailed instruction for the partition about datasets.}
  \item Did you specify all the training details (e.g., data splits, hyperparameters, how they were chosen)?
    \answerYes{See Section~\ref{sec:numerical} and Appendix.}
  \item Did you report error bars (e.g., with respect to the random seed after running experiments multiple times)?
    \answerNo{We run the experiments multiple times and report the failure rates instead.}
  \item Did you include the total amount of compute and the type of resources used (e.g., type of GPUs, internal cluster, or cloud provider)?
    \answerNo{}
\end{enumerate}

\item If you are using existing assets (e.g., code, data, models) or curating/releasing new assets...
\begin{enumerate}
  \item If your work uses existing assets, did you cite the creators?
    \answerYes{}
  \item Did you mention the license of the assets?
    \answerNo{}
  \item Did you include any new assets either in the supplemental material or as a URL?
    \answerNo{}
  \item Did you discuss whether and how consent was obtained from people whose data you're using/curating?
    \answerNo{}
  \item Did you discuss whether the data you are using/curating contains personally identifiable information or offensive content?
    \answerNo{}
\end{enumerate}

\item If you used crowdsourcing or conducted research with human subjects...
\begin{enumerate}
  \item Did you include the full text of instructions given to participants and screenshots, if applicable?
    \answerNA{}
  \item Did you describe any potential participant risks, with links to Institutional Review Board (IRB) approvals, if applicable?
    \answerNA{}
  \item Did you include the estimated hourly wage paid to participants and the total amount spent on participant compensation?
    \answerNA{}
\end{enumerate}

\end{enumerate}

\appendix
% !TEX root = ../main.tex
\newpage
\allowdisplaybreaks

\section{Proofs for Fat-Tailed Federated Learning} \label{proof_heavytail}

% -----------------------------
\subsection{Proof of \algpr} \label{proof_cpr}

For notional clarity, we have the following update:
\begin{align*}
	\text{Local update: } \x_{t, i}^{k+1} &= \x_{t, i}^{k} - \eta_L \nabla f_i(\x_{t, i}^k, \xi_{t, i}^k), k \in [K], \\
	\text{Clipping: } \x_{t, i}^{K+1} &= \x_{t, i}^{k} - \eta_L \text{clipping} (\sum_{k \in [K]} \nabla f_i(\x_{t, i}^k, \xi_{t, i}^k)),\\
	\Delta_{t, i} &= \sum_{k \in [K]} \nabla f_i(\x_{t, i}^k, \xi_{t, i}^k), \tilde{\Delta}_{t, i} = \text{clipping} (\sum_{k \in [K]} \nabla f_i(\x_{t, i}^k, \xi_{t, i}^k), \lambda), \\
	\Delta_{t} &= \frac{1}{m} \sum_{i \in [m]} \Delta_{t, i}, \tilde{\Delta}_{t} = \frac{1}{m} \sum_{i \in [m]} \tilde{\Delta}_{t, i}\\
	\x_{t+1} &= \x_t - \eta \eta_L \tilde{\Delta}_{t}.
\end{align*}

\begin{lem} [Bounded Variance of Stochastic Local Updates for \algprns] \label{lem: variance_clipping_updates_cpr}
	Assume $f_i(\x, \xi)$ satisfies the Bounded $\alpha-$Moment assumption~\ref{assum_moment}, then for \algpr we have:
	\begin{align*}
		\mathbb{E}[\| \tilde{\Delta}_{t} \|^2] &\leq K^2 G^\alpha \lambda^{2 - \alpha}, \\
		\mathbb{E} \| \tilde{\Delta}_{t} - \mathbb{E} [\tilde{\Delta}_{t}] \|^2 &\leq \frac{K^2}{m} G^\alpha \lambda^{2-\alpha}, \\
		\| \frac{1}{K} \mathbb{E}[\tilde{\Delta}_{t}] - \nabla f(\x_t) \|^2 &\leq L^2 \eta_L^2 K^2 G^2 + K^2 G^{2\alpha} \lambda^{-2(\alpha -1)} + L \eta_L K^2 G^{1+\alpha} \lambda^{1-\alpha}.
	\end{align*}
	Note here the expectation is on the random samples $\xi_{t, i}^k$.
\end{lem}

\begin{proof}
	\begin{align*}
		\mathbb{E}[\| \tilde{\Delta}_{t} \|^2] &= \max_{i \in [m]} \mathbb{E}[\| \tilde{\Delta}_{t, i} \|^2] \\
		&\leq \mathbb{E}[\| \tilde{\Delta}_{t, j} \|^\alpha] \lambda^{2 - \alpha} \\
		&\leq K \sum_{k \in K} \mathbb{E}[\| \nabla f(\x_{t, j}^k, \xi_{t, j}^k) \|^\alpha] \lambda^{2 - \alpha} \\
		&\leq K^2 G^\alpha \lambda^{2 - \alpha},
	\end{align*}
	where $j = argmax_{i \in [m]} \mathbb{E}[\| \tilde{\Delta}_{t, i} \|^2]$, and the first inequality is due to the clipping, i.e., $\| \tilde{\Delta}_{t, i} \| \leq \lambda$.

	\begin{align*}
		\mathbb{E} \| \tilde{\Delta}_{t} - \mathbb{E} [\tilde{\Delta}_{t}] \|^2 &= \mathbb{E} \left\| \frac{1}{m} \sum_{i \in [m]} \left( \tilde{\Delta}_{t, i} - \mathbb{E} [\tilde{\Delta}_{t, i}] \right) \right\|^2 \\
		&\leq \frac{1}{m^2} \sum_{i \in [m]} \mathbb{E} \| \tilde{\Delta}_{t, i} - \mathbb{E} [\tilde{\Delta}_{t, i}] \|^2 \\
		&\leq \frac{1}{m^2} \sum_{i \in [m]} \mathbb{E} \| \tilde{\Delta}_{t, i} \|^2 \\
		&\leq \frac{K^2}{m} G^\alpha \lambda^{2-\alpha}.
	\end{align*}

	\begin{align*}
		&\| \frac{1}{K} \mathbb{E}[\tilde{\Delta}_{t}] - \nabla f(\x_t) \| \leq \left\| \nabla f(\x_t)- \frac{1}{K} \mathbb{E}[\Delta_t] \right\| + \frac{1}{K} \left\| \mathbb{E}[\Delta_t]- \mathbb{E}[\tilde{\Delta}_t] \right\| \\
		&\leq \frac{1}{m K} \sum_{i \in [m]} \sum_{k \in [K]} \mathbb{E} \left\| \nabla f_i(\x_t)- \nabla f_i (\x_{t, i}^k) \right\| + \frac{1}{m K} \sum_{i \in [m]} \left\| \mathbb{E}[\Delta_{t,i} - \tilde{\Delta}_{t,i}] \right\| \\
		&\leq \frac{ L \eta_L}{m K} \sum_{i \in [m]} \sum_{k \in [K]} \mathbb{E}\left\| \sum_{j \in [k]} \nabla f_i (\x_{t, i}^j, \xi_{t, i}^j) \right\| + \frac{1}{m K} \sum_{i \in [m]} \mathbb{E} [ \left\| \Delta_{t,i} \right\| \mathbf{1}_{ \{ \| \Delta_{t, i} \| \geq \lambda \} }] \\
		&\leq L \eta_L K G + K G^{\alpha} \lambda^{1 - \alpha}
	\end{align*}
	where $\mathbf{1}_{ \{ \cdot \}}$ is the indicator function,
	the last inequality follows from the fact that $\Delta_{t,i} = \tilde{\Delta}_{t,i}$ if $\| \Delta_{t, i} \| \leq \lambda$ and $\mathbb{E} [ \left\| \Delta_{t,i} \right\| \mathbf{1}_{ \{ \| \Delta_{t, i} \| \geq \lambda \} }] \leq \mathbb{E} [ \| \Delta_{t, i} \|^\alpha ] \lambda^{1 - \alpha} \leq K^2 G^{\alpha} \lambda^{1 - \alpha}$;
	the second last inequality is due to L-smoothness, Jenson's inequality (i.e.,$\mathbb{E}[\Delta_{t,i} - \tilde{\Delta}_{t,i}] \| \leq \mathbb{E} \| [\Delta_{t,i} - \tilde{\Delta}_{t,i}] \|$) and the clipping step.
	Then, we have 

	\begin{align*}
		&\| \frac{1}{K} \mathbb{E}[\tilde{\Delta}_{t}] - \nabla f(x) \|^2
		\leq L^2 \eta_L^2 K^2 G^2 + K^2 G^{2\alpha} \lambda^{-2(\alpha -1)} + L \eta_L K^2 G^{1+\alpha} \lambda^{1-\alpha}
	\end{align*}
\end{proof}

% --------------------------------------
	% CPR bound in strongly-convex
% --------------------------------------
\ClippingPerRoundBoundedConvex*

\begin{proof}

\begin{align*}
	&\mathbb{E} [\| \x_{t+1} - x_* \|^2] = \mathbb{E} [\| \x_{t} - \eta \eta_L \tilde{\Delta}_{t}  - x_* \|^2] \nonumber\\
	&= \| \x_{t} - x_* \|^2 + \eta^2 \eta_L^2 \mathbb{E} [\| \tilde{\Delta}_{t} \|^2] - 2 \left< \x_{t} - x_*, \eta \eta_L \left( \mathbb{E}[\tilde{\Delta}_{t}] - K \nabla f(\x_t) + K \nabla f(\x_t) \right)\right> \nonumber\\
	&\leq (1 - \mu \eta \eta_L K )\| \x_{t} - x_* \|^2 + \eta^2 \eta_L^2 \mathbb{E} [\| \tilde{\Delta}_{t} \|^2] - 2 \eta \eta_L K \left< \x_{t} - x_*, \left( \frac{1}{K} \mathbb{E}[ \tilde{\Delta}_{t}] - \nabla f(\x_t) \right)\right> \nonumber \\
	&\quad - 2\eta \eta_L K \left(f(\x_t) - f(\x_*) \right) \nonumber\\
	&\leq (1 - \frac{1}{2} \mu \eta \eta_L K )\| \x_{t} - x_* \|^2 + \eta^2 \eta_L^2 \mathbb{E} [\| \tilde{\Delta}_{t} \|^2] + \frac{8 \eta \eta_L K}{\mu} \| \frac{1}{K} \mathbb{E}[\tilde{\Delta}_{t}] - \nabla f(\x_t) \|^2 \nonumber \\
	&\quad - 2\eta \eta_L K \left(f(\x_t) - f(\x_*) \right).
\end{align*}

The first inequality follows from the strongly-convex property, i.e., $- \left< \x_{t} - x_*, \nabla f(\x_t) \right> \leq - (f(\x_t) - f(\x_*) + \frac{\mu}{2} \| x_t - \x_* \|^2)$, and the last inequality is due to Young's inequality.
Then we have
\begin{align*}
	f(\x_t) - f(\x_*) &\leq \frac{1}{2\eta \eta_L K} \left[- \mathbb{E} [\| \x_{t+1} - x_* \|^2] + (1 - \frac{1}{2} \mu \eta \eta_L K )\| \x_{t} - x_* \|^2 \right] \\
	&\quad + \frac{\eta \eta_L}{2 K} \mathbb{E} [\| \tilde{\Delta}_{t} \|^2] + \frac{4}{\mu} \| \mathbb{E}[\frac{1}{K} \tilde{\Delta}_{t} - \nabla f(\x_t) \|^2 \\
	&\leq \frac{1}{2\eta \eta_L K} \left[- \mathbb{E} [\| \x_{t+1} - x_* \|^2] + (1 - \frac{1}{2} \mu \eta \eta_L K )\| \x_{t} - x_* \|^2 \right] \\
	&\quad + \frac{\eta \eta_L}{2 K} K^2 G^\alpha \lambda^{2-\alpha} + \frac{4}{\mu} [L^2 \eta_L^2 K^2 G^2 + K^2 G^{2\alpha} \lambda^{-2(\alpha -1)} + L \eta_L K^2 G^{1+\alpha} \lambda^{1-\alpha}],
\end{align*}
where the last inequality is due to Lemma~\ref{lem: variance_clipping_updates_cpr}.
% $\mathbb{E} [\| \tilde{\Delta}_{t} \|^2] \leq K^2 G^\alpha \lambda^{2-\alpha}$,
% $\| \mathbb{E}[\frac{1}{K} \tilde{\Delta}_{t} - \nabla f(\x_t) \|^2 \leq L^2 \eta_L^2 K^2 G^2 + K^2 G^{2\alpha} \lambda^{-2(\alpha -1)} + L \eta_L K^2 G^{1+\alpha} \lambda^{1-\alpha}$.

Let $w_t = (1 - \frac{1}{2} \mu \eta \eta_L K)^{1 - t} $, $\bar{\x}_T = \x_t$ with probability $\frac{w_t}{\sum_{j \in [T]} w_j}$.
\begin{align*}
	f(\bar{x}_T) - f(\x_*)
	&\leq \frac{1}{\sum_{j \in [T]} w_j} \sum_{t \in [T]} \left( \frac{w_t}{2\eta \eta_L K} \left[- \| \x_{t+1} - x_* \|^2 + (1 - \frac{1}{2} \mu \eta \eta_L K )\| \x_{t} - x_* \|^2 \right] \right) \\
	&\quad + \frac{\eta \eta_L K}{2} G^\alpha \lambda^{2-\alpha} + \frac{4}{\mu} [L^2 \eta_L^2 K^2 G^2 + K^2 G^{2\alpha} \lambda^{-2(\alpha -1)} + L \eta_L K^2 G^{1+\alpha} \lambda^{1-\alpha}] \\
	&\leq \frac{1}{\sum_{j \in [T]} w_j} \sum_{t \in [T]} \left( \frac{1}{2\eta \eta_L K} \left[- w_t \| \x_{t+1} - x_* \|^2 + w_{t-1} \| \x_{t} - x_* \|^2 \right] \right) \\
	&\quad + \frac{\eta \eta_L K}{2} G^\alpha \lambda^{2-\alpha} + \frac{4}{\mu} [L^2 \eta_L^2 K^2 G^2 + K^2 G^{2\alpha} \lambda^{-2(\alpha -1)} + L \eta_L K^2 G^{1+\alpha} \lambda^{1-\alpha}] \\
	&\leq \frac{1}{\sum_{j \in [T]} w_j} \frac{1}{2\eta \eta_L K} \| \x_{1} - x_* \|^2 \\
	&\quad + \frac{\eta \eta_L K}{2} G^\alpha \lambda^{2-\alpha} + \frac{4}{\mu} [L^2 \eta_L^2 K^2 G^2 + K^2 G^{2\alpha} \lambda^{-2(\alpha -1)} + L \eta_L K^2 G^{1+\alpha} \lambda^{1-\alpha}],
\end{align*}
where the second inequality follows from $w_t \leq w_{t-1}$.

\begin{align*}
	2 \eta \eta_L K \sum_{t \in [T]} w_t &= 2 \eta \eta_L K \left(1 - \frac{1}{2} \mu \eta \eta_L K \right)^{-T} \sum_{t \in [T]} \left(1 - \frac{1}{2} \mu \eta \eta_L K \right)^t \\
	&= \frac{4}{\mu} \left(1 - \frac{1}{2} \mu \eta \eta_L K \right)^{-T} \left[1 - \left(1 - \frac{1}{2} \mu \eta \eta_L K \right)^{T}\right] \\
	&\geq \frac{4}{\mu} \left(1 - \frac{1}{2} \mu \eta \eta_L K \right)^{-T} \left[1 - \exp{\left( - \frac{1}{2} \mu \eta \eta_L K T \right)}\right] \\
	&\geq \frac{2}{\mu} \left(1 - \frac{1}{2} \mu \eta \eta_L K \right)^{-T},
\end{align*}

where the last inequality follows from that $\eta \eta_L K \geq \frac{2}{\mu T}$, the second last inequality is due to $\left(1 - \frac{1}{2} \mu \eta \eta_L K \right)^{T} \leq \exp{\left( - \frac{1}{2} \mu \eta \eta_L K T \right)}$.

\begin{align*}
	f(\bar{x}_T) - f(\x_*)
	&\leq \frac{\mu}{2} \left(1 - \frac{1}{2} \mu \eta \eta_L K\right)^T + \frac{\eta \eta_L K}{2} G^\alpha \lambda^{2-\alpha} \\
	&\quad + \frac{4}{\mu} [L^2 \eta_L^2 K^2 G^2 + K^2 G^{2\alpha} \lambda^{-2(\alpha -1)} + L \eta_L K^2 G^{1+\alpha} \lambda^{1-\alpha}] \\
	&\leq \frac{\mu}{2} \exp{\left( - \frac{1}{2} \mu \eta \eta_L K T \right)} + \frac{\eta \eta_L K}{2} G^\alpha \lambda^{2-\alpha} \\
	&\quad + \frac{4}{\mu} [L^2 \eta_L^2 K^2 G^2 + K^2 G^{2\alpha} \lambda^{-2(\alpha -1)} + L \eta_L K^2 G^{1+\alpha} \lambda^{1-\alpha}].
\end{align*}

Let $\eta \eta_L K = \frac{2c}{\mu} \frac{\ln(T)}{mKT}$ ($c \geq 1$ is a constant and $T^{-c - \frac{2-2\alpha}{\alpha}} \leq m^{\frac{2-2\alpha}{\alpha}} K^{\frac{2}{\alpha}}$), $\lambda = (mKT)^{\frac{1}{\alpha}}$, and $\eta_L \leq (mKT)^{\frac{1-\alpha}{\alpha}}$,
\begin{align*}
	f(\bar{x}_T) - f(\x_*)
	&\leq \frac{1}{T^{c}} + (mKT)^{\frac{2 - 2 \alpha}{\alpha}} \ln(T) + (mT)^{\frac{2-2\alpha}{\alpha}} K^{\frac{2}{\alpha}} = \mathcal{O}((mT)^{\frac{2-2\alpha}{\alpha}} K^{\frac{2}{\alpha}}).
\end{align*}
\end{proof}

% --------------------------------------
		% CPR bound in nonconvex
% --------------------------------------

\ClippingPerRoundBoundedNonconvex*
\begin{proof}
Due to the smoothness in Assumption \ref{assum_smooth}, taking expectation of $f(\x_{t+1})$ over the randomness at communication round $t$, we have:
\begin{align}
	&\mathbb{E} [f(\x_{t+1})] - f(\x_t) \leq \big< \nabla f(\x_t), \mathbb{E} [\x_{t+1} - \x_t] \big> + \frac{L}{2} \mathbb{E} [\| \x_{t+1} - \x_t \|^2] \nonumber \\
	&= - \eta \eta_L \big< \nabla f(\x_t), \mathbb{E} [\tilde{\Delta}_{t}] \big> + \frac{L}{2} \eta^2 \eta_L^2 \mathbb{E} [\| \tilde{\Delta}_{t} \|^2] \nonumber\\
	&= - \frac{\eta \eta_L K}{2} \| \nabla f(\x_t) \|^2 - \frac{\eta \eta_L}{2K} \| \mathbb{E}[\tilde{\Delta}_t] \|^2 + \frac{\eta \eta_L K}{2}  \| \nabla f(\x_t)- \frac{1}{K} \mathbb{E}[\tilde{\Delta}_t] \|^2 + \frac{L \eta^2 \eta_L^2}{2} \mathbb{E} [ \| \tilde{\Delta}_t \|^2 ] \nonumber \\
	&= - \frac{\eta \eta_L K}{2} \| \nabla f(\x_t) \|^2 + \left(- \frac{\eta \eta_L}{2K} + \frac{L \eta^2 \eta_L^2}{2} \right) \| \mathbb{E}[\tilde{\Delta}_t] \|^2 + \frac{\eta \eta_L K}{2}  \| \nabla f(\x_t)- \frac{1}{K} \mathbb{E}[\tilde{\Delta}_t] \|^2 \nonumber \\
	&\quad + \frac{L \eta^2 \eta_L^2}{2} \mathbb{E} [ \| \tilde{\Delta}_t - \mathbb{E} [\tilde{\Delta}_t] \|^2 ] \nonumber \\
	&\leq - \frac{\eta \eta_L K}{2} \| \nabla f(\x_t) \|^2 + \frac{\eta \eta_L K}{2} \underbrace{\| \nabla f(\x_t)- \frac{1}{K} \mathbb{E}[\tilde{\Delta}_t] \|^2}_{A_1} + \frac{L \eta^2 \eta_L^2}{2} \underbrace{\mathbb{E} [ \| \tilde{\Delta}_t - \mathbb{E} [\tilde{\Delta}_t] \|^2 ]}_{A_2}, \label{ineq_smooth_cpr}
\end{align}
where the last inequality follows from $\left(- \frac{\eta \eta_L}{2K} + \frac{L \eta^2 \eta_L^2}{2} \right) \leq 0$ if $\eta \eta_L K L \leq 1$.

From Lemma~\ref{lem: variance_clipping_updates_cpr}, we have the bound of $A_1$ and $A_2$ in (\ref{ineq_smooth_cpr}).
By rearranging and telescoping, we have:
\begin{align*}
	\frac{1}{T} \sum_{t \in [T]} \mathbb{E} \| \nabla f(\x_t) \|^2 &\leq \frac{2 \left( f(\x_1) - f(x_T) \right)}{\eta \eta_L K T} + \left( L^2 \eta_L^2 K^2 G^2 + K^2 G^{2\alpha} \lambda^{-2(\alpha -1)} + L \eta_L K^2 G^{1+\alpha} \lambda^{1-\alpha} \right) \\
	&\quad + \frac{L \eta \eta_L}{m} \left(K G^\alpha \lambda^{2-\alpha} \right).
\end{align*}

Suppose $\eta \eta_L = m^{\frac{2 \alpha - 2}{3 \alpha - 2}} K^{\frac{- \alpha - 2}{3 \alpha - 2}} T^{\frac{-\alpha}{3 \alpha - 2}}, \eta_L \leq (mT)^{\frac{1 - \alpha}{3\alpha - 2}} K^{\frac{4 - 4 \alpha}{3 \alpha - 2}}$, and $\lambda = (mK^4T)^{\frac{1}{3 \alpha - 2}}$,
\begin{align*}
	\min_{t \in [T]} \mathbb{E} \| \nabla f(\x_t) \|^2 \leq \mathcal{O}((mT)^{\frac{2 - 2 \alpha}{3 \alpha - 2}} K^{\frac{4 - 2 \alpha}{3 \alpha - 2}}).
\end{align*}

\end{proof}

%--------------------------------------------
\subsection{Proof of \algpins} \label{proof_cpi}
For \algpins, we have the following notions:
\begin{align*}
	\tilde{\nabla} f_i(\x_{t, i}^k, \xi_{t, i}^k) &= \min \{1, \frac{\lambda_t}{\| \nabla f_i(\x_{t, i}^k, \xi_{t, i}^k) \| } \} \nabla f_i(\x_{t, i}^k, \xi_{t, i}^k), \tilde{\nabla} f(\x_{t,i}^k) = \mathbb{E} [\tilde{\nabla} f_i(\x_{t, i}^k, \xi_{t, i}^k)];\\
	\text{Local steps: } \x_{t, i}^{k+1} &= \x_{t, i}^{k} - \eta_L \tilde{\nabla} f_i(\x_{t, i}^k, \xi_{t, i}^k), k \in [K];\\
	\Delta_{t, i} &= \sum_{k \in [K]} \nabla f_i(\x_{t, i}^k, \xi_{t, i}^k), \tilde{\Delta}_{t, i} = \sum_{k \in [K]} \tilde{\nabla} f_i(\x_{t, i}^k, \xi_{t, i}^k) \\
	\Delta_{t} &= \frac{1}{m} \sum_{i \in [m]} \Delta_{t, i}, \tilde{\Delta}_{t} = \frac{1}{m} \sum_{i \in [m]} \tilde{\Delta}_{t, i}\\
	\x_{t+1} &= \x_t - \eta \eta_L \frac{1}{m} \sum_{i \in [m]} \sum_{k \in [K]} \tilde{\nabla} f_i(\x_{t, i}^k, \xi_{t, i}^k) = \x_t - \eta \eta_L \tilde{\Delta}_{t}.
\end{align*}

\begin{lem} [Bounded Variance of Stochastic Local Updates for \algpins] \label{lem: variance_clipping_updates_cpi}
	Assume $f_i(\x, \xi)$ satisfies the Bounded $\alpha-$Moment assumption~\ref{assum_moment}, then we have:
	\begin{align*}
		\mathbb{E} [\| \tilde{\Delta}_{t} \|^2] &\leq K^2 G^\alpha \lambda^{2-\alpha}, \\
		\mathbb{E} \| \tilde{\Delta}_{t} - \mathbb{E} [\tilde{\Delta}_{t}] \|^2 &\leq \frac{K}{m} G^\alpha \lambda^{2-\alpha}, \\
		\| \frac{1}{K} \mathbb{E}[\tilde{\Delta}_{t}] - \nabla f(x) \|^2 &\leq 2 G^{2 \alpha} \lambda^{-2(\alpha - 1)} + 2 L^2  \eta_L^2 K^2 G^\alpha \lambda^{2 - \alpha}.
	\end{align*}
\end{lem}

\begin{proof}
\begin{align*}
	\mathbb{E} [\| \tilde{\Delta}_{t} \|^2] &= \frac{1}{m} \sum_{i \in [m]} \mathbb{E} [\| \tilde{\Delta}_{t, i} \|^2] \\
	&\leq \frac{1}{m} \sum_{i \in [m]} \mathbb{E} [\| \sum_{j \in [K]} \tilde{\nabla} f(\x_{t, i}^j, \xi_{t, i}^j) \|^2] \\
	&\leq \frac{K}{m} \sum_{i \in [m]} \sum_{j \in [K]} \mathbb{E} [\| \tilde{\nabla} f(\x_{t, i}^j, \xi_{t, i}^j) \|^2] \\
	&\leq K^2 G^\alpha \lambda^{2-\alpha},
\end{align*}
where the last inequality follows from the fact that $\mathbb{E} \| \tilde{\nabla} f_i(\x_{t,i}^k, \xi_{t,i}^k) \|^2 \leq \mathbb{E} \| \tilde{\nabla} f_i(\x_{t,i}^k, \xi_{t,i}^k) \|^\alpha \lambda^{2-\alpha} \leq G^\alpha \lambda^{2-\alpha}$ (see Lemma 9 in~\cite{zhang2020adaptive}).

	\begin{align*}
		&\mathbb{E} \| \tilde{\Delta}_{t} - \mathbb{E} [\tilde{\Delta}_{t}] \|^2 = \mathbb{E} \left\| \frac{1}{m} \sum_{i \in [m]} \sum_{k \in [K]} \tilde{\nabla} f_i(\x_{t,i}^k, \xi_{t,i}^k) - \frac{1}{m} \sum_{i \in [m]} \sum_{k \in [K]} \tilde{\nabla} f_i(\x_{t,i}^k) \right\|^2 \\
		&\leq \frac{1}{m^2} \sum_{i \in [m]} \sum_{k \in [K]} \mathbb{E} \| \tilde{\nabla} f_i(\x_{t,i}^k, \xi_{t,i}^k) - \tilde{\nabla} f_i(\x_{t,i}^k) \|^2 \\
		&\leq \frac{1}{m^2} \sum_{i \in [m]} \sum_{k \in [K]} \mathbb{E} \| \tilde{\nabla} f_i(\x_{t,i}^k, \xi_{t,i}^k) \|^2 \\
		&\leq \frac{K}{m} G^\alpha \lambda^{2-\alpha},
	\end{align*}
	where the first inequality follows from the fact that $\{\tilde{\nabla} f_i(\x_{t,i}^k, \xi_{t,i}^k) - \tilde{\nabla} f_i(\x_{t,i}^k)\}$ form a martingale difference sequence (Lemma 4 in ~\cite{Karimireddy2020SCAFFOLD}), the second inequalities is due to $\mathbb{E} [\| X - \mathbb{E}[X] \|^2] \leq \mathbb{E} [\| X \|^2]$, and the third inequality follows from the fact that $\mathbb{E} \| \tilde{\nabla} f_i(\x_{t,i}^k, \xi_{t,i}^k) \|^2 \leq \mathbb{E} \| \tilde{\nabla} f_i(\x_{t,i}^k, \xi_{t,i}^k) \|^\alpha \lambda^{2-\alpha} \leq G^\alpha \lambda^{2-\alpha}$ (see Lemma 9 in~\cite{zhang2020adaptive}).

	\begin{align*}
		&\| \frac{1}{K} \mathbb{E}[\tilde{\Delta}_{t}] - \nabla f(x) \|^2 \leq \left\| \frac{1}{mK} \sum_{i \in [m]} \sum{k \in [K]} \left( \tilde{\nabla} f_i(\x_{t,i}^k) - f_i(\x_t) \right) \right\|^2 \\
		&\leq \frac{1}{mK} \sum_{i \in [m]} \sum_{k \in [K]} \left\| \tilde{\nabla} f_i(\x_{t,i}^k) - f_i(\x_t) \right\|^2 \\
		&\leq \frac{1}{mK} \sum_{i \in [m]} \sum_{k \in [K]} \left(2 \left\| \tilde{\nabla} f_i(\x_{t,i}^k) - \nabla f_i(\x_{t, i}^k) \right\|^2 + 2 \left\| \nabla f_i(\x_{t, i}^k) - f_i(\x_t) \right\|^2 \right) \\
		&\leq 2 G^{2 \alpha} \lambda^{-2(\alpha - 1)} + 2 L^2 \frac{1}{mK} \sum_{i \in [m]} \sum_{k \in [K]} \| \x_{t, i}^k - \x_t \|^2 \\
		&\leq 2 G^{2 \alpha} \lambda^{-2(\alpha - 1)} + 2 L^2 \eta_L^2 \frac{1}{mK} \sum_{i \in [m]} \sum_{k \in [K]} \left\| \sum_{j \in [K]} \nabla f(x_{t, i}^j, \xi_{t, i}^j) \right\|^2 \\
		&\leq 2 G^{2 \alpha} \lambda^{-2(\alpha - 1)} + 2 L^2  \eta_L^2 K^2 G^\alpha \lambda^{2 - \alpha},
	\end{align*}
	where the forth inequality is due to $\| \tilde{\nabla} f_i(\x) - \nabla f_i(\x) \|^2 \leq G^{2 \alpha} \lambda^{-2(\alpha - 1)}$ (see Lemma 9 in~\cite{zhang2020adaptive}), and
	the last inequality follows from the fact that $\mathbb{E} \| \tilde{\nabla} f_i(\x_{t,i}^k, \xi_{t,i}^k) \|^2 \leq G^\alpha \lambda^{2-\alpha}$.
\end{proof}

% --------------------------------------
	% CPI bound in strongly-convex
% --------------------------------------
\ClippingPerIterationBoundedConvex*

\begin{proof}
% \begin{align}
% 	&\mathbb{E} [\| \x_{t+1} - x_* \|^2] = \mathbb{E} [\| \x_{t} - \eta \eta_L \tilde{\Delta}_{t}  - x_* \|^2] \nonumber\\
% 	&= \| \x_{t} - x_* \|^2 + \eta^2 \eta_L^2 \mathbb{E} [\| \tilde{\Delta}_{t} \|^2] - 2 \left< \x_{t} - x_*, \eta \eta_L \left( \mathbb{E}[\tilde{\Delta}_{t}] - K \nabla f(\x_t) + K \nabla f(\x_t) \right)\right> \nonumber\\
% 	&\leq (1 - \mu \eta \eta_L K )\| \x_{t} - x_* \|^2 + \eta^2 \eta_L^2 \mathbb{E} [\| \tilde{\Delta}_{t} \|^2] - 2 \eta \eta_L K \left< \x_{t} - x_*, \left( \frac{1}{K} \mathbb{E}[ \tilde{\Delta}_{t}] - \nabla f(\x_t) \right)\right> - 2\eta \eta_L K \left(f(\x_t) - f(\x_*) \right) \nonumber\\
% 	&\leq (1 - \frac{1}{2} \mu \eta \eta_L K )\| \x_{t} - x_* \|^2 + \eta^2 \eta_L^2 \mathbb{E} [\| \tilde{\Delta}_{t} \|^2] + \frac{8 \eta \eta_L K}{\mu} \| \frac{1}{K} \mathbb{E}[\tilde{\Delta}_{t}] - \nabla f(\x_t) \|^2 - 2\eta \eta_L K \left(f(\x_t) - f(\x_*) \right). \label{eq: cpi_convex}
% \end{align}

% The first inequality follows from the strongly-convex property, i.e., $- \left< \x_{t} - x_*, \nabla f(\x_t) \right> \leq - (f(\x_t) - f(\x_*) + \frac{\mu}{2} \| x_t - \x_* \|^2)$.

Similarly, we have the following one step iteration:
\begin{align*}
	f(\x_t) - f(\x_*) &\leq \frac{1}{2\eta \eta_L K} \left[- \mathbb{E} [\| \x_{t+1} - x_* \|^2] + (1 - \frac{1}{2} \mu \eta \eta_L K )\| \x_{t} - x_* \|^2 \right] \\
	&\quad + \frac{\eta \eta_L}{2 K} \mathbb{E} [\| \tilde{\Delta}_{t} \|^2] + \frac{4}{\mu} \| \mathbb{E}[\frac{1}{K} \tilde{\Delta}_{t} - \nabla f(\x_t) \|^2 \\
	&\leq \frac{1}{2\eta \eta_L K} \left[- \mathbb{E} [\| \x_{t+1} - x_* \|^2] + (1 - \frac{1}{2} \mu \eta \eta_L K )\| \x_{t} - x_* \|^2 \right] \\
	&\quad + \frac{\eta \eta_L}{2 K} K^2 G^\alpha \lambda^{2-\alpha} + \frac{4}{\mu} [2 G^{2 \alpha} \lambda^{-2(\alpha - 1)} + 2 L^2 \eta_L^2 K^2 G^{\alpha} \lambda^{2 - \alpha}],
\end{align*}
where the last inequality is due to Lemma~\ref{lem: variance_clipping_updates_cpi}.

% $\mathbb{E} [\| \tilde{\Delta}_{t} \|^2] \leq K^2 G^\alpha \lambda^{2-\alpha}$,
% $\| \mathbb{E}[\frac{1}{K} \tilde{\Delta}_{t} - \nabla f(\x_t) \|^2 \leq 2 G^{2 \alpha} \lambda^{-2(\alpha - 1)} + 2 L^2 \eta_L^2 K^2 G^{\alpha} \lambda^{2 - \alpha}$.

Let $w_t = (1 - \frac{1}{2} \mu \eta \eta_L K)^{1 - t} $, $\bar{\x}_T = \x_t$ with probability $\frac{w_t}{\sum_{j \in [T]} w_j}$.
\begin{align*}
	f(\bar{x}_T) - f(\x_*)
	&\leq \frac{1}{\sum_{j \in [T]} w_j} \sum_{t \in [T]} \left( \frac{w_t}{2\eta \eta_L K} \left[- \| \x_{t+1} - x_* \|^2 + (1 - \frac{1}{2} \mu \eta \eta_L K )\| \x_{t} - x_* \|^2 \right] \right) \\
	&\quad + \frac{\eta \eta_L K}{2} G^\alpha \lambda^{2-\alpha} + \frac{4}{\mu} [2 G^{2 \alpha} \lambda^{-2(\alpha - 1)} + 2 L^2 \eta_L^2 K^2 G^{\alpha} \lambda^{2 - \alpha}] \\
	&\leq \frac{1}{\sum_{j \in [T]} w_j} \sum_{t \in [T]} \left( \frac{1}{2\eta \eta_L K} \left[- w_t \| \x_{t+1} - x_* \|^2 + w_{t-1} \| \x_{t} - x_* \|^2 \right] \right) \\
	&\quad + \frac{\eta \eta_L K}{2} G^\alpha \lambda^{2-\alpha} + \frac{4}{\mu} [2 G^{2 \alpha} \lambda^{-2(\alpha - 1)} + 2 L^2 \eta_L^2 K^2 G^{\alpha} \lambda^{2 - \alpha}] \\
	&\leq \frac{1}{\sum_{j \in [T]} w_j} \frac{1}{2\eta \eta_L K} \| \x_{1} - x_* \|^2 \\
	&\quad + \frac{\eta \eta_L K}{2} G^\alpha \lambda^{2-\alpha} + \frac{4}{\mu} [2 G^{2 \alpha} \lambda^{-2(\alpha - 1)} + 2 L^2 \eta_L^2 K^2 G^{\alpha} \lambda^{2 - \alpha}].
\end{align*}

\begin{align*}
	2 \eta \eta_L K \sum_{t \in [T]} w_t &= 2 \eta \eta_L K \left(1 - \frac{1}{2} \mu \eta \eta_L K \right)^{-T} \sum_{t \in [T]} \left(1 - \frac{1}{2} \mu \eta \eta_L K \right)^t \\
	&= \frac{4}{\mu} \left(1 - \frac{1}{2} \mu \eta \eta_L K \right)^{-T} \left[1 - \left(1 - \frac{1}{2} \mu \eta \eta_L K \right)^{T}\right] \\
	&\geq \frac{4}{\mu} \left(1 - \frac{1}{2} \mu \eta \eta_L K \right)^{-T} \left[1 - \exp{\left( - \frac{1}{2} \mu \eta \eta_L K T \right)}\right] \\
	&\geq \frac{2}{\mu} \left(1 - \frac{1}{2} \mu \eta \eta_L K \right)^{-T},
\end{align*}

where the last inequality follows from that $\eta \eta_L K \geq \frac{2}{\mu T}$, teh second last inequality is due to $\left(1 - \frac{1}{2} \mu \eta \eta_L K \right)^{T} \leq \exp{\left( - \frac{1}{2} \mu \eta \eta_L K T \right)}$.

\begin{align*}
	f(\bar{x}_T) - f(\x_*)
	&\leq \frac{\mu}{2} \left(1 - \frac{1}{2} \mu \eta \eta_L K\right)^T + \frac{\eta \eta_L K}{2} G^\alpha \lambda^{2-\alpha} + \frac{4}{\mu} [2 G^{2 \alpha} \lambda^{-2(\alpha - 1)} + 2 L^2 \eta_L^2 K^2 G^{\alpha} \lambda^{2 - \alpha}] \\
	&\leq \frac{\mu}{2} \exp{\left( - \frac{1}{2} \mu \eta \eta_L K T \right)} + \frac{\eta \eta_L K}{2} G^\alpha \lambda^{2-\alpha} + \frac{4}{\mu} [2 G^{2 \alpha} \lambda^{-2(\alpha - 1)} + 2 L^2 \eta_L^2 K^2 G^{\alpha} \lambda^{2 - \alpha}].
\end{align*}

Let $\eta \eta_L K = \frac{2c}{\mu} \frac{\ln(T)}{mKT}$ ($c \geq 1$ is a constant and $T^{-c - \frac{2-2\alpha}{\alpha}} \leq (mK)^{\frac{2-2\alpha}{\alpha}}$), $\lambda = (mKT)^{\frac{1}{\alpha}}$, and $\eta_L \leq (mT)^{-\frac{1}{2}} K^{- \frac{3}{2}}$,
\begin{align*}
	f(\bar{x}_T) - f(\x_*)
	&\leq \frac{1}{T^{c}} + (mKT)^{\frac{2 - 2 \alpha}{\alpha}} \ln(T) = \tilde{\mathcal{O}}((mKT)^{\frac{2 - 2 \alpha}{\alpha}}).
\end{align*}
\end{proof}

% --------------------------------------
			% CPI bound in nonconvex
% --------------------------------------
\ClippingPerIterationBoundedNonconvex*

\begin{proof}
Due to the smoothness in Assumption \ref{assum_smooth}, taking expectation of $f(\x_{t+1})$ over the randomness at communication round $t$, we have:
\begin{align}
	&\mathbb{E} [f(\x_{t+1})] - f(\x_t) \leq \big< \nabla f(\x_t), \mathbb{E} [\x_{t+1} - \x_t] \big> + \frac{L}{2} \mathbb{E} [\| \x_{t+1} - \x_t \|^2] \nonumber\\
	&= - \eta \eta_L \big< \nabla f(\x_t), \mathbb{E} [\tilde{\Delta}_{t}] \big> + \frac{L}{2} \eta^2 \eta_L^2 \mathbb{E} [\| \tilde{\Delta}_{t} \|^2] \nonumber\\
	&= - \frac{\eta \eta_L K}{2} \| \nabla f(\x_t) \|^2 - \frac{\eta \eta_L}{2K} \| \mathbb{E}[\tilde{\Delta}_t] \|^2 + \frac{\eta \eta_L K}{2}  \| \nabla f(\x_t)- \frac{1}{K} \mathbb{E}[\tilde{\Delta}_t] \|^2 + \frac{L \eta^2 \eta_L^2}{2} \mathbb{E} [ \| \tilde{\Delta}_t \|^2 ] \nonumber\\
	&= - \frac{\eta \eta_L K}{2} \| \nabla f(\x_t) \|^2 + \left(- \frac{\eta \eta_L}{2K} + \frac{L \eta^2 \eta_L^2}{2} \right) \| \mathbb{E}[\tilde{\Delta}_t] \|^2 + \frac{\eta \eta_L K}{2}  \| \nabla f(\x_t)- \frac{1}{K} \mathbb{E}[\tilde{\Delta}_t] \|^2 + \frac{L \eta^2 \eta_L^2}{2} \mathbb{E} [ \| \tilde{\Delta}_t - \mathbb{E} [\tilde{\Delta}_t] \|^2 ] \nonumber\\
	&\leq - \frac{\eta \eta_L K}{2} \| \nabla f(\x_t) \|^2 + \frac{\eta \eta_L K}{2} \underbrace{\| \nabla f(\x_t)- \frac{1}{K} \mathbb{E}[\tilde{\Delta}_t] \|^2}_{A_1} + \frac{L \eta^2 \eta_L^2}{2} \underbrace{\mathbb{E} [ \| \tilde{\Delta}_t - \mathbb{E} [\tilde{\Delta}_t] \|^2 ]}_{A_2}, \label{ineq_smooth_cpi}
\end{align}
where the last inequality follows from $\left(- \frac{\eta \eta_L}{2K} + \frac{L \eta^2 \eta_L^2}{2} \right) \leq 0$ if $\eta \eta_L K L \leq 1$.

From Lemma~\ref{lem: variance_clipping_updates_cpi}, we have the bound of $A_1$ and $A_2$ in (\ref{ineq_smooth_cpi}).
By rearranging and telescoping, we have:
\begin{align*}
	\frac{1}{T} \sum_{t \in [T]} \mathbb{E} \| \nabla f(\x_t) \|^2 \leq \frac{2 \left( f(\x_1) - f(x_T) \right)}{\eta \eta_L K T} + \left( 2 G^{2 \alpha} \lambda^{-2(\alpha - 1)} + 2 L^2  \eta_L^2 K^2 G^\alpha \lambda^{2 - \alpha} \right) + \frac{L \eta \eta_L}{m} \left(G^\alpha \lambda^{2-\alpha} \right).
\end{align*}

Suppose $\eta \eta_L = m^{\frac{2 \alpha - 2}{3 \alpha - 2}} (KT)^{\frac{-\alpha}{3 \alpha - 2}}, \eta_L \leq (mKT)^{\frac{- \alpha}{6 \alpha - 4}}$, and $\lambda = (mKT)^{\frac{1}{3 \alpha - 2}}$,
\begin{align*}
	\min_{t \in [T]} \mathbb{E} \| \nabla f(\x_t) \|^2 \leq \mathcal{O}((mKT)^{\frac{2 - 2 \alpha}{3 \alpha - 2}}).
\end{align*}
\end{proof}

% !TEX root = ../main.tex
% \newpage
\allowdisplaybreaks

\subsection{Proof of \algpr in Gaussian Noise} \label{subsec:GaussianNoise}
In this subsection, we utilize the classic bounded variance and bounded gradient assumption.
\begin{assum}(Bounded Stochastic Gradient Variance) \label{assum_variance}
	There exists a constant $\sigma > 0$, such that  the variance of each local gradient estimator is bounded by
	$\mathbb{E} [\| \nabla f_i(\x, \xi) -  \nabla f_i(\x) ||^2] \leq \sigma^2$, $\forall i \in [m]$.
\end{assum}
\begin{assum}(Bounded Gradient ) \label{assum_bounded_gradient}
	There exists a constant $G \geq 0$, such that gradient is bounded by
	$\| \nabla f_i(\x) ||^2 \leq G^2$, $\forall i \in [m]$.
\end{assum}

\begin{lem} [Lemma F.5~\cite{gorbunov2020stochastic}] \label{lem: aux_variance}
	Suppose there exists a constant $\sigma$ such that the variance of the stochastic gradient of $F$ has bounded variance, i.e., $\mathbb{E} [\| \nabla F(\x, \xi) -  \nabla F(\x) \|^2] \leq \sigma^2$, and $\| \nabla F(\x) \|^2 \leq \frac{\lambda}{2}$, then we have the following inequalities for the clipping $\tilde{\nabla} F(\x_t) = \mathbb{E} [\tilde{\nabla} F(\x, \xi)] = \mathbb{E}[\min \{1, \frac{\lambda}{\| \nabla F(\x, \xi) \| } \} \nabla F(\x, \xi)]$:
	\begin{align*}
		&\| \mathbb{E} [\tilde{\nabla} F(\x, \xi)] - \nabla F(\x) \|^2 \leq  \frac{16 \sigma^4}{\lambda^2}, \\
		&\mathbb{E} \| \tilde{\nabla} F(\x, \xi) - \nabla F(\x) \|^2 \leq 18 \sigma^2, \\
		&\mathbb{E} \| \tilde{\nabla} F(\x, \xi) - \mathbb{E} [\tilde{\nabla} F(\x, \xi)] \|^2 \leq 18 \sigma^2.
	\end{align*}
\end{lem}
We remark that for any stochastic estimator satisfies the above conditions, the above inequalities hold. The proof is the exactly same as that in original proof~\cite{gorbunov2020stochastic}.

\begin{lem} [Bounded Variance of Clipping Stochastic Local Updates in \algprns] \label{lem: variance_clipping_updates_gaussian}
	Assume $f_i$ satisfies the bounded variance assumption, then we have:
	\begin{align*}
		\mathbb{E} [\| \Delta_{t, i} - \mathbb{E} [\Delta_{t, i}] \|^2] &\leq K \sigma^2.
	\end{align*}

	In addition, assume there exists a constant $G$ such that gradient is bounded $\| \nabla f_i(\x) \|^2 \leq G^2 $, if we set clipping parameter as $\lambda^2 \geq 2 K^2 G^2$, i.e., $\| \nabla f_i(\x) \| \leq \frac{\lambda}{2}$, then we have:
	\begin{align*}
		\left\| \mathbb{E}[\Delta_{t,i}]- \mathbb{E}[\tilde{\Delta}_{t,i}] \right\|^2 &\leq \frac{16 K \sigma^4}{\lambda^2}, \\
		\mathbb{E} \| \tilde{\Delta}_{t} - \mathbb{E} [\tilde{\Delta}_{t}] \|^2
		&\leq \frac{18K}{m} \sigma^4.
	\end{align*}.
\end{lem}

\begin{proof}
	\begin{align*}
		\mathbb{E} [\| \Delta_{t, i} - \mathbb{E} [\Delta_{t, i}] \|^2] &= \mathbb{E} [\| \nabla f(\x_{t, i}^j, \xi_{t, i}^j) - \mathbb{E} [\nabla f(\x_{t, i}^j)] \|^2] \\
		&\leq K \sigma^2,
	\end{align*}
	where $\{\nabla f(\x_{t, i}^j, \xi_{t, i}^j) - \mathbb{E} [\nabla f(\x_{t, i}^j)\}$ forms martingale difference sequence (Lemma 4 in ~\cite{Karimireddy2020SCAFFOLD}).

	Then by applying Lemma~\ref{lem: aux_variance}, we have the bound of $\left\| \mathbb{E}[\Delta_{t,i}]- \mathbb{E}[\tilde{\Delta}_{t,i}] \right\|^2$.

	\begin{align*}
		\mathbb{E} \| \tilde{\Delta}_{t} - \mathbb{E} [\tilde{\Delta}_{t}] \|^2 &= \mathbb{E} \left\| \frac{1}{m} \sum_{i \in [m]} \tilde{\Delta}_{t, i} - \frac{1}{m} \sum_{i \in [m]} \mathbb{E} [\tilde{\Delta}_{t, i}] \right\|^2 \\
		&\leq \frac{18K}{m} \sigma^4.
	\end{align*}
	where the last inequality follows from the fact that $\mathbb{E} [\| \Delta_{t, i} - \mathbb{E} [\Delta_{t, i}] \|^2] \leq K \sigma^2$, $\{\Delta_{t, i} - \mathbb{E} [\Delta_{t, i}]\}$ forms martingale difference sequence and Lemma~\ref{lem: aux_variance}.
\end{proof}

\begin{restatable} {theorem}{ClippingPerRoundBoundedGaussian}
	\label{thm:ClippingPerRoundBoundedGaussian}
	Suppose $f$ is non-convex function, under Assumptions~\ref{assum_smooth}, ~\ref{assum_unbias}, ~\ref{assum_variance}, and ~\ref{assum_bounded_gradient}, if $\eta \eta_L K L \leq 1$,
	then the sequence of outputs $\{ \x_k \}$ generated by Algorithm \algpr satisfies:
	\begin{align*}
		\frac{1}{T} \sum_{t \in [T]} \mathbb{E} \| \nabla f(\x_t) \|^2 &\leq \frac{2 \left( f(\x_0) - f(\x_T) \right)}{\eta \eta_L K T} + \frac{1}{T} \sum_{t \in [T]} \left(2 L^2 \eta_L^2 K^2 (\sigma^2 + G^2) + \frac{32 \sigma^4}{K^2 \lambda_t^2} \right) + \left( \frac{18 L \eta \eta_L}{m} \sigma^2 \right).
	\end{align*}
	Choosing learning rates and clipping parameter as $\eta \eta_L = \frac{m^{1/2}}{(KT)^{1/2}}, \eta_L \leq \frac{1}{(mT)^{1/2} K^{5/2}}$, and $\lambda_t \geq (mT)^{1/4} K^{-3/4}$,
	\begin{align*}
		\min_{t \in [T]} \mathbb{E} \| \nabla f(\x_t) \|^2 \leq \mathcal{O}((mKT)^{-\frac{1}{2}}).
	\end{align*}
	\end{restatable}

	% \ClippingPerRoundBoundedGaussian*
	\begin{proof}
		
		% \begin{align}
		% 	\text{Local update: } \x_{t, i}^{k+1} &= \x_{t, i}^{k} - \eta_L \nabla f_i(\x_{t, i}^k, \xi_{t, i}^k), k \in [K], \\
		% 	\text{Clipping: } \x_{t, i}^{K+1} &= \x_{t, i}^{k} - \eta_L \text{clipping} (\sum_{k \in [K]} \nabla f_i(\x_{t, i}^k, \xi_{t, i}^k)),\\
		% 	\Delta_{t, i} &= \sum_{k \in [K]} \nabla f_i(\x_{t, i}^k, \xi_{t, i}^k), \tilde{\Delta}_{t, i} = \text{clipping} (\sum_{k \in [K]} \nabla f_i(\x_{t, i}^k, \xi_{t, i}^k)), \\
		% 	\Delta_{t} &= \frac{1}{m} \sum_{i \in [m]} \Delta_{t, i}, \tilde{\Delta}_{t} = \frac{1}{m} \sum_{i \in [m]} \tilde{\Delta}_{t, i}\\
		% 	\x_{t+1} &= \x_t - \eta \eta_L \tilde{\Delta}_{t}.
		% \end{align}
	
	Due to the smoothness in Assumption \ref{assum_smooth}, taking expectation of $f(\x_{t+1})$ over the randomness at communication round $t$, we have the same inequality:

	\begin{align}
		&\mathbb{E} [f(\x_{t+1})] - f(\x_t)
		\leq - \frac{\eta \eta_L K}{2} \| \nabla f(\x_t) \|^2 + \frac{\eta \eta_L K}{2} \underbrace{\| \nabla f(\x_t)- \frac{1}{K} \mathbb{E}[\tilde{\Delta}_t] \|^2}_{A_1} + \frac{L \eta^2 \eta_L^2}{2} \underbrace{\mathbb{E} [ \| \tilde{\Delta}_t - \mathbb{E} [\tilde{\Delta}_t] \|^2 ]}_{A_2}, \label{ineq_smooth_cprb}
	\end{align}
	where it requires $\eta \eta_L K L \leq 1$.
	
	Note that the term $A_1$ in (\ref{ineq_smooth_cprb}) can be bounded as follows:
	\begin{align*}
		A_1 &= \| \nabla f(\x_t)- \frac{1}{K} \mathbb{E}[\tilde{\Delta}_t] \|^2 \\
		&= 2 \left\| \nabla f(\x_t)- \frac{1}{K} \mathbb{E}[\Delta_t] \right\|^2 + \frac{2}{K^2} \left\| \mathbb{E}[\Delta_t]- \mathbb{E}[\tilde{\Delta}_t] \right\|^2 \\
		&\leq \frac{2}{m K} \sum_{i \in [m]} \sum_{k \in [K]} \left\| \nabla f_i(\x_t)- \nabla f_i (\x_{t, i}^k) \right\|^2 + \frac{2}{m K^2} \sum_{i \in [m]} \left\| \mathbb{E}[\Delta_{t,i}]- \mathbb{E}[\tilde{\Delta}_{t,i}] \right\|^2 \\
		&\leq \frac{2 L^2 \eta_L^2}{m K} \sum_{i \in [m]} \sum_{k \in [K]} \mathbb{E}\left\| \sum_{j \in [k]} \nabla f_i (\x_{t, i}^j, \xi_{t, i}^j) \right\|^2 + \frac{2}{m K^2} \sum_{i \in [m]} \left\| \mathbb{E}[\Delta_{t,i}]- \mathbb{E}[\tilde{\Delta}_{t,i}] \right\|^2 \\
		&\leq 2 L^2 \eta_L^2 K^2 \left( \mathbb{E}\left\| \nabla f_i (\x_{t, i}^k, \xi_{t, i}^k) - \nabla f_i (\x_{t, i}^k) \right\|^2 + \left\| \nabla f_i (\x_{t, i}^k) \right\|^2\right) + \frac{32 \sigma^4}{K^2 \lambda_t^2}  \\
		&\leq 2 L^2 \eta_L^2 K^2 (\sigma^2 + G^2) + \frac{32 \sigma^4}{K^2 \lambda_t^2}
	\end{align*}
	where the second inequality is due to smoothness assumption~\ref{assum_smooth},
	the third inequality is due to Lemma~\ref{lem: variance_clipping_updates_gaussian}, and the last inequality follows from bounded variance assumption~\ref{assum_variance} and bounded gradient assumption~\ref{assum_bounded_gradient}.
	
	From Lemma~\ref{lem: variance_clipping_updates_gaussian}, the term $A_2$ in (\ref{ineq_smooth_cprb}) can be bounded as follows:
	\begin{align*}
		A_2 &\leq \frac{18 K \sigma^2}{m}.
	\end{align*}
	
	Putting pieces together, we can have the one communication round descent in expectation:
	\begin{align*}
		&\mathbb{E} [f(\x_{t+1})] - f(\x_t) \leq  - \frac{\eta \eta_L K}{2} \| \nabla f(\x_t) \|^2 + \frac{\eta \eta_L K}{2} \underbrace{\| \nabla f(\x_t)- \frac{1}{K} \mathbb{E}[\tilde{\Delta}_t] \|^2}_{A_1} + \frac{L \eta^2 \eta_L^2}{2} \underbrace{\mathbb{E} [ \| \tilde{\Delta}_t - \mathbb{E} [\tilde{\Delta}_t] \|^2 ]}_{A_2}\\
		&\leq - \frac{\eta \eta_L K}{2} \| \nabla f(\x_t) \|^2 + \frac{\eta \eta_L K}{2} \left(2 L^2 \eta_L^2 K^2 (\sigma^2 + G^2) + \frac{32 \sigma^4}{K^2 \lambda_t^2} \right) + \frac{18 L K \eta^2 \eta_L^2}{2m} \sigma^2.
	\end{align*}
	
	Rearranging and telescoping, we have the final convergence result:
	\begin{align*}
		\frac{1}{T} \sum_{t \in [T]} \mathbb{E} \| \nabla f(\x_t) \|^2 &\leq \frac{2 \left( f(\x_0) - f(\x_T) \right)}{\eta \eta_L K T} + \frac{1}{T} \sum_{t \in [T]} \left(2 L^2 \eta_L^2 K^2 (\sigma^2 + G^2) + \frac{32 \sigma^2}{K^2 \lambda_t^2} \right) + \left( \frac{18 L \eta \eta_L}{m} \sigma^2 \right).
	\end{align*}
	
	Suppose $\eta \eta_L = \frac{m^{1/2}}{(KT)^{1/2}}, \eta_L \leq \frac{1}{(mT)^{1/2} K^{5/2}}$, and $\lambda_t \geq (mT)^{1/4} K^{-3/4}$,
	\begin{align*}
		\min_{t \in [T]} \mathbb{E} \| \nabla f(\x_t) \|^2 \leq \mathcal{O}((mKT)^{-\frac{1}{2}}).
	\end{align*}
	
	\end{proof}

	\begin{restatable} {theorem}{ClippingPerRoundBoundedConvexGaussian}
		\label{thm:ClippingPerRoundBoundedConvexGaussian}
		Suppose f is $\mu$-strongly convex function, under Assumptions~\ref{assum_smooth}--\ref{assum_moment}, if $\eta \eta_L K \geq \frac{2}{\mu T}$,
		then the outputs $\bar{\x}_T $ in Algorithm 2 (\algprns) by $\bar{\x}_T = \x_t$ with probability $\frac{w_t}{\sum_{j \in [T]} w_j}$ where $w_t = (1 - \frac{1}{2} \mu \eta \eta_L K)^{1 - t} $ satisfies:
		\begin{align*}
			f(\bar{x}_T) - f(\x_*)
			&\leq \frac{\mu}{2} \exp{\left( - \frac{1}{2} \mu \eta \eta_L K T \right)} + \frac{\eta \eta_L K}{2} G^2 + \frac{4}{\mu} [2 L^2 \eta_L^2 K^2 (G^2 + \sigma^2) + \frac{32 \sigma^4}{\lambda^2}].
		\end{align*}
		Suppose $\eta \eta_L K = \frac{2c}{\mu} \frac{\ln(T)}{mKT}$ ($c > 0$ is a constant and $T^{-c + 1} \leq (mK)^{-1} $), $\lambda \geq (mKT)^{\frac{1}{2}}$, and $\eta_L \leq (mT)^{-\frac{1}{2}} K^{-\frac{3}{2}}$,
		\begin{align*}
			f(\bar{x}_T) - f(\x_*)
			&= \tilde{\mathcal{O}}((mKT)^{-1}).
		\end{align*}
	\end{restatable}
	
	\begin{proof}

		\begin{align*}
			&\| \frac{1}{K} \mathbb{E}[\tilde{\Delta}_{t}] - \nabla f(x) \|^2 \leq 2 \left\| \nabla f(\x_t)- \frac{1}{K} \mathbb{E}[\Delta_t] \right\|^2 + \frac{2}{K} \left\| \mathbb{E}[\Delta_t]- \mathbb{E}[\tilde{\Delta}_t] \right\|^2 \\
			&\leq \frac{2}{m K} \sum_{i \in [m]} \sum_{k \in [K]} \mathbb{E} \left\| \nabla f_i(\x_t)- \nabla f_i (\x_{t, i}^k) \right\|^2 + \frac{2}{m K} \sum_{i \in [m]} \left\| \mathbb{E}[\Delta_{t,i}] - \mathbb{E}[\tilde{\Delta}_{t,i}] \right\|^2 \\
			&\leq \frac{ 2 L^2 \eta_L^2}{m K} \sum_{i \in [m]} \sum_{k \in [K]} \mathbb{E}\left\| \sum_{j \in [k]} \nabla f_i (\x_{t, i}^j, \xi_{t, i}^j) \right\|^2 + \frac{32 \sigma^4}{\lambda^2} \\
			&\leq 2 L^2 \eta_L^2 K^2 (G^2 + \sigma^2) + \frac{32 \sigma^4}{\lambda^2}.
		\end{align*}

		Similarly, we have 
		\begin{align*}
			f(\x_t) - f(\x_*) &\leq \frac{1}{2\eta \eta_L K} \left[- \mathbb{E} [\| \x_{t+1} - x_* \|^2] + (1 - \frac{1}{2} \mu \eta \eta_L K )\| \x_{t} - x_* \|^2 \right] \\
			&\quad + \frac{\eta \eta_L}{2 K} \mathbb{E} [\| \tilde{\Delta}_{t} \|^2] + \frac{4}{\mu} \| \mathbb{E}[\frac{1}{K} \tilde{\Delta}_{t} - \nabla f(\x_t) \|^2 \\
			&\leq \frac{1}{2\eta \eta_L K} \left[- \mathbb{E} [\| \x_{t+1} - x_* \|^2] + (1 - \frac{1}{2} \mu \eta \eta_L K )\| \x_{t} - x_* \|^2 \right] \\
			&\quad + \frac{\eta \eta_L K}{2} G^2 + \frac{4}{\mu} [2 L^2 \eta_L^2 K^2 (G^2 + \sigma^2) + \frac{32 \sigma^4}{\lambda^2}],
		\end{align*}

		Let $w_t = (1 - \frac{1}{2} \mu \eta \eta_L K)^{1 - t} $, $\bar{\x}_T = \x_t$ with probability $\frac{w_t}{\sum_{j \in [T]} w_j}$.
		\begin{align*}
			f(\bar{x}_T) - f(\x_*)
			&\leq \frac{1}{\sum_{j \in [T]} w_j} \sum_{t \in [T]} \left( \frac{w_t}{2\eta \eta_L K} \left[- \| \x_{t+1} - x_* \|^2 + (1 - \frac{1}{2} \mu \eta \eta_L K )\| \x_{t} - x_* \|^2 \right] \right) \\
			&\quad + \frac{\eta \eta_L K}{2} G^2 + \frac{4}{\mu} [2 L^2 \eta_L^2 K^2 (G^2 + \sigma^2) + \frac{32 \sigma^4}{\lambda^2}] \\
			&\leq \frac{1}{\sum_{j \in [T]} w_j} \frac{1}{2\eta \eta_L K} \| \x_{1} - x_* \|^2 + \frac{\eta \eta_L K}{2} G^2 + \frac{4}{\mu} [2 L^2 \eta_L^2 K^2 (G^2 + \sigma^2) + \frac{32 \sigma^4}{\lambda^2}].
		\end{align*}

		Same as that in heavy-tailed noise case, we have the same bound for $2 \eta \eta_L K \sum_{t \in [T]} w_t$:
		\begin{align*}
			2 \eta \eta_L K \sum_{t \in [T]} w_t 
			&\geq \frac{2}{\mu} \left(1 - \frac{1}{2} \mu \eta \eta_L K \right)^{-T},
		\end{align*}
		where it requires $\eta \eta_L K \geq \frac{2}{\mu T}$.

		\begin{align*}
			f(\bar{x}_T) - f(\x_*)
			&\leq \frac{\mu}{2} \left(1 - \frac{1}{2} \mu \eta \eta_L K\right)^T + \frac{\eta \eta_L K}{2} G^2 + \frac{4}{\mu} [2 L^2 \eta_L^2 K^2 (G^2 + \sigma^2) + \frac{32 \sigma^4}{\lambda^2}] \\
			&\leq \frac{\mu}{2} \exp{\left( - \frac{1}{2} \mu \eta \eta_L K T \right)} + \frac{\eta \eta_L K}{2} G^2 + \frac{4}{\mu} [2 L^2 \eta_L^2 K^2 (G^2 + \sigma^2) + \frac{32 \sigma^4}{\lambda^2}].
		\end{align*}

		Let $\eta \eta_L K = \frac{2c}{\mu} \frac{\ln(T)}{mKT}$ ($c > 0$ is a constant and $T^{-c + 1} \leq (mK)^{-1} $), $\lambda \geq (mKT)^{\frac{1}{2}}$, and $\eta_L \leq (mT)^{-\frac{1}{2}} K^{-\frac{3}{2}}$,
		\begin{align*}
			f(\bar{x}_T) - f(\x_*)
			&= \tilde{\mathcal{O}}((mKT)^{-1}).
		\end{align*}
	\end{proof}

% \begin{lem} \label{lem: aux_n_x_i}
% $\mathbb{E}[\|x_1 + \cdots + x_n \|^2] \leq n \mathbb{E}[\|x_1\|^2 + \cdots + \| x_n \|^2]$ 
% \end{lem}

% \begin{lem} \label{lem: aux_x_i}
% $\mathbb{E} [\|x_1 + \cdots + x_n \|^2] = \mathbb{E}[\|x_1\|^2 + \cdots + \| x_n \|^2]$ if $x_i^{'}$s are independent with zero mean.
% \end{lem}

% \begin{lem} \label{lem: aux_exp}
% $\mathbb{E}[\| \x \|^2] = \mathbb{E}[\| \x - \mathbb{E}[\x] \|^2] + \| \mathbb{E}[\x] \|^2]$ 
% \end{lem}

% \begin{lem} \label{lem: aux_time}
% $\big<\x, \y \big> = \frac{1}{2} [ \| \x \|^2 + \| \y \|^2 - \| \x - \y \|^2] \leq \frac{1}{2} [ \| \x \|^2 + \| \y \|^2 ]$ 
% \end{lem}

% !TEX root = main.tex

\section{Experiments in Section~\ref{sec:mot}}
In this section, we provide experimental details to demonstrate the fat-tailed noise phenomenon in federated learning. 
We conduct experiments with CNN on CIFAR-10 dataset as shown in Section~\ref{sec:mot}, and provide additional results of RNN model on Shakespeare dataset. 
Furthermore, we verify the accuracy of $\alpha$ estimation with logistic regression on MNIST dataset.

\subsection{CNN on CIFAR-10 Dataset}

\subsubsection{Experiment details}\label{CNN_CIFAR10}
% To find the evidences of fat-tailed noise phenomenon in FedAvg, 
We run a convolutional neural network (CNN) model on CIFAR-10 dataset using FedAvg. The CNN architecture is shown in Table~\ref{table:CNN_arch}.
To simulate data heterogeneity across clients, we manually distribute the the data to each client in a label-based partition. 
Specifically, we split the data according to the classes ($p$) of images that each client has. 
Then, we randomly distribute these partitioned data to $m=100$ clients such that each client has only $p$ classes of images in both training and test data, which causes the heterogeneity of data among different clients. 
For example, for $p=10$, each client contains training/test data samples with ten classes.
Since CIFAR-10 has $10$ classes of images, $p=10$ is the nearly i.i.d case. 
For the remaining $p$, each client contains data samples with class $p$. Therefore, the classes ($p$) of images in each client's local dataset can be used to represent the non-i.i.d. degree. 
The smaller the $p$-value, the more heterogeneous the data between clients.

In this experimental setting, we use the global learning rate $\frac{\eta\eta_{L}}{m}=1.0$ and the local learning rate $\eta_{L}=0.1$. 
The batch size is set to 500, and the communication round is $T=4000$.
We run this experiment in different cases, including singleSGD and different local epochs $ \left\{ 1,2,5 \right\}$ and non-iid index $p \in \left\{ 1,2,5,10 \right\}$. Single SGD means one local update step, which is equivalent to mini-batch SGD.

\begin{table}[t!]
\centering
\caption{CNN architecture for CIFAR-10.}
\label{table:CNN_arch}
\begin{tabular}{ll}
\hline
LAYER TYPE             & SIZE     \\ \hline
Convolution + ReLu     & $3\times 32\times 5$   \\
Max Pooling            & $2\times 2$      \\
Convolution + ReLu     & $32\times 64\times 5$  \\
Max Pooling            & $2\times 2$      \\
Fully Connected + ReLU & $1600\times 512$ \\
Fully Connected + ReLU & $512\times 128$  \\
Fully Connected        & $128\times 10$   \\ \hline
\end{tabular}
\end{table}
\FloatBarrier

\subsubsection{Additional experimental results}
We provide additional distributions of the norms of the pseudo-gradient noises in different cases as follows.
From Fig.~\ref{noise_cifar_singleSGD}-~\ref{noise_cifar_localepoch5}, the observation is that the gradient norm statistics are contracted together for more iid cases while dispersed uniformly for more non-iid cases.
This is 

% we can see that the fat-tailed noise exists in both iid and non-iid cases, and when the data is more non-iid, the noise tail is heavier.

\begin{figure*}[htbp]
\centering
\begin{subfigure}[t]{0.24\textwidth}
    \centering
    \includegraphics[width=1\linewidth]{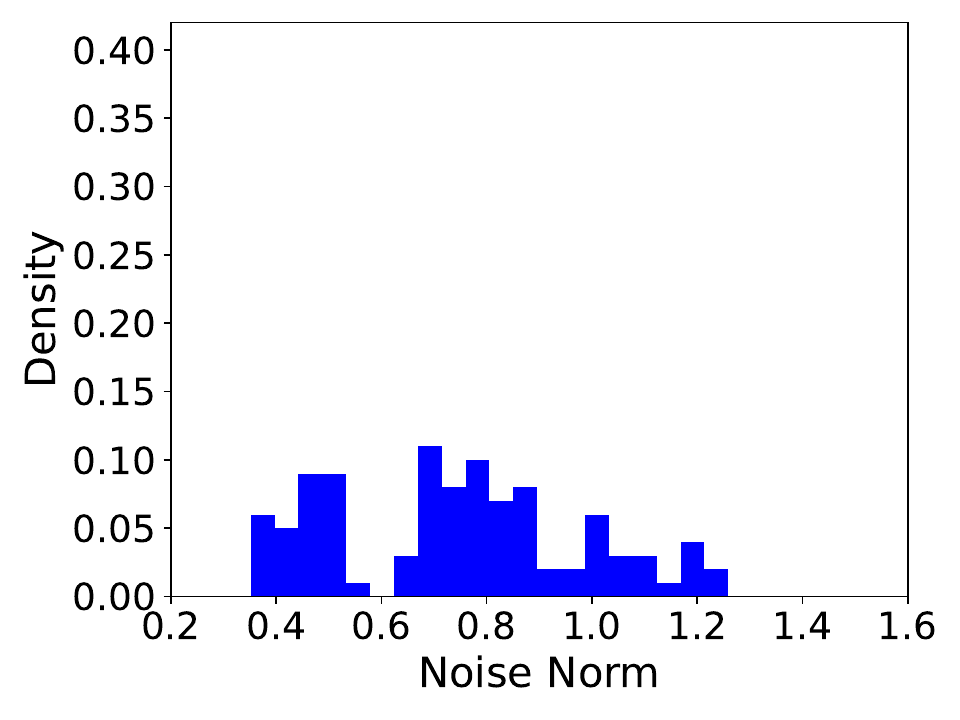}
    \caption{$p=1$}
\end{subfigure}
\begin{subfigure}[t]{0.24\textwidth}
    \centering
    \includegraphics[width=1\linewidth]{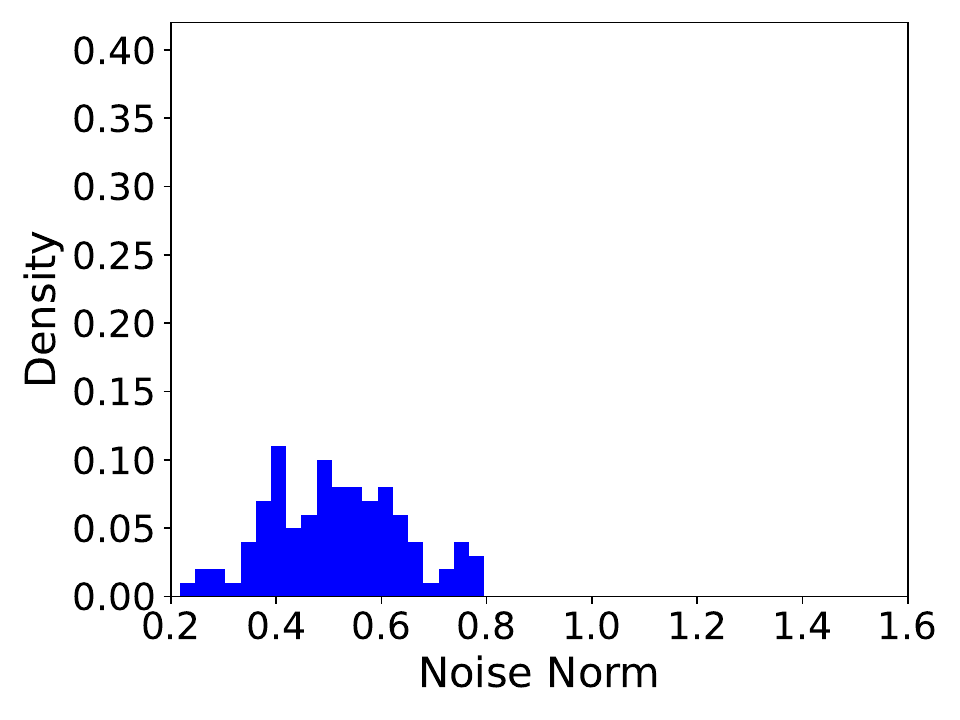}
    \caption{$p=2$}
\end{subfigure}
\begin{subfigure}[t]{0.24\textwidth}
    \centering
    \includegraphics[width=1\linewidth]{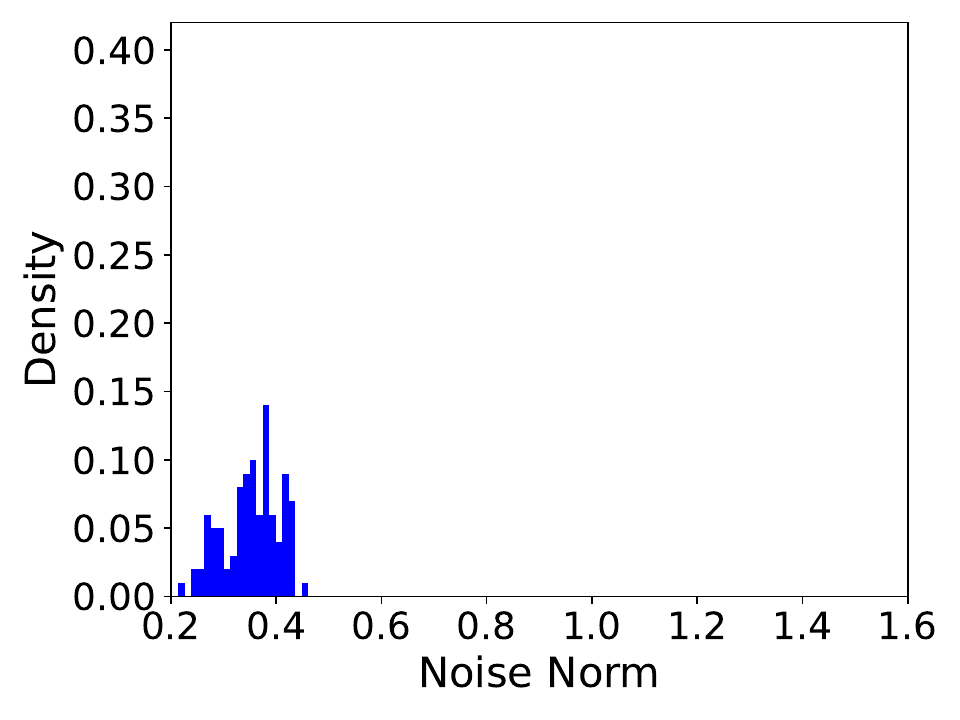}
    \caption{$p=5$}
\end{subfigure}
\begin{subfigure}[t]{0.24\textwidth}
    \centering
    \includegraphics[width=1\linewidth]{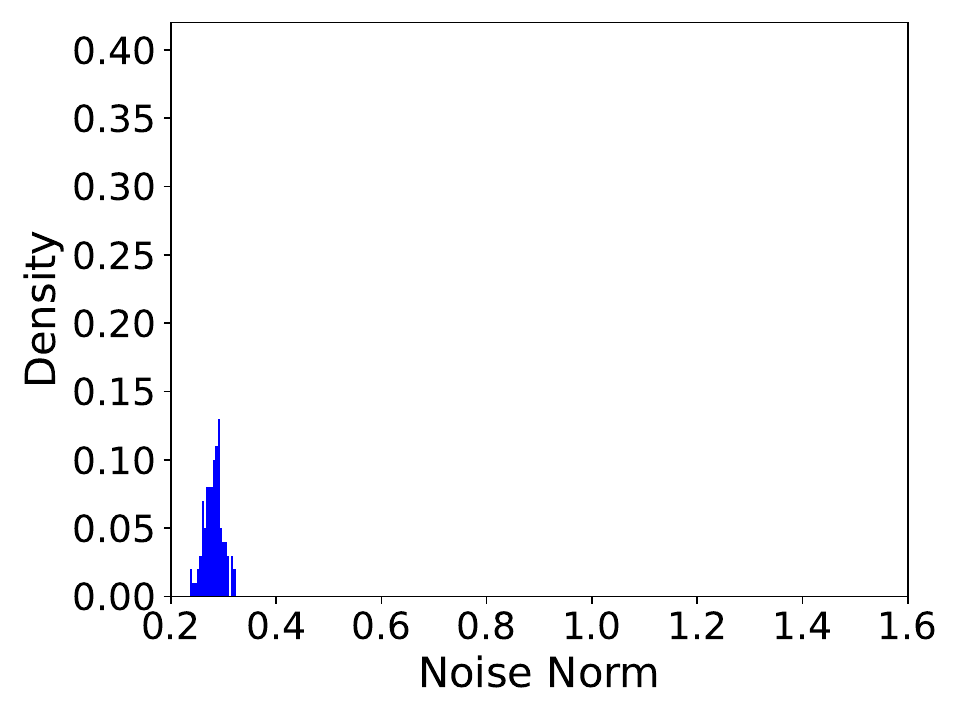}
    \caption{$p=10$}
\end{subfigure}
\caption{Distributions of the norms of the pseudo-gradient noises for CIFAR-10 dataset in the case of \textit{Single SGD}.}
\label{noise_cifar_singleSGD}
\vspace{-0.1in}
\end{figure*}

\begin{figure*}[htbp]
\centering
\begin{subfigure}[t]{0.24\textwidth}
    \centering
    \includegraphics[width=1\linewidth]{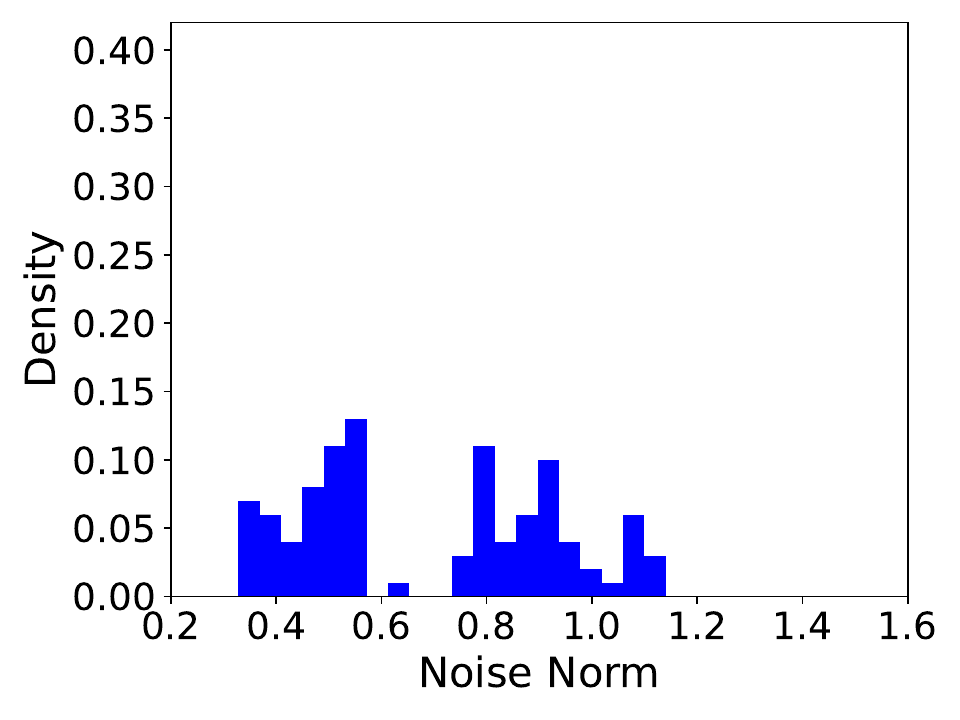}
    \caption{$p=1$}
\end{subfigure}
\begin{subfigure}[t]{0.24\textwidth}
    \centering
    \includegraphics[width=1\linewidth]{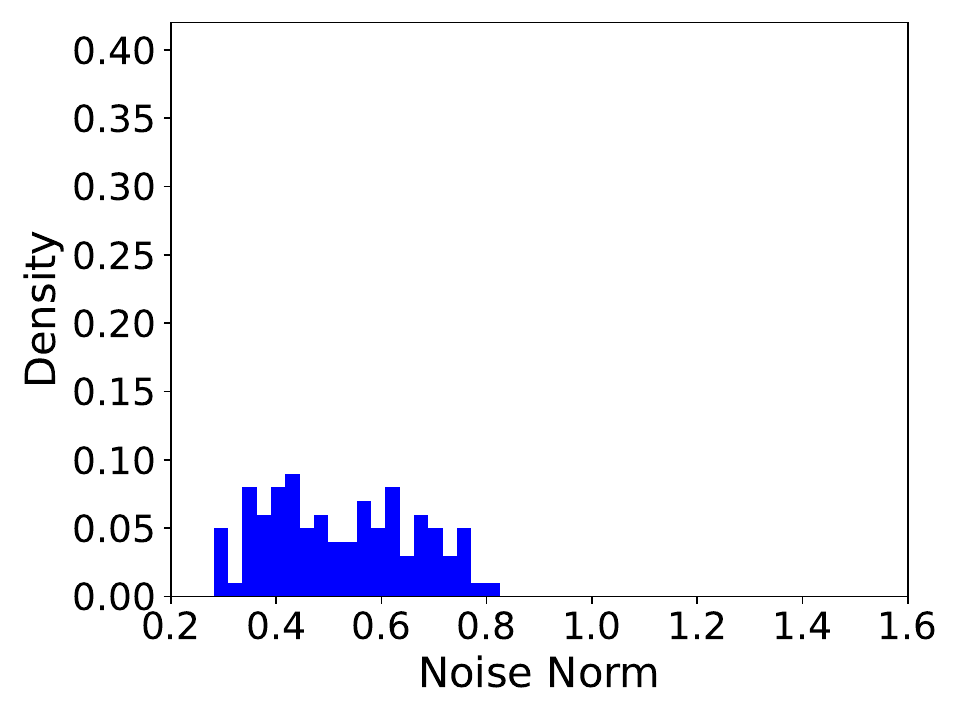}
    \caption{$p=2$}
\end{subfigure}
\begin{subfigure}[t]{0.24\textwidth}
    \centering
    \includegraphics[width=1\linewidth]{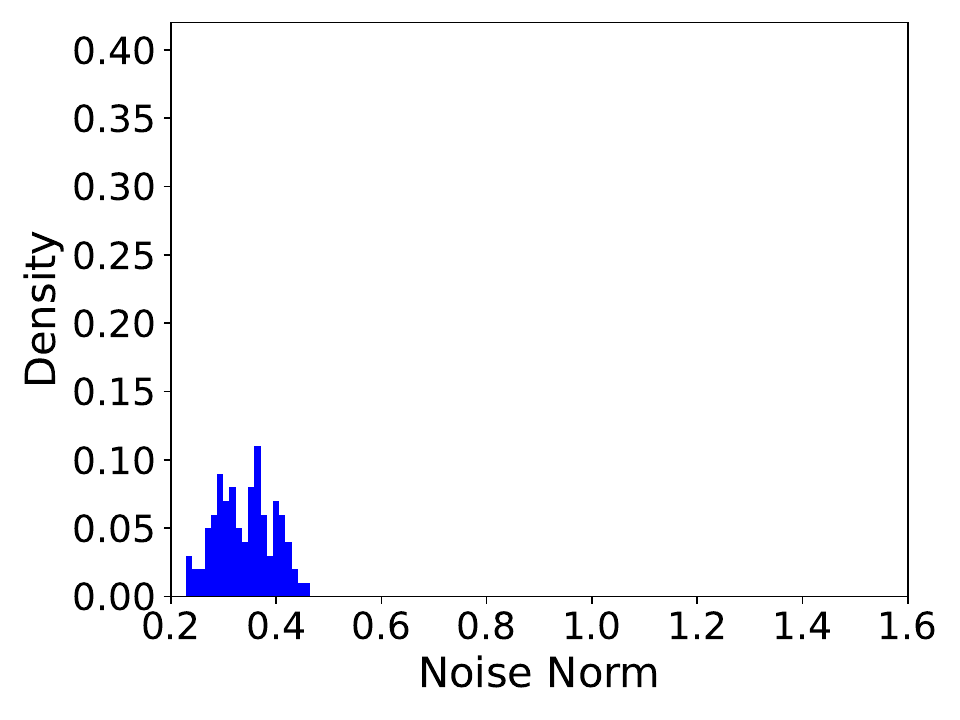}
    \caption{$p=5$}
\end{subfigure}
\begin{subfigure}[t]{0.24\textwidth}
    \centering
    \includegraphics[width=1\linewidth]{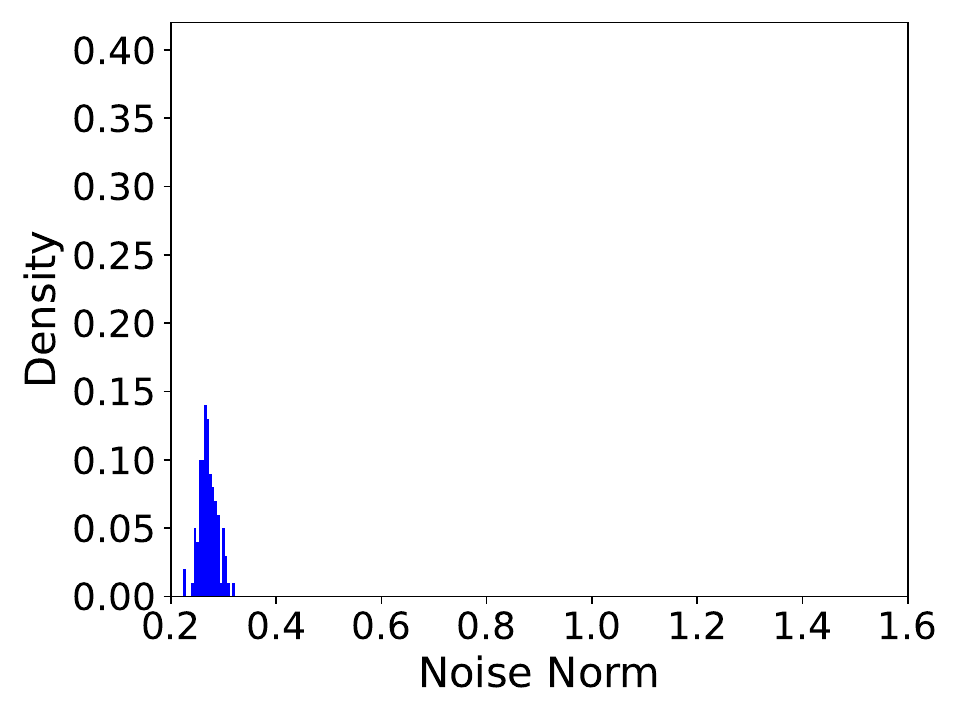}
    \caption{$p=10$}
\end{subfigure}
\caption{Distributions of the norms of the pseudo-gradient noises for CIFAR-10 dataset in the case of \textit{Local Epoch=1}.}
\vspace{-0.1in}
\end{figure*}

\begin{figure*}[htbp]
\centering
\begin{subfigure}[t]{0.24\textwidth}
    \centering
    \includegraphics[width=1\linewidth]{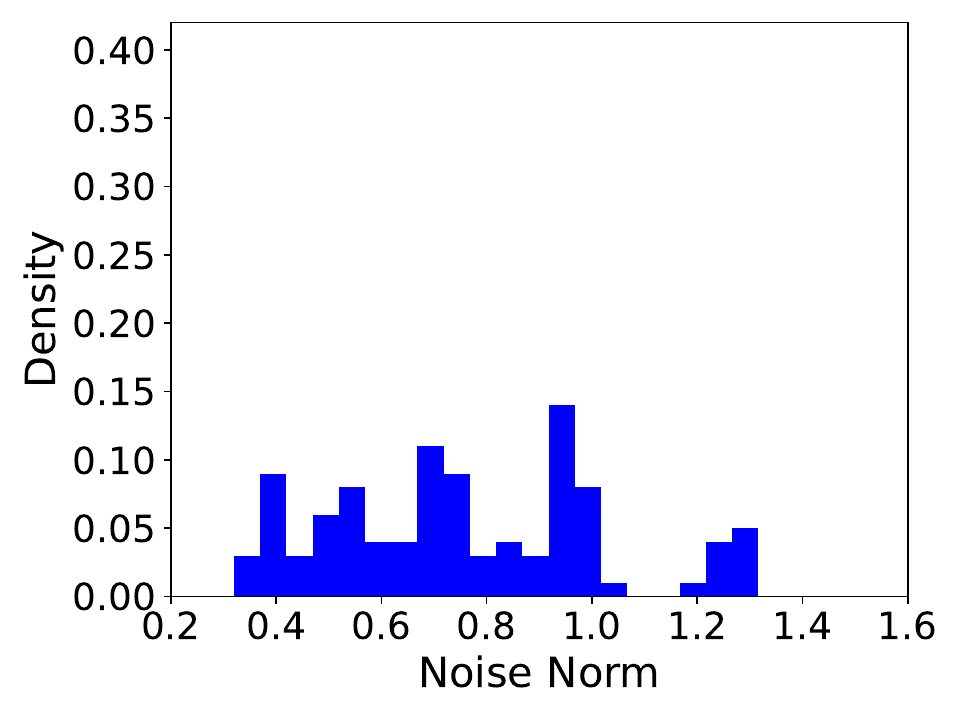}
    \caption{$p=1$}
\end{subfigure}
\begin{subfigure}[t]{0.24\textwidth}
    \centering
    \includegraphics[width=1\linewidth]{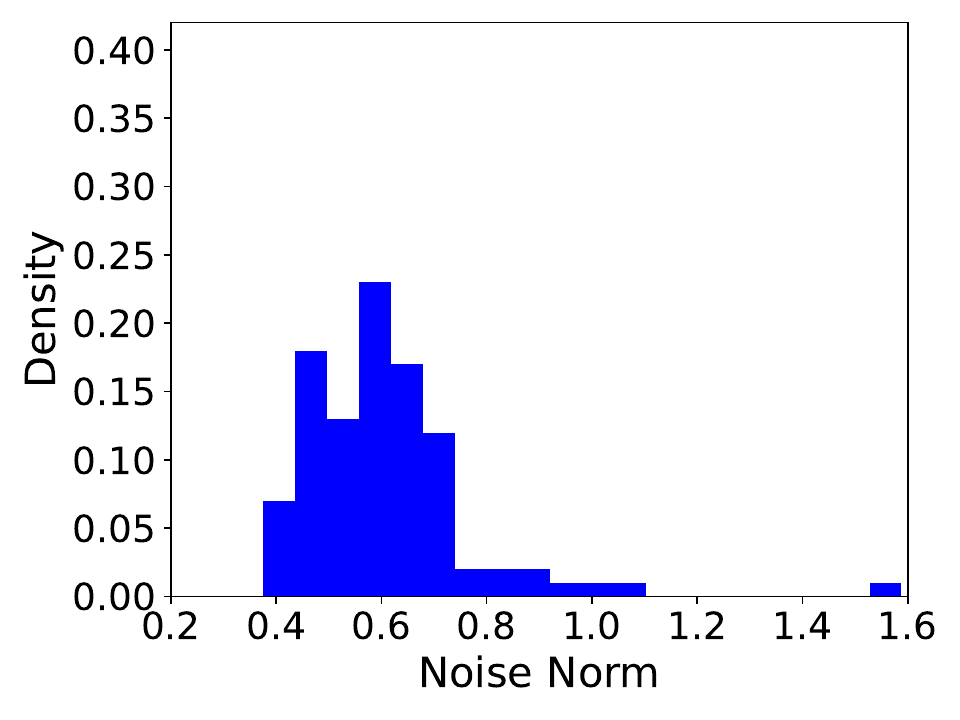}
    \caption{$p=2$}
\end{subfigure}
\begin{subfigure}[t]{0.24\textwidth}
    \centering
    \includegraphics[width=1\linewidth]{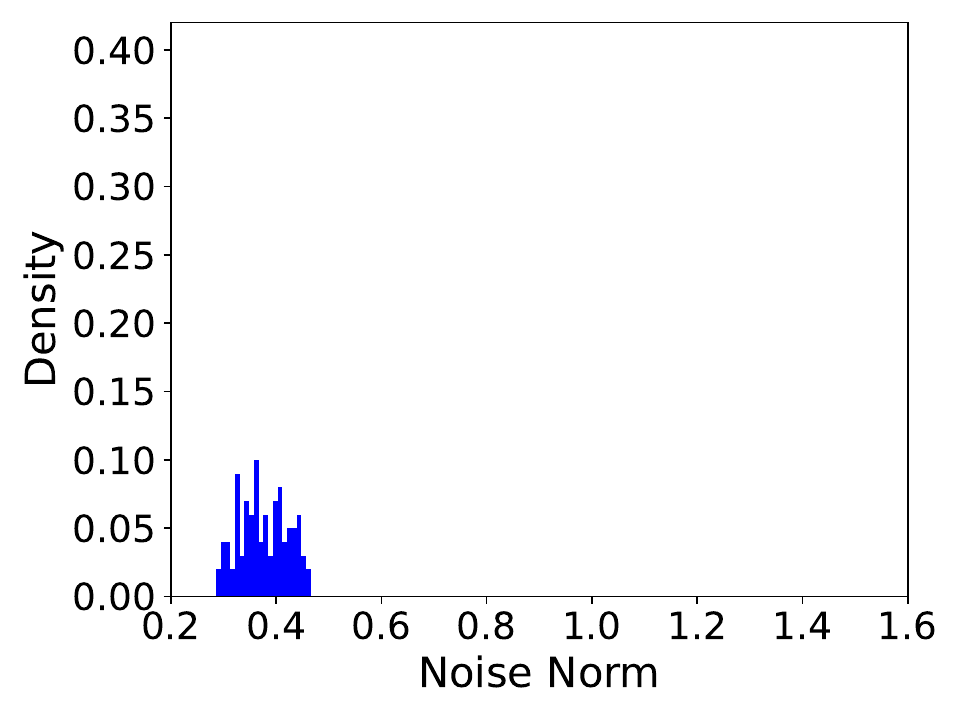}
    \caption{$p=5$}
\end{subfigure}
\begin{subfigure}[t]{0.24\textwidth}
    \centering
    \includegraphics[width=1\linewidth]{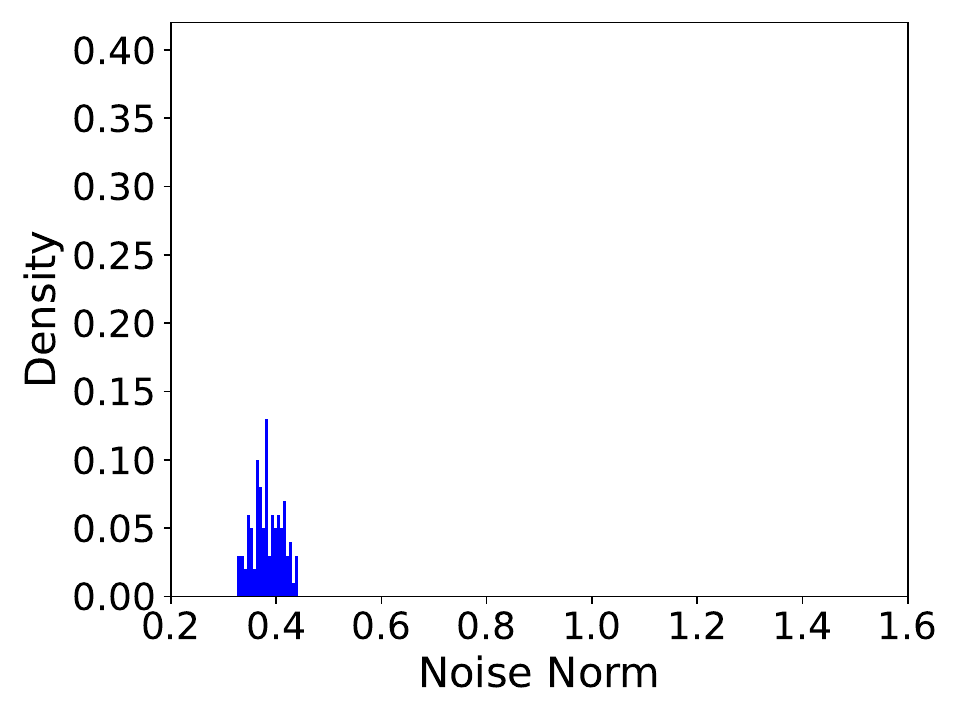}
    \caption{$p=10$}
\end{subfigure}
\caption{Distributions of the norms of the pseudo-gradient noises for CIFAR-10 dataset in the case of \textit{Local Epoch=2}.}
\vspace{-0.1in}
\end{figure*}

\begin{figure*}[htbp]
\centering
\begin{subfigure}[t]{0.24\textwidth}
    \centering
    \includegraphics[width=1\linewidth]{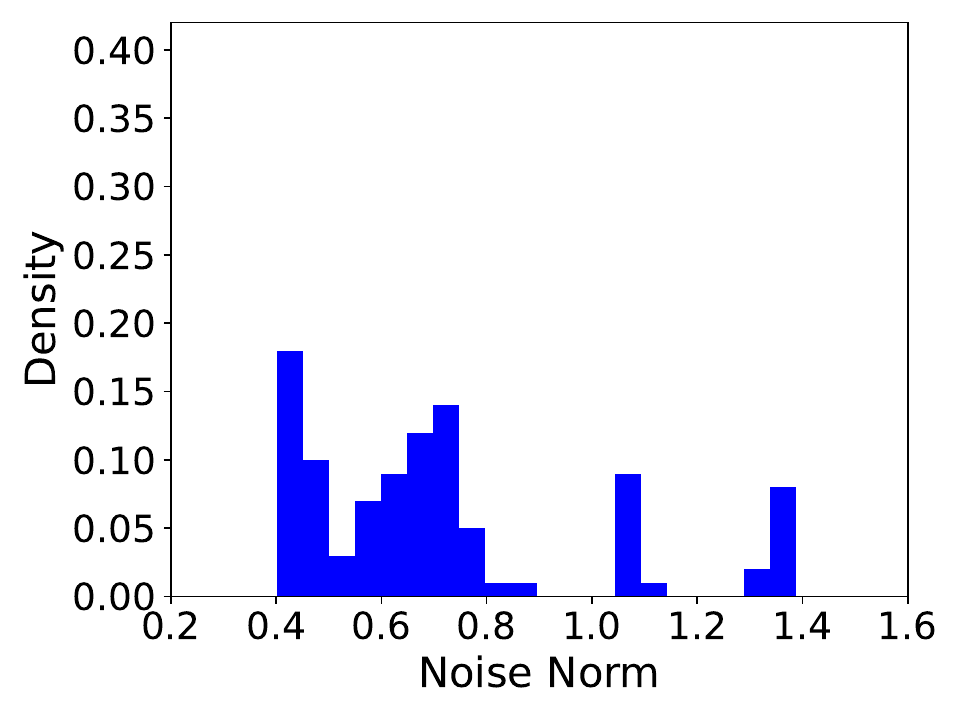}
    \caption{$p=1$}
\end{subfigure}
\begin{subfigure}[t]{0.24\textwidth}
    \centering
    \includegraphics[width=1\linewidth]{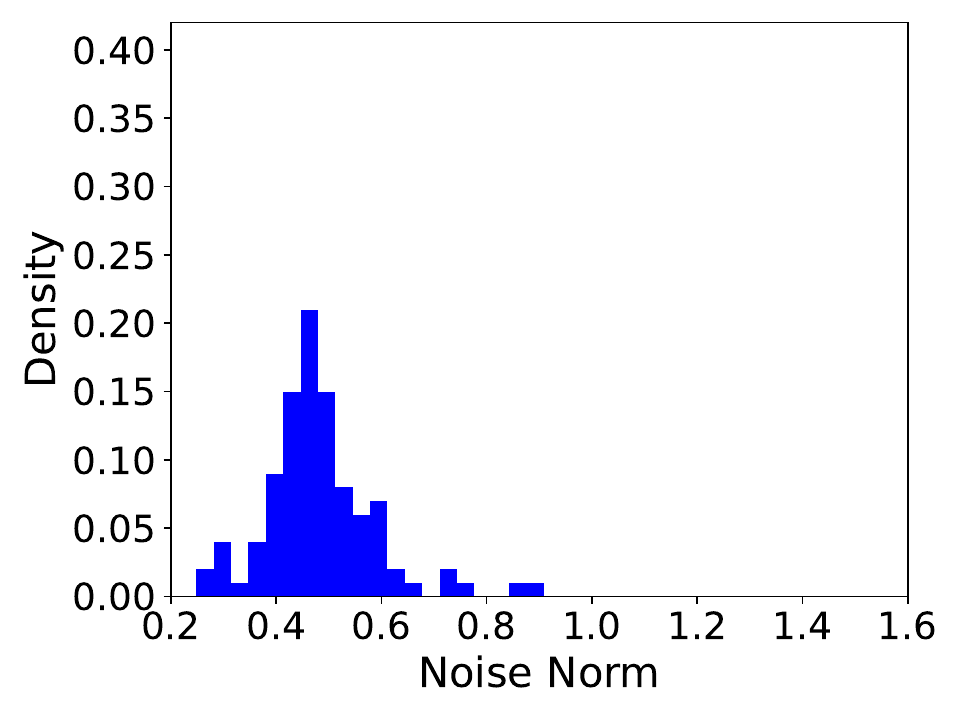}
    \caption{$p=2$}
\end{subfigure}
\begin{subfigure}[t]{0.24\textwidth}
    \centering
    \includegraphics[width=1\linewidth]{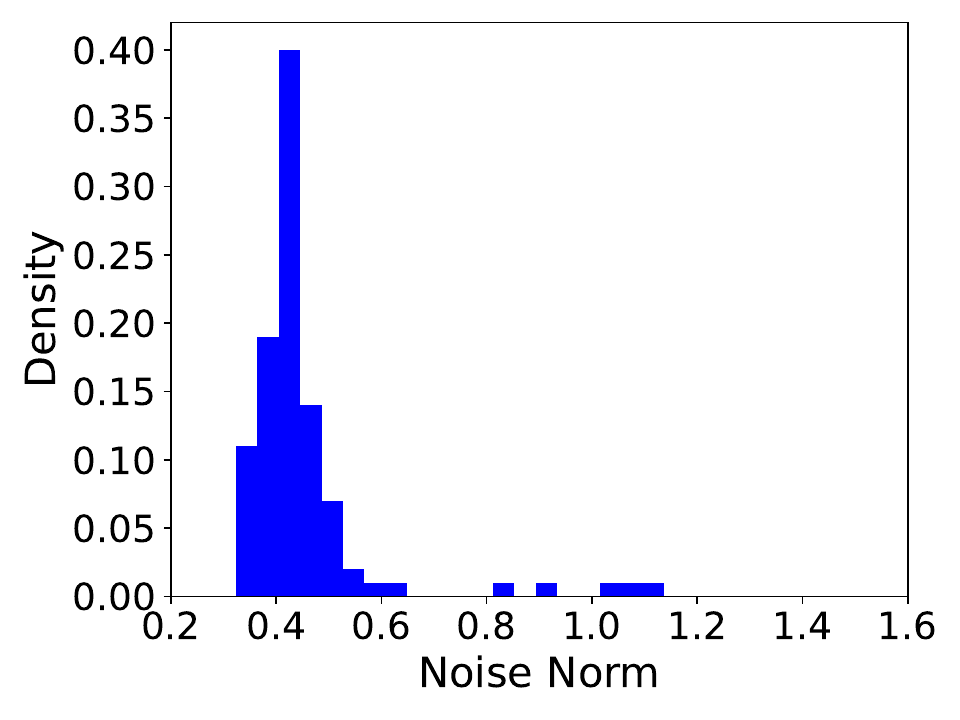}
    \caption{$p=5$}
\end{subfigure}
\begin{subfigure}[t]{0.24\textwidth}
    \centering
    \includegraphics[width=1\linewidth]{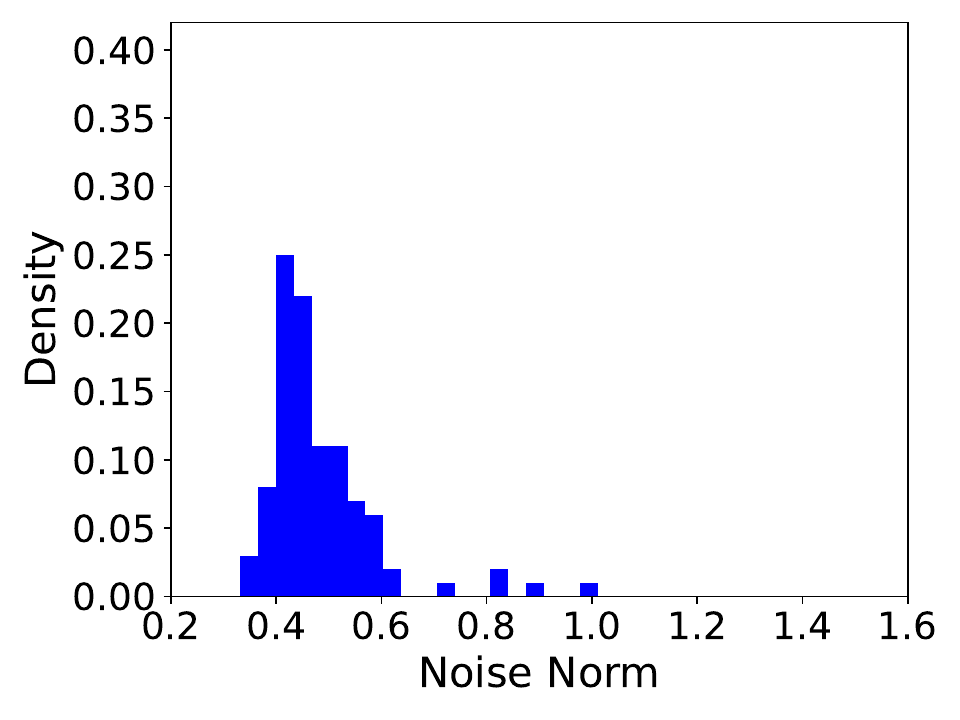}
    \caption{$p=10$}
\end{subfigure}
\caption{Distributions of the norms of the pseudo-gradient noises for CIFAR-10 dataset in the case of \textit{Local Epoch=5}.}
\label{noise_cifar_localepoch5}
\vspace{-0.1in}
\end{figure*}
\FloatBarrier

\subsection{RNN on Shakespeare Dataset}

\subsubsection{Experiment details}
To provide more evidences of the fat-tailed noise phenomenon, we further run a recurrent neural network (RNN) model on Shakespeare dataset.

Shakespeare dataset is a natural non-iid dataset, and it is built from \textit{The Complete Works of William Shakespeare}~\cite{mcmahan2017communication}. 
The learning task is to predict next character, and there are 80 classes of characters in total.
We use a two-layer LSTM classifier containing 100 hidden units with an 8-dimensional (8D) embedding layer.
The model inputs a sequence of 80 characters, embeds each of the characters into a learned 8D space, and then outputs one character per training sample after two LSTM layers and a densely-connected layer.
The dataset and model are taken from~\cite{li2020federated}.

There are $m=143$ clients participating in this experiment.
The global learning rate is chosen as $1.0$, and the local learning rate is chosen as $0.8$.
The batch size is set to $10$, and the communication round is $T=150$.

\subsubsection{Experimental results}
We show the results when local step is set to be one (Single SGD), and multiple local epochs $\{1, 2, 5 \}$. 
In Fig.~\ref{alpha_shakespeare}, we observe that the $\alpha$-value is smaller than 2, and it increases when the number of local epoch increases. This implies that the gradient noise is fat-tailed.
Fig.~\ref{gradient_norm_shakespeare} shows that the distributions of the norms of the pseudo-gradient noises are fat-tailed.

\begin{figure*}[htbp]
\centering
\includegraphics[width=0.5\linewidth]{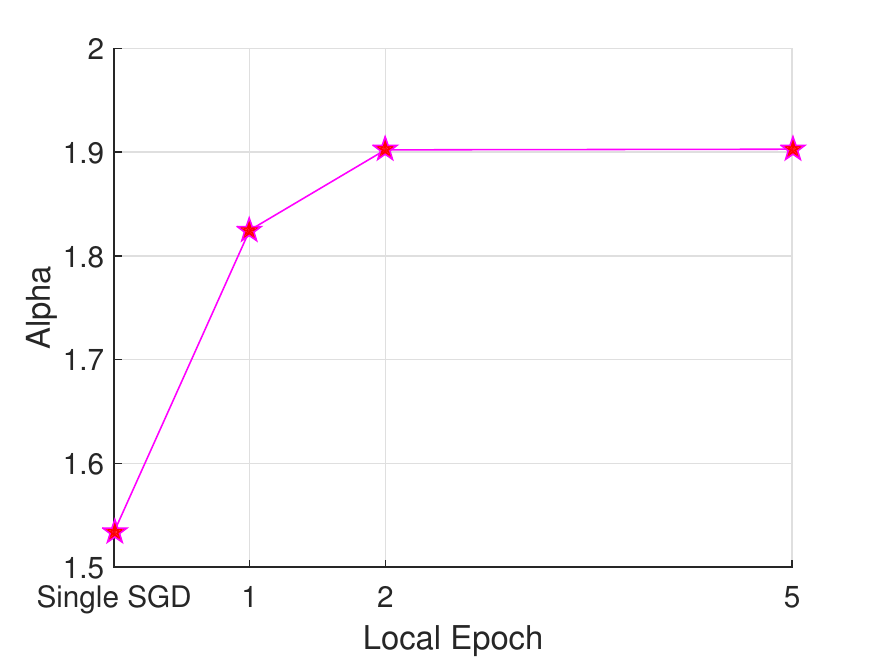}
\caption{Estimation of $\alpha$ for Shakespeare dataset.}
\label{alpha_shakespeare}
\vspace{-0.1in}
\end{figure*}

\begin{figure*}[htbp]
\centering
\begin{subfigure}[t]{0.24\textwidth}
    \centering
    \includegraphics[width=1\linewidth]{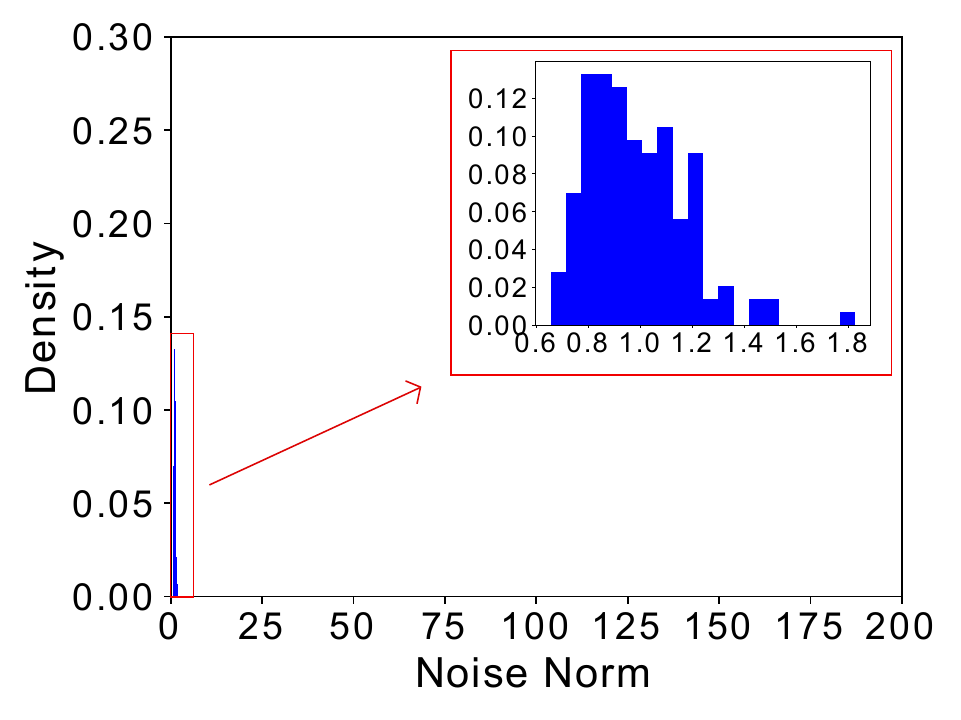}
    \caption{Single SGD}
\end{subfigure}
\begin{subfigure}[t]{0.24\textwidth}
    \centering
    \includegraphics[width=1\linewidth]{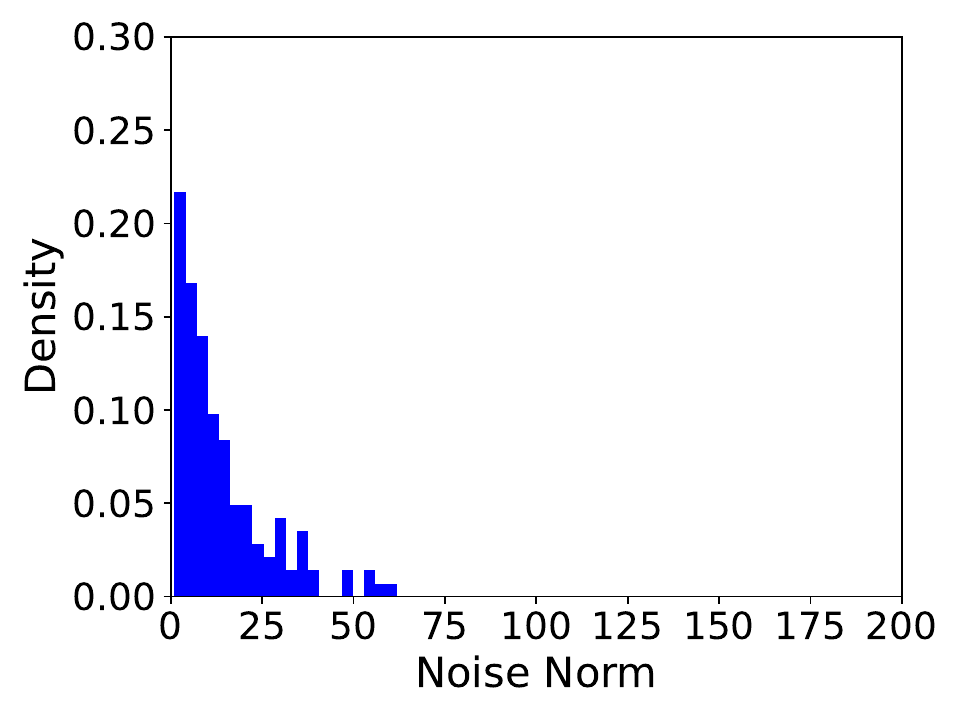}
    \caption{Local Epoch=1}
\end{subfigure}
\begin{subfigure}[t]{0.24\textwidth}
    \centering
    \includegraphics[width=1\linewidth]{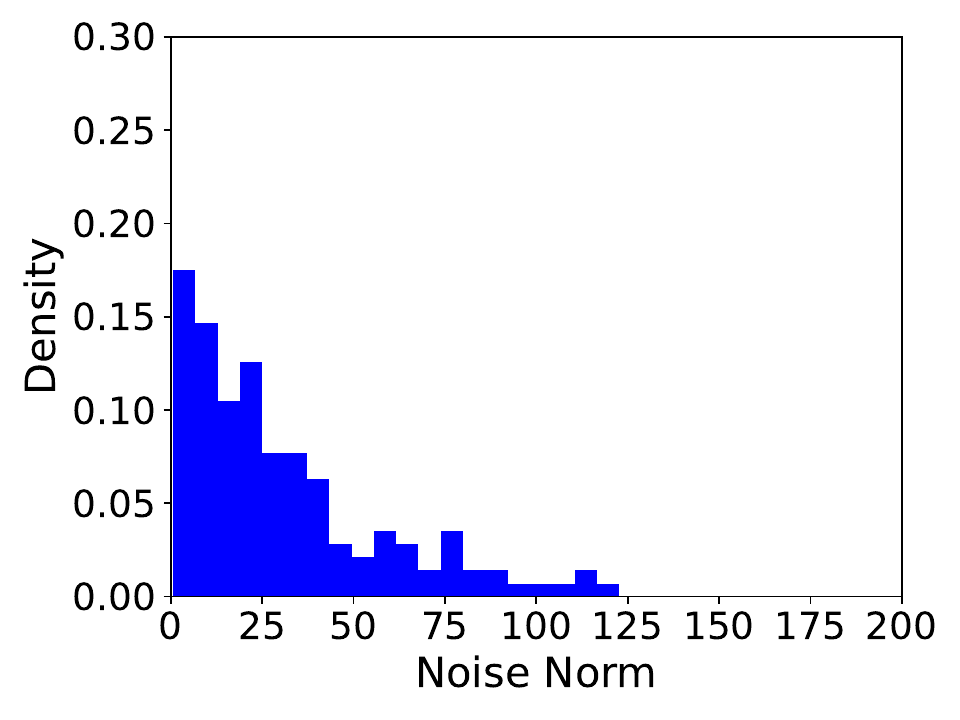}
    \caption{Local Epoch=2}
\end{subfigure}
\begin{subfigure}[t]{0.24\textwidth}
    \centering
    \includegraphics[width=1\linewidth]{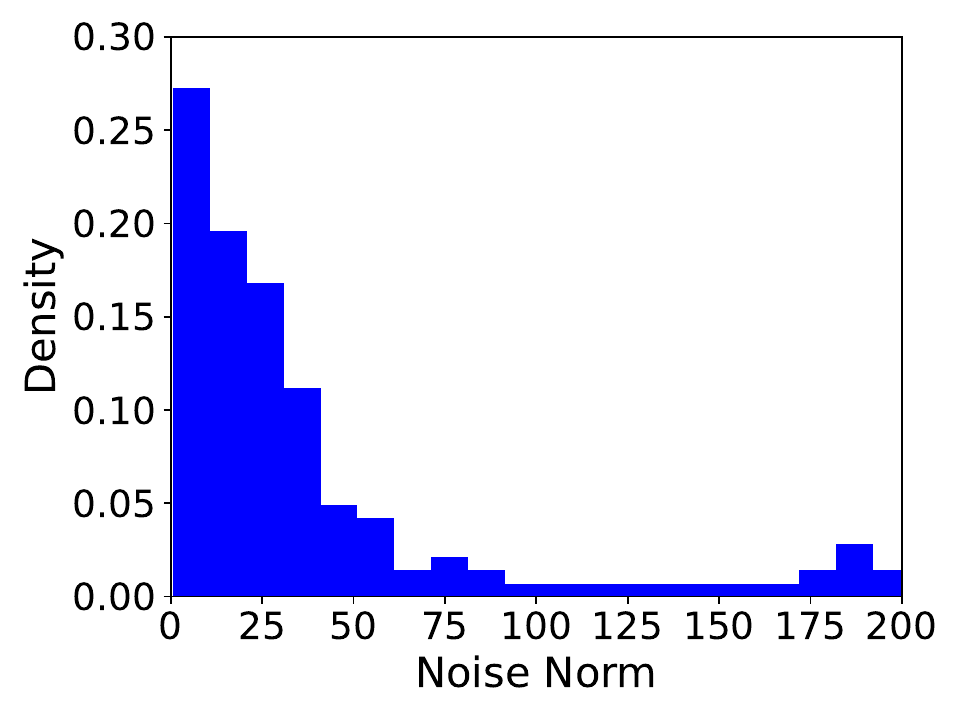}
    \caption{Local Epoch=5}
\end{subfigure}
\caption{Distributions of the norms of the pseudo-gradient noises for Shakespeare dataset.}
\label{gradient_norm_shakespeare}
\vspace{-0.1in}
\end{figure*}

\subsection{Accuracy of Alpha Estimation (Logistic Regression on MNIST Dataset)}

Accurate $\alpha$-value computation requires the full-gradient calculation, and we have to compute both full-gradient and stochastic gradient in each local step. 
This is computationally expensive.
Instead, we use an estimation to approximate the exact $\alpha$-value.
The full-gradient is replaced by the mean value of the stochastic gradients.
We verify the accuracy of this estimation method by running logistic regression on MNIST dataset~\cite{lecun1998gradient}.
The details and the results are described as follows.

\subsubsection{Experiment details}
MNIST dataset contains ten classes of images, and it is manually partitioned using the same method as to partition CIFAR-10 dataset (see details in Appendix~\ref{CNN_CIFAR10}). The number of classes ($p$) that each client has can be used to represent the non-iid level.

$m=100$ clients participate in the experiment. 
The communication round is $T=150$.
The global learning rate is set to $1.0$, and the local learning rate is set to $0.1$.
The batch size is chosen to be $64$.

\subsubsection{Experimental results}
Table~\ref{error_rate_alpha} shows the error rate of $\alpha$-value estimation in different cases, and this implies that the estimation of $\alpha$-value is within an acceptable margin of error.

\begin{table}[]
\centering
\caption{Error rate (\%) of $\alpha$-value estimation.}
\label{error_rate_alpha}
\begin{tabular}{ccccc}
\hline
\multirow{2}{*}{} & \multicolumn{4}{c}{NonIID Index (p)} \\ \cline{2-5} 
                  & 1       & 2       & 5       & 10     \\ \hline
Single SGD        & -2.82   & -1.09   & -0.12   & 3.12   \\
Local Epoch=1     & 1.19    & 0.37    & 1.4     & 2.08   \\
Local Epoch=2     & 1.8     & 1.4     & 1.43    & 1.74   \\
Local Epoch=5     & 1.86    & 0.23    & 0.56    & 0.25   \\ \hline
\end{tabular}
\end{table}

% \begin{figure*}[htbp]
% \centering
% \includegraphics[width=0.5\linewidth]{Fig/Add/Alpha_compare}
% \caption{Comparisons of accurate $\alpha$ and $\alpha$ approximation on MNIST dataset.}
% \vspace{-0.1in}
% \end{figure*}

\section{Experiments in Section~\ref{sec:numerical}}

In this section, we describe the details of the  numerical experiments from Section~\ref{sec:numerical} and provide some extra experimental results.

\subsection{Experiment details}

\subsubsection{Strongly Convex Model with Synthetic Data}
In these experiments, we consider a strongly convex model for Problem~(\ref{objective}) as follows:
\begin{align*}
    & f_{i}\left( x \right)=\mathbb{E}_{\xi}\left[ f_{i}\left( x,\xi \right) \right]\\
    & f_{i}\left( x,\xi \right)=\frac{1}{2}\left\| x \right\|^{2}+\left\langle \xi,x \right\rangle,
\end{align*}
where $x\in \mathbb{R}^{3\times 1}$ and $\xi$ is a random vector. The optimal solution is $f\left( x^{*} \right)=0$ with $x^{*}=\left[ 0;0;0 \right]$.

To compare the performance of FedAvg, \algpi and \algpr, we consider the noise $\xi$ to be a Cauchy distribution($\alpha<2$, fat-tailed) with a location parameter of $0$ and a scale parameter of $2.1$.

To compare the performance of \algpi and \algpr under different scenarios, we consider the noise $\xi$ having different tail-indexes ($\alpha=0.5, 1.0, \text{ and }1.5$) with the same location parameters of $0$ and the same scale parameters of $1$.

For all the distributions of $\xi$ mentioned above, we use the same experimental setup. There are $m=5$ clients participating in the training. We choose the starting point $x_{0}=\left[ 2;1;1.5 \right]$. We set the global learning rate $\frac{\eta\eta_{L}}{m}=0.1$ and the local learning rate $\eta_{L}=0.1$. The local steps we use is $K=2$, and the communication round is $T=300$. The clipping parameter in \algpi we select is $\lambda=3$, and the clipping parameter in \algpr is $\lambda=5$.

\subsubsection{CNN (Non-convex Model) on the CIFAR-10}
To test the performance of \algpi and \algpr for non-convex function, we run a convolutional neural network (CNN) on CIFAR-10 dataset. We compare \algpi and \algpr with FedAvg under different data heterogeneity. 

% To simulate data heterogeneity across clients, we manually distribute the the data to each client in a label-based partition. Specifically, we split the data according to the classes ($p$) of images that each client has. Then, we randomly distribute these partitioned data to $m=10$ clients such that each client has only $p$ classes of images in both training and test data, which causes the heterogeneity of data among different clients. For example, for $p=10$, each client contains training/test data samples with ten classes, and this is the nearly i.i.d case. For the remaining $p$, each client contains data samples with class $p$. Therefore, the classes ($p$) of images in each client's local dataset can be used to represent the non-i.i.d. degree. The smaller the $p$-value, the more heterogeneous the data between clients.

In this experimental setting, we randomly select five clients from $m=10$ clients to participate in each round of the training. 
% We use the global learning rate $\frac{\eta\eta_{L}}{m}=1.0$ and the local learning rate $\eta_{L}=0.1$. 
% The batch size is set to 500. 
The local epoch we use is two. 
% and the communication round is $T=4000$. 
The clipping parameter in \algpi we select is $\lambda=50$, and the clipping parameter in \algpr is $\lambda=2$. 
% The architecture of the CNN model is described in Table~\ref{table:CNN_arch}.
All the remaining settings are the same as described in Appendix~\ref{CNN_CIFAR10}.

% \begin{table}[t!]
% \centering
% \caption{CNN architecture for CIFAR-10.}
% \label{table:CNN_arch}
% \begin{tabular}{ll}
% \hline
% LAYER TYPE             & SIZE     \\ \hline
% Convolution + ReLu     & $3\times 32\times 5$   \\
% Max Pooling            & $2\times 2$      \\
% Convolution + ReLu     & $32\times 64\times 5$  \\
% Max Pooling            & $2\times 2$      \\
% Fully Connected + ReLU & $1600\times 512$ \\
% Fully Connected + ReLU & $512\times 128$  \\
% Fully Connected        & $128\times 10$   \\ \hline
% \end{tabular}
% \end{table}

\subsection{Additional experimental Results}
We provide two additional results when applying FedAvg, \algpi and \algpr to the CNN model on CIFAR-10 dataset. In Fig.~\ref{fig:noniid}, we show the percentage of successful training over 5 trials in non-i.i.d. cases when the non-i.i.d. index $p=1$ and $p=5$. These results further support our finding that \alg methods and especially \algpi reduce catastrophic training failures compared to FedAvg.

\begin{figure}[t!]
     \centering
     \begin{subfigure}[b]{0.48\textwidth}
         \centering
         \includegraphics[width=\textwidth]{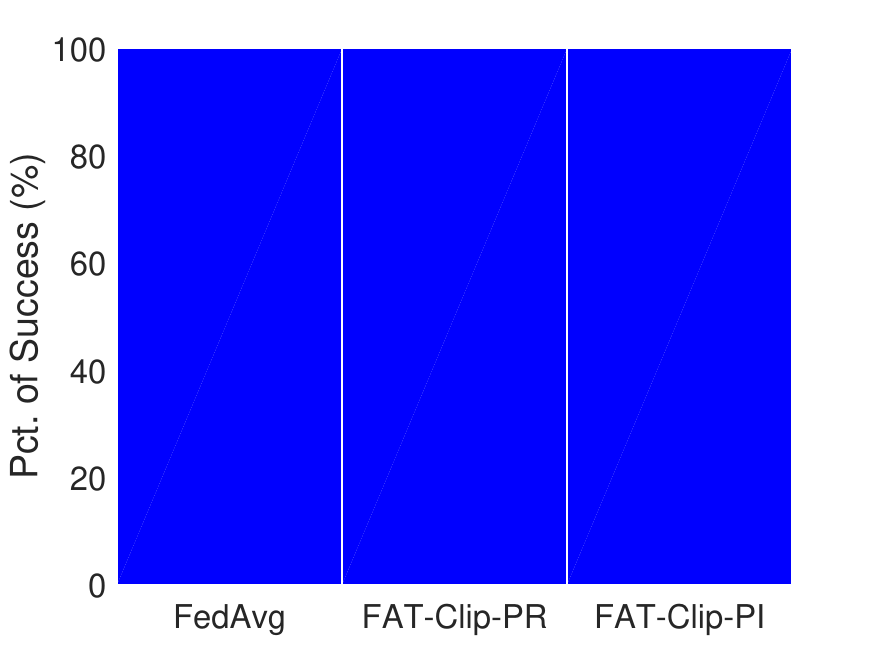}
         \caption{non-i.i.d. index $p=1$}
         \label{fig:noniid_1}
     \end{subfigure}
     \hfill
     \begin{subfigure}[b]{0.48\textwidth}
         \centering
         \includegraphics[width=\textwidth]{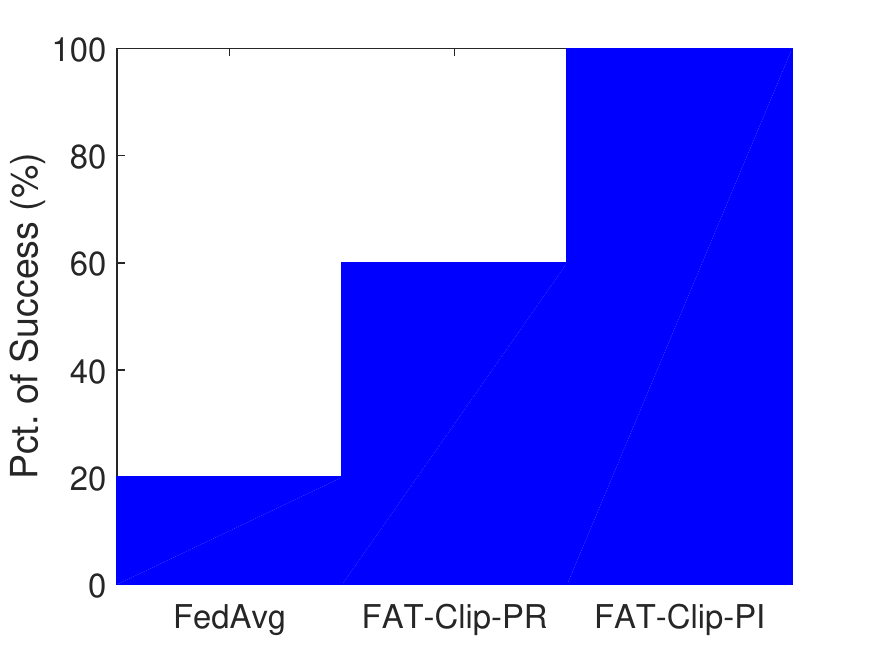}
         \caption{non-i.i.d. index $p=5$}
         \label{fig:noniid_5}
     \end{subfigure}
        \caption{Percentage of successful training over 5 trials when applying FedAvg, \algpr and \algpi to CIFAR-10 dataset in non-i.i.d. cases.}
        \label{fig:noniid}
\end{figure}

\end{document}